%% file: thesis.tex
\documentclass[12pt]{rutgersthesis}

\usepackage{url}
\usepackage[T1]{fontenc}
\usepackage{amsmath}
\usepackage{amsthm}
\usepackage{wrapfig}

\usepackage[linesnumbered,ruled,vlined]{algorithm2e}
\usepackage{algpseudocode}
\usepackage{multirow}
\newcommand{\thesisType}{Dissertation} 
\newcommand{\thesisTitle}{From Theory to Practice: Advancing Multi-Robot Path Planning Algorithms and Applications} 
\newcommand{\yourName}{Teng Guo} 
\newcommand{\yourDegree}{Doctor of Philosophy} 
\newcommand{\yourProgram}{Department of Computer Science} 
\newcommand{\yourDirector}{Jingjin Yu} 
\newcommand{\yourCommitteeNumber}{4} 
\newcommand{\yourMonth}{October} 
\newcommand{\yourYear}{2025} 
\newtheorem{problem}{Problem}
\newtheorem{proposition}{Proposition}[section]
\newtheorem{lemma}{Lemma}[section]

\newtheorem{corollary}{Corollary}[section]
\newtheorem{theorem}{Theorem}[section]
\newtheorem{remark}{Remark}
\newtheorem{definition}{Definition}

\def\gtl{\mathcal G_{tl}\xspace}
\def\gbr{\mathcal G_{br}\xspace}

\begin{document}

\input{copyright}

\makeTitlePage{Month}{Year}

\begin{frontmatter}
    \input{abstract}
    \input{acknowledgments}
    \makeTOC
    
    \makeListOfTables
    \makeListOfFigures
    \input{abbrevs}
\end{frontmatter}

\begin{thesisbody}
    \input{chapters/chapter1}
    \input{chapters/chapter2}
    \input{chapters/chapter3}
    \input{chapters/chapter4}
    \input{chapters/chapter5}

    \input{chapters/chapter6}

    \input{chapters/chapter7}
\input{chapters/appendix.tex}\input{publications.tex}
    \makeBibliography
\end{thesisbody}

\end{document}

%% file: copyright.tex
\clearpage
\begin{center}

\vspace*{\fill}

\copyright {\yourYear}\\
{Teng Guo}\\
ALL RIGHTS RESERVED\\

\vspace*{\fill}

\end{center}

\pagenumbering{gobble}
\clearpage

%% file: abstract.tex
\begin{my_abstract}

The labeled \gls{mpp} problem, despite its various setups and variations, generally involves routing a set of mobile robots from a start configuration to a goal configuration as efficiently as possible while avoiding collisions. In recent years, \gls{mpp} has garnered significant attention, with notable advancements in both solution quality and efficiency. However, due to its inherent complexity and critical role in industrial applications, there remains a continuous drive for improved performance.  

This dissertation introduces methods with provable guarantees and effective heuristics aimed at advancing \gls{mpp} algorithms, particularly in terms of scalability and optimality. The first part of this dissertation presents a theoretical study of \gls{mpp} on grid graphs with high robot density. Specifically, we investigate \gls{mpp} on dense two-dimensional grids, with practical applications in areas such as warehouse automation and parcel sorting. While previous methods provide polynomial-time solutions with $O(1)$ optimality guarantees, their high constant factors limit practical applicability. In contrast, we introduce the Rubik Table method, which achieves $(1 + \delta)$-makespan optimality, where $\delta \in (0, 0.5]$, for large grids. Our approach supports up to $\frac{m_1 m_2}{2}$ robots and integrates effective heuristics, enabling near-optimal makespan solutions in minutes for instances involving tens of thousands of robots. This establishes a new theoretical benchmark for scalable, provably optimal multi-robot routing algorithms.  

The second part of this dissertation shifts focus to real-world applications of \gls{mpp}. First, we consider the design of optimal layouts for well-structured infrastructures, which are essential for traditional autonomous warehouses and parking systems. The effectiveness of these systems heavily depends on efficient space utilization, minimizing robot congestion, and optimizing task execution time. To further enhance operational efficiency, we propose a puzzle-based multi-robot parking system that enables high-density storage and retrieval of autonomous vehicles. By leveraging combinatorial optimization techniques, we design algorithms that ensure smooth transitions between vehicle placements while avoiding unnecessary movements and deadlocks. We analyze different retrieval strategies, demonstrating how our approach can significantly improve throughput and reduce waiting times compared to conventional parking methodologies.  
Additionally, we investigate \gls{mpp} when robots are modeled as autonomous Reeds-Shepp cars, which introduce additional constraints due to nonholonomic motion. These vehicles require smooth, continuous paths that respect turning radius limitations, making traditional grid-based planning techniques insufficient. By extending state-of-the-art \gls{mpp} algorithms, we develop novel heuristics and trajectory optimization techniques to address these constraints. Our approach incorporates motion primitives and local path smoothing, ensuring both feasibility and efficiency in real-world autonomous navigation scenarios. Through extensive simulations and real-world experiments, we validate our methods, demonstrating their applicability in urban autonomous driving, robotic convoy coordination, and automated transport systems.
\end{my_abstract}

%% file: acknowledgments.tex
\begin{acknowledgments}

First and foremost, I would like to express my heartfelt gratitude to my advisor, Prof. Jingjin Yu, director of the Algorithmic Robotics and Control Lab (ARCL). Throughout my Ph.D. journey, he has been an exceptional advisor, providing invaluable guidance in research, supporting my studies through funding, and helping me navigate academic challenges. His professional advice, especially during the difficult times of my final semester, was a source of immense support. It has been both a privilege and a joy to work on algorithmic robotics and engage in numerous exciting projects under his mentorship.

I am deeply grateful to my committee members, Prof. Kostas Bekris, Prof. Abdeslam Boularias and Prof. Bo Yuan, for their guidance and support. Prof. Bekris held numerous insightful talks, from which I learned a great deal about robotics. Prof. Boularias’s course, Robotics Learning, provided me with a strong foundation in robot learning. Their valuable feedback and encouragement, especially during my qualification exam, were instrumental in shaping my research journey.

My labmates at ARCL have been an integral part of this journey. I am deeply grateful to Shuai D. Han, Si Wei Feng, Kai Gao,  Baichuan Huang, Wei Tang, Zihe Ye, Duo Zhang, Kowndinya Boyalakuntla, Tzvika Geft, and Justin Yu for their camaraderie and support.

Pursuing a Ph.D. is a demanding endeavor, and I am immensely thankful to my friends who made my graduate school life more enjoyable and meaningful. I would like to acknowledge Shiyang Lu, Haonan Chang, Junchi Liang, Rui Wang,  Aravind Sivaramakrishnan, Yinglong Miao, Yuhan Liu, , and Xinyu Zhang for their unwavering friendship and encouragement.

Special thanks go to my collaborators in the League of Robot Runners, including Zhe Chen, Shaohung Chan, Han Zhang, Yue Zhang and Prof.  Daniel Harabor. It has been an honor and a pleasure to work and learn alongside such talented individuals.

The pandemic, which spanned the first three years of my Ph.D., brought unprecedented challenges to our lives and work. Despite these difficulties, we have been fortunate to return to normalcy. I would like to extend my gratitude to the selfless medical workers and volunteers who dedicated themselves to combating the pandemic and supporting communities worldwide.

Finally, I wish to express my deepest appreciation to my parents for their unconditional love, care, and support throughout my life. Their belief in me has been a constant source of strength and motivation.


\end{acknowledgments}

%% file: abbrevs.tex
%

\newacronym{rtp}{GRP}{Grid Rearrangement Problem}
\newacronym{rta}{GRA}{Grid Rearrangement Algorithm}
\newacronym{mapf}{MAPF}{Multi-Agent Path Finding}
\newacronym{eecbsthreads}{EECBS}{Enhanced Conflict-Based Search}
\newacronym{ecbsts}{ECBS-4t}{Enhanced Conflict-Based Search with 4 splitting threads}
\newacronym{ddm}{DDM}{Data-Driven MRPP Algorithm}
\newacronym{lacam}{LaCAM}{Lazy Constraints Addition Search for Multi-Agent Pathfinding}
\newacronym{rtpk}{GRPKD}{Grid Rearrangement Problem with K-Dimensions}
\newacronym{rtptwo}{GRP2D}{Grid Rearrangement Problem in 2D}
\newacronym{rtatwo}{GRA2D}{Grid Rearrangement Algorithm in 2D}
\newacronym{rtpthree}{GRP3D}{Grid Rearrangement Problem in 3D}
\newacronym{rtathree}{GRA3D}{Grid Rearrangement Algorithm in 3D}
\newacronym{rthddd}{GRH3D}{Grid Rearrangement with Highways in 3D}
\newacronym{rthdddlba}{GRH3D-LBA}{Grid Rearrangement with Highways in 3D using LBA}
\newacronym{rtlmddd}{GRLM3D}{Grid Rearrangement with Linear Merge in 3D}
\newacronym{rtmddd}{GRM3D}{Grid Rearrangement for MRPP in 3D}
\newacronym{rtmdd}{GRM}{Grid Rearrangement for MRPP}
\newacronym{rth}{GRH}{Grid Rearrangement with Highways}
\newacronym{mcp}{MCP}{Minimal  Communication  Policy}
\newacronym{rtlm}{GRLM}{Grid Rearrangement with Linear Merge}
\newacronym{rtmlba}{GRM-LBA}{Grid Rearrangement for MRPP using LBA}
\newacronym{rthlba}{GRH-LBA}{Grid Rearrangement for MRPP using LBA}
\newacronym{rthip}{GRH-IP}{Grid Rearrangement with Highways using Integer Programming}
\newacronym{rthpr}{GRH-PR}{Grid Rearrangement with Highways using Path Refinement}
\newacronym{irth}{iGRH}{Improved Grid Rearrangement with Highways}

\newacronym{mpp}{MRPP}{Multi-Robot Path Planning}
\newacronym{ilp}{ILP}{Integer Linear Programming}
\newacronym{lba}{LBA}{Linear Bottleneck Assignment}
\newacronym{cbs}{CBS}{Conflict-Based Search}
\newacronym{ecbs}{ECBS}{Enhanced Conflict Based Search}
\newacronym{eecbs}{EECBS}{Explicit Estimation Conflict-Based Search}
\newacronym{bcpr}{BVPR}{Batched Vehicle Parking and Retrieval}
\newacronym{ccpr}{CVPR}{Continuous Vehicle Parking and Retrieval}
\newacronym{crp}{VSP}{Vehicle Shuffling Problem}

\newacronym{csmp}{CSMP}{Coupling Single Motion Primitives by MCP}
\newacronym{makespan}{MKPN}{Makespan}
\newacronym{anm}{ANM}{Average Number of Moves per task}
\newacronym{aprt}{APRT}{Average Parking and Retrieval Time}

\newacronym{pibt}{PIBT}{Priority Inheritance with Backtracking}
\newacronym{clpibt}{PBCR}{Priority Inheritance with Backtracking for Car-like Robots}
\newacronym{clcbs}{CL-CBS}{Conflict Based Search for Car-like Robots}
\newacronym{clecbs}{ECCR}{Enhanced Conflict Based Search for Car-like Robots}

\newacronym{soc}{SOC}{Sum of Costs}
\newacronym{noc}{NOC}{Number of Conflicts}
\newacronym{dcr}{DCR}{Enhanced Conflict Based Search}

\newacronym{unpp}{UNPP}{Unlabeled Multi-Robot Path Planning}
\newacronym{hca}{HCA}{Hierarchical Coupled A$\*$}
\newacronym{mapd}{MAPD}{Multi-Agent Pickup and Delivery}
\newacronym{pbs}{PBS}{Priority-Based Search}
\newacronym{wfvs}{WFVS}{Enhanced Conflict Based Search}
\newacronym{wcvs}{WCVS}{Enhanced Conflict Based Search}
\newacronym{per}{PER}{Path Efficiency Ratio}

\newacronym{wcs}{WCS}{Well-Connected Set}
\newacronym{mwcs}{MWCS}{Maximal Well-Connected Set}
\newacronym{lwcs}{LWCS}{Largest Well-Connected Set}
\newacronym{swcs}{SWCS}{Semi-Well-Connected Set}
\makeListOfAcronyms

%% file: chapters/chapter1.tex
\chapter{Introduction}\label{chap:introduction}
Cooperative path and motion planning for multiple mobile robots is a fundamental problem in robotics. Its complexity has been well studied \cite{hopcroft1984complexity,Solovey2016,yu2013structure}, and it has numerous applications, including rigid body assembly~\cite{halperin1998general}, evacuation planning~\cite{rodriguez2010behavior}, formation control~\cite{smith2009automatic,preiss2017crazyswarm}, collaborative localization~\cite{fox2000probabilistic}, micro-droplet manipulation~\cite{griffith2005a}, object transportation~\cite{rus1995a}, human-robot interaction~\cite{knepper2012a}, search and rescue~\cite{jennings1997a}, coverage and surveillance~\cite{feng2023optimal}, and notably, warehouse automation~\cite{wurman2008coordinating} (see Figure~1.1). Efficient algorithms for different settings can significantly enhance process efficiency.

In ~\ref{fig:intro_applications}, we highlight four real-world \gls{mpp} applications: (a) Autonomous mobile robots transporting goods in an Amazon fulfillment center; (b) Self-driving cars navigating an intersection with intelligent traffic control; (c) Drone swarms creating dynamic aerial formations in a night show; (d) Multi-agent pathfinding in real-time strategy (RTS) games.
In this dissertation, we study the labeled \gls{mpp} on its multiple different variations. Typically, in an \gls{mpp} problem instance, there is a set of mobile robots in a workspace. Given the start and goal configurations of the robots and their permissible motion primitives,  the objective is to route  the robots from their start configuration to the goal configuration, collision-free.  

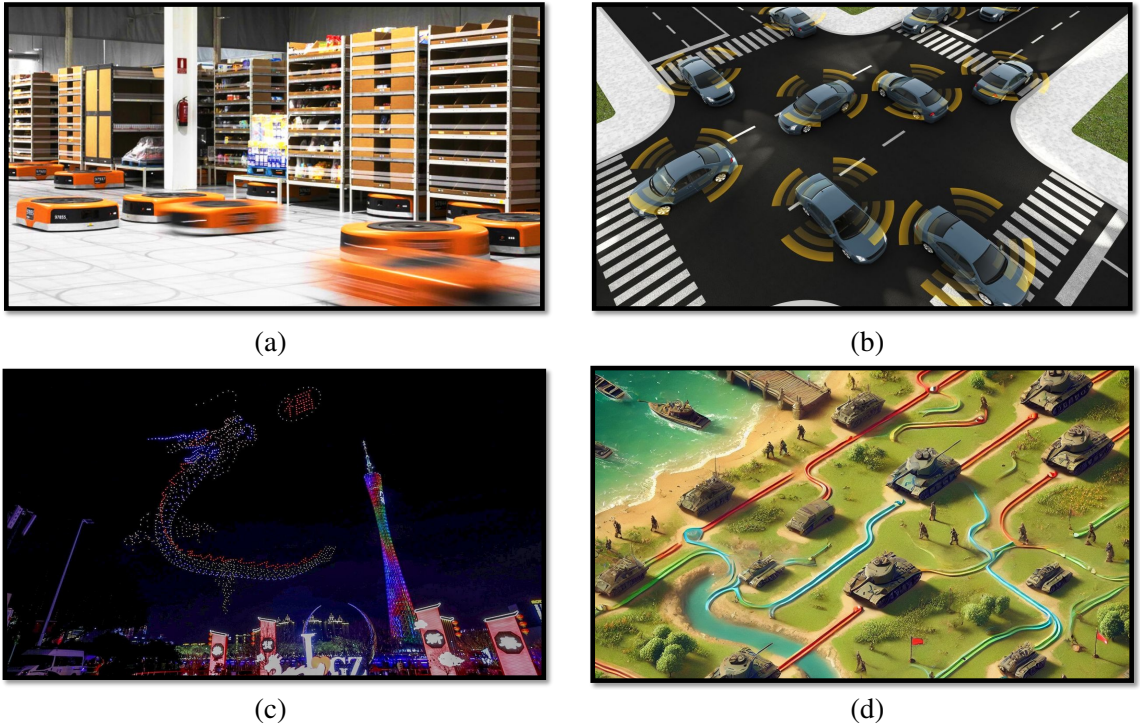
\begin{figure}[ht]
    \centering
  \begin{overpic}               
        [width=1\linewidth]{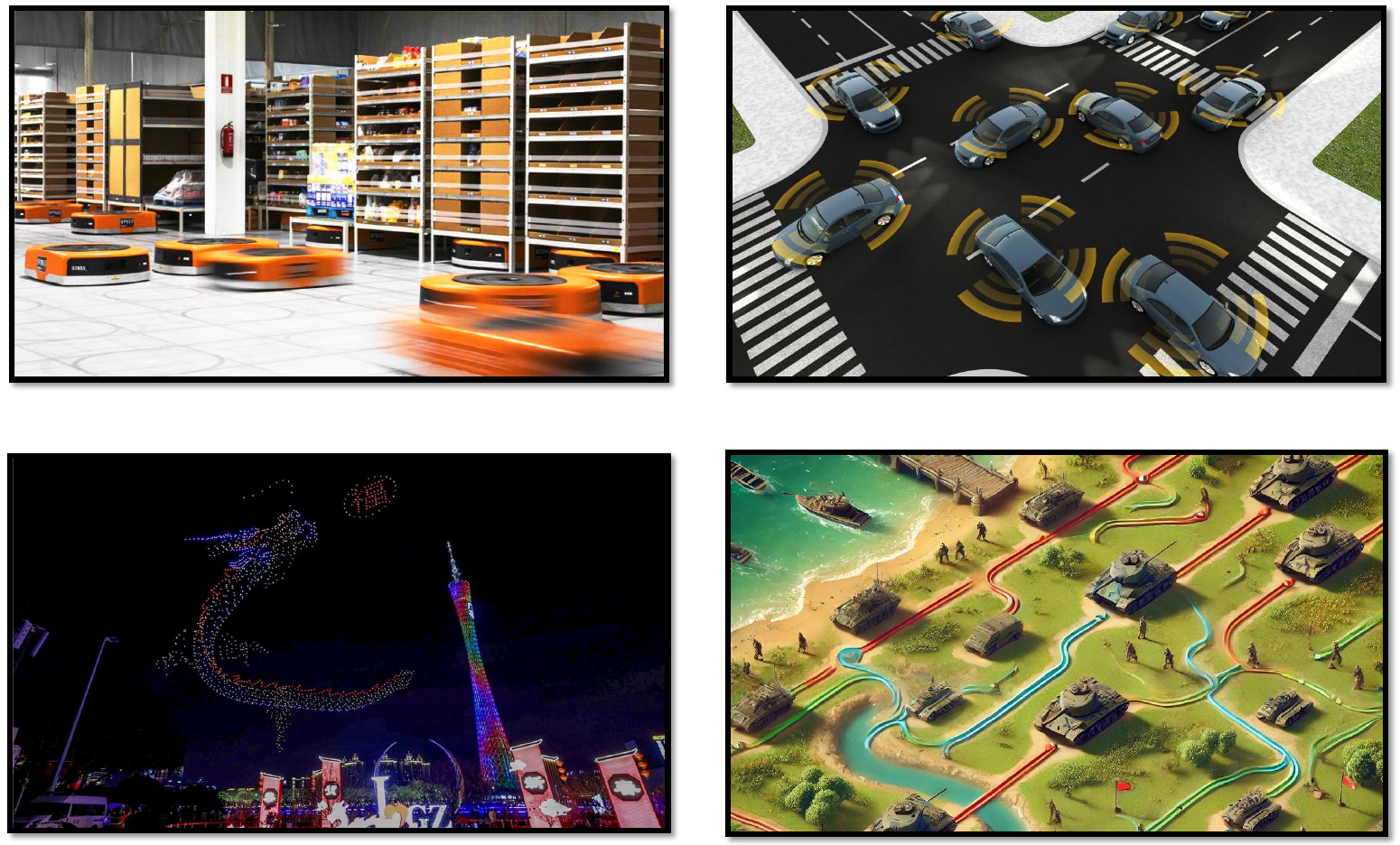}
             \small
             \put(22.5, 29.5) {(a)}
             \put(74.5, 29.5) {(b)}
             \put(22.5, -2.5) {(c)}
             \put(74.5, -2.5) {(d)}
        \end{overpic}
\vspace{-3mm}
    \caption[Applications of multi-robot path planning]{(a) Autonomous mobile robots navigate in Amazon autonomous fufillment warehouse, transporting goods between shelves; (b)  Self-driving cars at an intersection, with intelligent transportation control; (c) Drones form a dragon shape in the night sky; (d) Pathfinding for multiple units in RTS games}
    \label{fig:intro_applications}
\end{figure}%

Here, collision-free indicates that the robots must avoid colliding with each other and obstacles in the environment. The term labeled indicates that the robots are not interchangeable, i.e., each robot has a distinctive goal that it must reach. After a feasible motion plan is generated, the solution quality is often measured in terms of plan execution time and robot travel distance, which we usually want to minimize.
\gls{mpp} and its variations have been actively studied for decades~\cite{erdmann1987multiple,lavalle1998optimal,guo2002a,berg2008a,standley2011a,solovey2012a,turpin2014a,solovey2015a,wagner2015subdimensional,araki2017a,vedder2019a,khan2021a}. 

On one hand, many algorithms were proposed, and state-of-the-art methods perform well under certain metrics and situations. On the other hand, \gls{mpp}, in its many forms, has been shown to be computationally hard~\cite{hopcroft1984complexity,Solovey2016,yu2013structure}, which indicates that it is highly unlikely for any algorithm to consistently find optimal solutions for large and challenging problem instances. Observing the discrepancy between the current \gls{mpp} algorithm performance and the demand for performance in various \gls{mpp} applications, the statement of this thesis is
presented as:

Novel algorithms and heuristics can significantly improve the time efficiency and solution quality for solving different Multi-Robot Motion Planning variations, delivering more concrete and promising theoretical guarantees as well as higher throughput in Multi-Robot Motion Planning applications.

To address existing research gaps in \gls{mpp}, this dissertation advances both the theoretical foundations and practical applications of \gls{mpp}. We begin by examining one of its most fundamental variants: \gls{mpp} on graphs, also known as \gls{mapf}. This extensively studied problem models the shared workspace as a grid graph. Prior research has established that optimally solving \gls{mpp} on general graphs with respect to makespan and sum-of-costs objectives is NP-hard~\cite{yu2013structure,surynek2010optimization}. More recent studies further demonstrate that \gls{mpp} remains NP-hard even on empty grids without obstacles~\cite{demaine2019coordinated}.

Although polynomial-time algorithms exist that achieve constant-factor time optimality for \gls{mpp} on grids~\cite{demaine2019coordinated, yu2018constant}, their high constant factors often lead to significantly suboptimal solutions. Motivated by this limitation, we propose a novel polynomial-time algorithm inspired by prior work on the Grid Rearrangement Algorithm~\cite{szegedy2020rearrangement}. We show that our algorithm runs in polynomial time while achieving a lower constant-factor optimality ratio, resulting in more efficient solutions.

For \gls{mpp} on $m_1 \times m_2$ grids with full ($100\%$) robot density, we integrate the Grid Rearrangement Algorithm at the high level with a novel parallel odd-even sorting algorithm for low-level shuffling. This approach yields a polynomial-time algorithm that guarantees a makespan of $4(m_1 + 2m_2)$. For \gls{mpp} on grids with $1/2$ and $1/3$ robot densities, we introduce the highway and line-merge motion primitives, enabling us to achieve an asymptotic makespan optimality ratio of $1 + \frac{m_2}{m_1 + m_2}$. Furthermore, we extend our algorithm to high-dimensional grids.

\gls{mpp} on grids has numerous real-world applications, including warehouse automation, drone light shows, and automated parking systems. In many modern warehouses, infrastructure is designed to be \emph{well-formed}, meaning that mobile robots can rest at designated endpoints without obstructing others. Such structured layouts are prevalent in automated warehouses and parking lots, as they significantly simplify multi-robot path planning. Prior work~\cite{ma2016multi} has established that prioritized planning is complete for well-formed \gls{mpp} on graphs. 

In this dissertation, we systematically investigate well-formed layouts by modeling the problem as an optimization task: identifying the \emph{largest well-connected vertex set}. Efficient space utilization is crucial in dense, high-traffic environments, and finding the largest well-connected vertex set plays a vital role in optimizing throughput. We demonstrate that this problem is NP-hard and propose efficient algorithms to compute maximal and largest well-connected sets.

Beyond structured layouts, space efficiency has become an increasingly pressing concern, particularly in densely populated urban areas where land is scarce and expensive. While well-formed infrastructures simplify pathfinding, they inherently limit space utilization. To address this tradeoff, we introduce a \emph{puzzle-based automated parking and retrieval system} capable of supporting near $100\%$ robot density. To ensure efficient vehicle storage and retrieval, we develop novel \gls{mpp}-based algorithms tailored to this system.

Furthermore, we extend our study beyond traditional grid-based \gls{mpp}, which assumes point-like robots and neglects real-world motion constraints. We explore a continuous variant of \gls{mpp} where robots have non-trivial kinematics and occupy space. Specifically, we consider environments with \emph{Ackermann-steering car-like robots}, a crucial step toward enabling autonomous vehicle coordination in future smart cities. For this problem, we propose both an \emph{enhanced search-based centralized algorithm} and a \emph{decentralized prioritized algorithm}. Our centralized approach exhibits superior scalability, while the decentralized method achieves higher success rates and is more resilient to deadlocks.

A comprehensive literature review is provided in ~\ref{chap:background}. Our five key contributions are systematically presented in ~\ref{chap:rubik} to~\ref{chap:cars}, followed by concluding remarks in ~\ref{chap:conclusion}.

%% file: chapters/chapter2.tex
\chapter{Background}~\label{chap:background}
The field of \gls{mpp} has been extensively studied, particularly in graph-theoretic settings where the state space is discrete~\cite{hopcroft1984complexity,ErdLoz86,yu2016optimal,stern2019multi}.  The discrete graph-based MRPP version is also known as \gls{mapf} in AI communities. 
\gls{mpp} has immense practical value, especially in applications such as e-commerce~\cite{wurman2008coordinating,mason2019developing,wan2018lifelong, chan2024the}, which are poised for continued growth~\cite{dekhne2019automation,covid-auto, feng2021team, huang2022ICRA}
\section{Computational Intractability}
While the feasibility of \gls{mpp} has been positively established—being polynomial-time decidable on undirected graphs~\cite{KorMilSpi84} and on strongly biconnected directed graphs~\cite{Botea2018}—finding optimal solutions remains a significant challenge. Computing time- or distance-optimal solutions is NP-hard across various settings, including general graphs~\cite{goldreich2011finding,surynek2010optimization,yu2013structure}, planar graphs~\cite{yu2015intractability,banfi2017intractability}, directed graphs~\cite{Nebel2020}, and regular grids~\cite{demaine2019coordinated, Geft2022Refined,Geft2023Fine}.  Moreover, it has been proven that finding a bounded-suboptimal solution whose optimality is within a factor of \( \frac{4}{3} \) is intractable~\cite{ma2016optimal}.

Despite its computational hardness,  a variety of algorithmic solutions have been proposed to address \gls{mpp}. These solutions can be broadly classified as optimal or suboptimal.

\section{Optimal Solvers} 

Optimal solvers leverage either combinatorial search or problem reduction techniques. Search-based methods explore the joint solution space and include algorithms like EPEA*\cite{goldenberg2014enhanced}, ICTS\cite{sharon2013increasing}, CBS~\cite{sharon2015conflict}, M*\cite{wagner2015subdimensional}, and others. Reduction-based methods reformulate \gls{mpp} as another problem, such as SAT\cite{surynek2012towards}, answer set programming~\cite{erdem2013general}, or integer linear programming (ILP)~\cite{yu2016optimal}. Although effective, optimal solvers often exhibit limited scalability due to the inherent computational complexity of \gls{mpp}.

\section{Suboptimal Solvers}

To balance efficiency and solution quality, various suboptimal solvers have been developed. Unbounded solvers, such as Push-and-Swap~\cite{luna2011push}, Push-and-Rotate~\cite{de2014push}, Windowed Hierarchical Cooperative A*\cite{silver2005cooperative}, and BIBOX\cite{surynek2009novel}, efficiently compute feasible solutions but often at the cost of optimality.

Bounded-suboptimal search-based algorithms, including ECBS~\cite{barer2014suboptimal}, EECBS~\cite{li2021eecbs}, and their variants, guarantee solutions within a specified suboptimality bound, significantly improving scalability. Other approaches, such as DDM~\cite{han2020ddm}, PIBT~\cite{Okumura2019PriorityIW}, LNS~\cite{li2021anytime}, LACAM~\cite{okumura2023lacam}, and PBS~\cite{ma2019searching}, achieve impressive runtime performance while maintaining relatively good solution quality.

Additionally, learning-based methods~\cite{damani2021primal, sartoretti2019primal, li2021message, ma2021distributed, ma2022learning} demonstrate strong scalability, particularly in sparse environments. However, these methods lack completeness and optimality guarantees, and they tend to underperform in dense instances. Recent advances include \(O(1)\)-approximation algorithms~\cite{yu2018constant, demaine2019coordinated, han2018sear}, which offer theoretical guarantees but face practical limitations due to large constant factors. Parallel techniques aimed at accelerating existing \gls{mpp} algorithms have also been investigated~\cite{guo2021spatial, guo2024targeted}.

\section{Beyond Graph-Based MRPP---Applications}
\subsection{Well-Formed Infrastructures}

The concept of well-connected vertex sets draws inspiration from well-formed infrastructures~\cite{vcap2015complete}. In such infrastructures, including parking lots and fulfillment warehouses~\cite{wurman2008coordinating}, endpoints are designed to allow multiple robots (or vehicles) to navigate between them without causing complete obstructions. These endpoints may represent parking slots, pickup stations, or delivery stations. Many real-world infrastructures are intentionally constructed in this manner to facilitate efficient pathfinding and collision avoidance.

Planning collision-free paths for robots to move from their start positions to target destinations, commonly referred to as multi-robot path planning, is an NP-hard problem to solve optimally~\cite{surynek2010optimization, yu2013structure}. In practical applications, prioritized planning~\cite{erdmann1987multiple, silver2005cooperative} is among the most widely used approaches for finding collision-free paths. This method sequences robots by assigning them priorities and plans paths sequentially, ensuring that each robot avoids collisions with higher-priority robots. While effective in less cluttered settings, prioritized planning is generally incomplete and prone to deadlocks in dense environments.

Previous studies~\cite{vcap2015complete, ma2019searching} have shown that prioritized planning with arbitrary robot ordering can guarantee deadlock-free solutions in well-formed environments. In contrast, for problems that do not meet well-formed criteria, solutions may still be found using prioritized planning with a carefully selected priority order, as proposed in~\cite{ma2019searching}. However, identifying such an order can be computationally expensive or even infeasible.

To the best of our knowledge, no prior work has investigated how to efficiently design well-formed layouts that maximize workspace utilization.
\subsection{Puzzle-based Autonomous Parking System}
Efficient high-density parking has been widely studied. Centralized systems for self-driving vehicles~\cite{ferreira2014self, timpner2015k, nourinejad2018designing} increase capacity by up to $50\%$ but face complex retrieval challenges due to blockages and maneuverability requirements. For non-self-driving vehicles, robotic valet-based systems like the Puzzle-Based Storage system~\cite{Gue2007PuzzlebasedSS} are promising. PBS uses movable storage units (e.g., AGVs) on a grid with at least one empty cell (escort) for maneuvering, akin to solving the NP-hard 15-puzzle~\cite{ratner1986finding}. While optimal algorithms exist for single-vehicle retrieval~\cite{Gue2007PuzzlebasedSS, optimalMultiEscort}, average retrieval times remain high compared to aisle-based solutions. To improve efficiency, incorporating more escorts and I/O ports can enable simultaneous retrievals, leveraging advanced \gls{mpp} algorithms~\cite{okoso2022high}.

\gls{mpp} addresses finding collision-free paths for multiple robots with unique start and goal positions. Solving \gls{mpp} optimally~\cite{surynek2009novel, yu2013structure} is NP-hard. Optimal solvers~\cite{yu2016optimal, sharon2015conflict} are unsuitable for large problems due to scalability issues. Bounded suboptimal solvers~\cite{barer2014suboptimal} improve scalability but struggle in high-density settings. 
Polynomial-time methods~\cite{luna2011push} offer scalability at the cost of solution quality, while $O(1)$ time-optimal algorithms~\cite{Guo2022PolynomialTN, han2018sear} focus on minimizing makespan, making them less suitable for dynamic environments.

\subsection{Multi-Robot Path Planning with Shape and Kinematic Constraints}  
With the growing interest in autonomous driving, path planning for car-like robots has attracted increasing attention~\cite{Guo2023TowardEP,okoso2022high}.  
Many existing \gls{mpp} algorithms assume a holonomic point-agent model, which ignores robot shape and kinematic constraints, limiting their applicability to real-world scenarios.  
Several approaches have been proposed to bridge this gap.  
Graph-based solvers, such as CL-CBS for car-like robots~\cite{wen2022cl}, prioritized trajectory optimization methods~\cite{li2020efficient}, and sampling-based techniques~\cite{le2019multi} effectively handle these additional constraints.  
However, these methods are centralized, making them difficult to scale for large robot fleets and unsuitable for decentralized applications like self-driving cars.  
Decentralized approaches, such as B-ORCA~\cite{alonso2012reciprocal} and $\epsilon$CCA~\cite{alonso2018cooperative}, offer faster computation but often lack guarantees of successful navigation, particularly in dense environments.

%% file: chapters/chapter3.tex
\chapter{Expected 1.x Optimal Multi-Robot Path Planning}~\label{chap:rubik}
\section{Introduction}

Multi-Robot Path Planning on graphs is NP-hard to solve optimally, even on grids, suggesting no polynomial-time algorithms can compute exact optimal solutions for them. 
This raises a natural question: How optimal can polynomial-time algorithms reach? 
Whereas algorithms for computing constant-factor optimal solutions have been developed, the constant factor is generally very large, limiting their application potential. 
In this chapter, among other breakthroughs, we propose the first low-polynomial-time \gls{mpp} algorithms delivering $1$-$1.5$ (resp., $1$-$1.67$)  asymptotic makespan optimality guarantees for 2D (resp., 3D) grids for random instances at a very high $1/3$ robot density, with high probability. Moreover, when regularly distributed obstacles are introduced, our methods experience no performance degradation. These methods generalize to support $100\%$ robot density.  

Regardless of the dimensionality and density, our high-quality methods are enabled by a unique hierarchical integration of two key building blocks. At the higher level, we apply the labeled \gls{rta}, capable of performing efficient reconfiguration on grids through row/column shuffles. At the lower level, we devise novel methods that efficiently simulate row/column shuffles returned by \gls{rta}. 
Our implementations of \gls{rta}-based algorithms are highly effective in extensive numerical evaluations, demonstrating excellent scalability compared to other SOTA methods. For example, in 3D settings, \gls{rta}-based algorithms readily scale to grids with over $370,000$ vertices and over $120,000$ robots and consistently achieve conservative makespan optimality approaching $1.5$, as predicted by our theoretical analysis. 
We examine \gls{mpp}~\cite{stern2019multi} on two- and three-dimensional grids with potentially regularly distributed obstacles (see~\ref{fig:jd_center}).
The main objective of \gls{mpp} is to find collision-free paths for routing many robots from a start configuration to a goal configuration.  
In practice, solution optimality is of key importance, yet optimally solving \gls{mpp} is NP-hard~\cite{surynek2010optimization,yu2013structure}, even in planar settings~\cite{yu2015intractability} and on obstacle-less grids~\cite{demaine2019coordinated}. 
\gls{mpp} algorithms apply to a diverse set of practical scenarios, including formation reconfiguration~\cite{PodSuk04}, object transportation~\cite{RusDonJen95}, swarm robotics~\cite{preiss2017crazyswarm,honig2018trajectory}, to list a few.
Even when constrained to grid-like settings, \gls{mpp} algorithms still find many large-scale applications, including warehouse automation for general order fulfillment~\cite{wurman2008coordinating}, grocery order fulfillment~\cite{mason2019developing}, and parcel sorting~\cite{wan2018lifelong}. 

 \begin{figure}[!htbp]
        \centering
        \begin{overpic}               
        [width=1\linewidth]{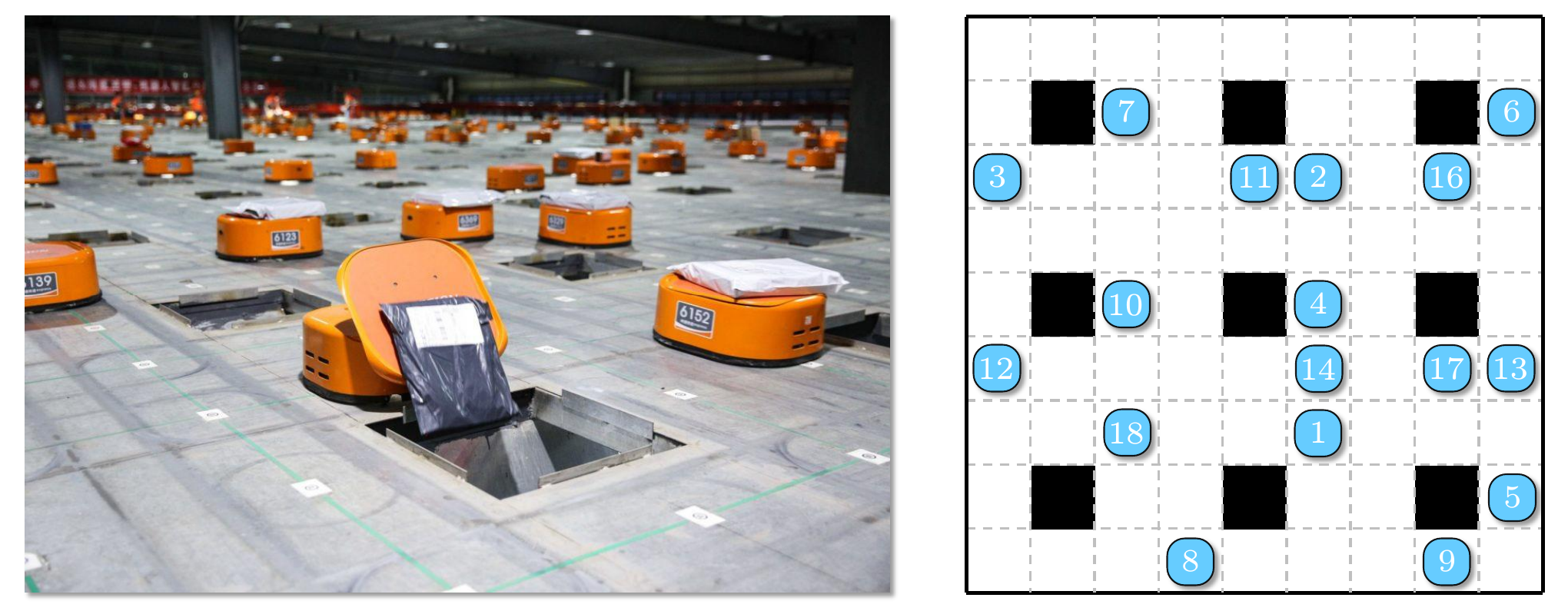}
             \footnotesize
             \put(28.5, -2.) {(a)}
             \put(78.5, -2.) {(b)}
        \end{overpic}
        \caption[Real world parcel sorting systems and our simulation environment example]{ (a) Real-world parcel sorting system (at JD.com) using many robots on a large grid-like environment with holes for dropping parcels; (b) A snapshot of a similar \gls{mpp} instance our methods can solve in polynomial-time with provable optimality guarantees. In practice, our algorithms scale to 2D maps of size over {$450 \times 300$}, supporting over $45$K robots and achieving better than $1.5$-optimality (see, e.g.,~\ref{fig:revise_third}). They scale further in 3D settings. 
        }\label{fig:jd_center}
    \end{figure}

Motivated by applications including order fulfillment and parcel sorting, we focus on grid-like settings with extremely high robot density, i.e., with $\frac{1}{3}$ or more graph vertices occupied by robots. Whereas recent studies~\cite{yu2018constant,demaine2019coordinated} show such problems can be solved in polynomial time yielding constant-factor optimal solutions, the constant factor is prohibitively high ($\gg 1$) to be practical.
In this research, we tear down this \gls{mpp} planning barrier, achieving $(1+\delta)$-makespan optimality asymptotically with high probability in polynomial time where $\delta \in (0, 0.5]$ for 2D grids and $\delta\in (0,\frac{2}{3}]$ for 3D grids.

This study's primary theoretical and algorithmic contributions are outlined below and summarized in ~\ref{tab: algorithm_summarize}. Through judicious applications of a novel row/column-shuffle-based global Grid Rearrangement (GR) method called the Rubik Table Algorithm~\cite{szegedy2023rubik}\footnote{It was noted in~\cite{szegedy2023rubik} that their Rubik Table Algorithm can be applied to solve \gls{mpp}; we do not claim this as a contribution of our work. Rather, our contribution is to dramatically bring down the constant factor through a combination of meticulous algorithm design and careful analysis.}, employing many classical algorithmic techniques, and combined with careful analysis, we establish that in \emph{low polynomial} time:
\begin{itemize}[leftmargin=3.5mm]
    \setlength\itemsep{1mm}
    \item For $m_1m_2$ robots on a 2D $m_1\times m_2$ grid, $m_1 \ge m_2$, i.e., at maximum robot density, \gls{rtmdd} (Grid Rearrangement for \gls{mpp}) computes a solution for an arbitrary \gls{mpp} instance under a makespan of $4m_1 + 8m_2$; For $m_1m_2m_3$ robots on a 3D $m_1\times m_2 \times m_3$ grids, $m_1 \ge m_2 \ge m_3$, \gls{rtmdd} computes a solution under a makespan of $4m_1+8m_2+8m_3$; 
    \item For $\frac{1}{3}$  robot density with uniformly randomly distributed start/goal configurations, \gls{rth} computes a solution with a makespan of $m_1 + 2m_2 + o(m_1)$ on 2D grids and $m_1+2m_2+2m_3+o(m_1)$ on 3D grids, with high probability. In contrast, such an instance has a minimum makespan of $m_1 + m_2 - o(m_1)$ for 2D grids and $m_1+m_2+m_3-o(m_1)$ for 3D grids with high probability. This implies that, as $m_1 \to \infty$, an asymptotic optimality guarantee of $\frac{m_1 + 2m_2}{m_1+m_2} \in (1, 1.5]$ is achieved for 2D grids and $\frac{m_1+2m_2+2m_3}{m_1+m_2+m_3}\in(1,\frac{5}{3}]$ for 3D grids, with high probability;
     \item For $\frac{1}{2}$ robot density, the same optimality guarantees as in the $\frac{1}{3}$ density setting can be achieved using \gls{rtlm} (Grid Rearrangements with Line Merge), using a merge-sort inspired robot routing heuristic; 
    \item Without modification, \gls{rth} achieves the same $\frac{m_1 + 2m_2}{m_1+m_2}$ optimality guarantee on 2D grids with up to $\frac{m_1m_2}{9}$ regularly distributed obstacles together with $\frac{2m_1m_2}{9}$ robots (e.g.,~\ref{fig:jd_center} (b)). Similar guarantees hold for 3D settings; 
    \item For robot density up to $\frac{1}{2}$, for an arbitrary (i.e., not necessarily random) instance, a solution can be computed with a makespan of $3m_1+4m_2+o(m_1)$ on 2D grids and $3m_1+4m_2+4m_3+o(m_1)$ on 3D grids. 
    \item With minor numerical updates to the relevant guarantees, the above-mentioned results also generalize to grids of arbitrary dimension $k \ge 4$. 
\end{itemize}

 Moreover, we have developed many effective and principled heuristics that work together with the \gls{rta}-based algorithms to further reduce the computed makespan by a significant margin. These heuristics include (1) An optimized matching scheme, required in the application of the 
 grid rearrangement algorithm,
 based on \gls{lba}, (2) A polynomial time path refinement method for compacting the solution paths to improve their quality.
As demonstrated through extensive evaluations, our methods are highly scalable and capable of tackling instances with tens of thousands of densely packed robots. Simultaneously, the achieved optimality matches/exceeds our theoretical prediction. With the much-enhanced scalability, our approach unveils a promising direction toward the development of practical, provably optimal multi-robot routing algorithms that run in low polynomial time. 

\begin{table}[h]
\scriptsize
  \centering
  \begin{tabular}{|c|c|c|c|c|}
    \hline
    \textbf{Algorithms} & \gls{rtmdd} 2$\times$3 &  \gls{rtmdd} 4$\times$2 & \gls{rtlm} &\gls{rth}  \\
\hline
\textbf{Applicable Density} & $\leq 1$ &$\leq 1$ &$\leq \frac{1}{2}$ & $\leq \frac{1}{3}$ \\
\hline
    \textbf{Makespan Upperbound} & $7(2m_2+m_1)$ & $4(2m_2+m_1)$ & $4m_2+3m_1$ & $4m_2+3m_1$ \\
    \hline
    \textbf{Asymptotic Makespan}&$7(2m_2+m_1)$ &$4(2m_2+m_1)$ &$2m_2+m_1+o(m_1)$ & $2m_2+m_1+o(m_1)$\\
    \hline
      \textbf{Asymptotic Optimality}&$7(1+\frac{m_2}{m_1+m_2})$ &$4(1+\frac{m_2}{m_1+m_2})$ &$1+\frac{m_2}{m_1+m_2}$ & $1+\frac{m_2}{m_1+m_2}$\\
    \hline
  \end{tabular}
  \caption[Summary results of the algorithms proposed in this chapter]{
Summary results of the algorithms proposed in this chapter. All algorithms operate on grids of dimensions $m_1\times m_2$.~\gls{rth}, \gls{rtmdd}, and \gls{rtlm} are all further derived from \gls{rtatwo}, each a variant characterized by distinct low-level shuffle movements. Specifically for \gls{rtmdd},  the low-level shuffle movement is tailored for facilitating full-density robot movement}\label{tab: algorithm_summarize}
\end{table}

 In this chapter, we provide a unified treatment of the problem that streamlined the description of \gls{rta}-based algorithms under 2D/3D/$k$D settings for journal archiving, and we further introduce many new results, including (1) The baseline \gls{rtmdd} method (for the full-density case) is significantly improved with a much stronger makespan optimality guarantee, by a factor of $\frac{7}{4}$; (2) A new and general polynomial-time path refinement technique is developed that significantly boosts the optimality of the plans generated by all \gls{rta}-based algorithms; (3) Complete and substantially refined proofs are provided for all theoretical developments in the paper; and (4) The evaluation section is fully revamped to reflect the updated theoretical and algorithmic development. 
%
%

{\textbf{Related work.}} 
Literature on multi-robot path and motion planning~\cite{hopcroft1984complexity,ErdLoz86} is expansive; here, we mainly focus on graph-theoretic (i.e., the state space is discrete) studies~\cite{yu2016optimal,stern2019multi}. As such, in this paper, \gls{mpp} refers explicitly to graph-based multi-robot path planning. 
Whereas the feasibility question has long been positively answered~\cite{KorMilSpi84}, the same cannot be said for finding optimal solutions, as computing time- or distance-optimal solutions are NP-hard in many settings, including on general graphs~\cite{goldreich2011finding,surynek2010optimization,yu2013structure}, planar graphs~\cite{yu2015intractability,banfi2017intractability}, and regular grids~\cite{demaine2019coordinated} that is similar to the setting studied in this chapter.  

Nevertheless, given its high utility, especially in e-commerce applications~\cite{wurman2008coordinating,mason2019developing,wan2018lifelong} that are expected to grow significantly~\cite{dekhne2019automation,covid-auto}, many algorithmic solutions have been proposed for optimally solving \gls{mpp}. 
Among these, combinatorial-search-based solvers~\cite{lam2019branch} have been demonstrated to be fairly effective.~\gls{mpp} solvers may be classified as being optimal or suboptimal. 
Reduction-based optimal solvers solve the problem by reducing the \gls{mpp} problem to another problem, e.g., SAT~\cite{surynek2012towards}, answer set programming~\cite{erdem2013general}, integer linear programming (ILP)~\cite{yu2016optimal}.
Search-based optimal \gls{mpp} solvers include  EPEA*~\cite{goldenberg2014enhanced}, ICTS~\cite{sharon2013increasing}, CBS~\cite{sharon2015conflict}, M*~\cite{wagner2015subdimensional}, and many others. 

Due to the inherent intractability of optimal \gls{mpp}, optimal solvers usually exhibit limited scalability, leading to considerable interest in suboptimal solvers.
Unbounded solvers like push-and-swap~\cite{luna2011push}, push-and-rotate~\cite{de2014push}, windowed hierarchical cooperative A$\*$~\cite{silver2005cooperative}, BIBOX~\cite{surynek2009novel}, all return feasible solutions very quickly, but at the cost of solution quality.
Balancing the running time and optimality is one of the most attractive topics in the study of \gls{mpp}. 
Some algorithms emphasize the scalability without sacrificing as much optimality, e.g., ECBS~\cite{barer2014suboptimal}, DDM~\cite{han2020ddm}, PIBT~\cite{Okumura2019PriorityIW}, PBS~\cite{ma2019searching}. There are also learning-based solvers~\cite{damani2021primal,sartoretti2019primal,li2021message} that scale well in sparse environments. Effective orthogonal heuristics have also been proposed~\cite{guo2021spatial}. 
Recently, $O(1)$-approximate or constant factor time-optimal algorithms have been proposed, e.g.~\cite{yu2018constant,demaine2019coordinated,han2018sear}, that tackle highly dense instances. However, these algorithms only achieve a low-polynomial time guarantee at the expense of very large constant factors, rendering them theoretically interesting but impractical. 

In contrast, with high probability, our methods run in low polynomial time with provable $1.x$ asymptotic optimality. To our knowledge, this paper presents the first \gls{mpp} algorithms to simultaneously guarantee polynomial running time and $1.x$ solution optimality, which works for any dimension $\ge 2$.

\textbf{Organization.} The rest of the paper is organized as follows.
In~\ref{sec:pre}, starting with 2D grids, we provide a formal definition of graph-based \gls{mpp}, and introduce the Grid Rearrangement problem and the associated baseline algorithm (\gls{rta}) for solving it.
\gls{rtmdd}, a basic adaptation of \gls{rta} for \gls{mpp} at maximum robot density which ensures a makespan upper bound of $4m_1 + 8m_2$, is described in ~\ref{sec:1:1}. An accompanying lower bound of $m_1 + m_2 - o(m_1)$ for random \gls{mpp} instances is also established. 
In~\ref{sec:1:2} we introduce \gls{rth} for $\frac{1}{3}$ robot density achieving a makespan upper bound of $m_1 + 2m_2 + o(m_1)$. Obstacle support is then discussed. 
We continue to show how $\frac{1}{2}$ robot density may be supported with similar optimality guarantees. 
In~\ref{sec: three_d}, we generalize the algorithms to work on 3$+$D grids.
In~\ref{sec:opt-boost}, we introduce multiple optimality-boosting heuristics to significantly improve the solution quality for all variants of Grid Rearrangement-based solvers.
We thoroughly evaluate the performance of our methods in~\ref{sec:eval} and conclude with~\ref{sec:conclusion}.

\section{Preliminaries}\label{sec:pre}
\subsection{Multi-Robot Path Planning on Graphs}
Consider an undirected graph $\mathcal G(V, E)$ and $n$ robots with start configuration $S = \{s_1, \dots, s_n\} \subseteq V$ and goal configuration $G = \{g_1, \dots, g_n\} \subseteq V$.
A {\em{path}} for robot $i$ is a map 
$P_i: \mathbb {N} \to V$ where $\mathbb N$ is the set of non-negative integers. 
A feasible $P_i$ must be a sequence of vertices that connects $s_i$ and $g_i$: 
(1) $P_i(0) = s_i$;
(2) $\exists T_i \in \mathbb N$, s.t. $\forall t \geq T_i, P_i(t) = g_i$;
(3) $\forall t > 0$, $P_i(t) = P_i(t - 1)$ or $(P_i(t), P_i(t - 1)) \in E$. 
A path set $\{P_1, \ldots, P_n\}$ is feasible iff each $P_i$ is feasible and for all $t \ge 0$ and $1\le i < j \le n$, $P_i(t)=P$, it does not happen that: (1) $P_i(t) = P_j(t)$; (2) $P_i(t) = P_j(t+1) \wedge P_j(t) = P_i(t+1)$.

We work with $\mathcal G$ being $4$-connected grids in 2D and $6$-connected grids in 3D, aiming to mainly minimize the \emph{makespan}, i.e., $\max_i\{|{P}_i|\}$ (later, a sum-of-cost objective is also briefly examined).
Unless stated otherwise, $\mathcal G$ is assumed to be an $m_1 \times m_2$ grid with $m_1 \ge m_2 $ in 2D and $m_1\times m_2\times m_3$ grid with $m_1\ge m_2\ge m_3$ in 3D. As a note, ``randomness'' in this paper always refers to uniform randomness. 
The version of \gls{mpp} we study is sometimes referred to as \emph{one-shot} \gls{mpp}~\cite{stern2019multi}. We mention our results also translate to guarantees on the life-long setting~\cite{stern2019multi}, briefly discussed in ~\ref{sec:conclusion}. 

\vspace{-2mm}
\subsection{\gls{rtp}}
The Grid Rearrangement Problem (\gls{rtp}) (first proposed in~\cite{szegedy2023rubik} as the Rubik Table problem) formalizes the task of carrying out globally coordinated object reconfiguration operations on lattices, with many interesting applications. 
The problem has many variations; we start with the 2D form, to be generalized to higher dimensions later.

\begin{problem}[{\normalfont\bf \gls{rtptwo}}~\cite{szegedy2023rubik}]
Let $M$ be an $m_1 (row) \times m_2 (column)$ table, $m_1 \ge m_2$, containing $m_1m_2$ items, one in each table cell. 
The items have $m_2$ colors with each color having a multiplicity of $m_1$.
In a \emph{shuffle} operation, the items in a single column or a single row of $M$ may be permuted in an arbitrary manner. 
Given an arbitrary configuration $X_I$ of the items, find a sequence of shuffles that take $M$ from $X_I$ to the configuration where
row $i$, $1 \leq i \leq m_1$, contains only items of color $i$. The problem may also be \emph{labeled}, i.e., each item has a unique label in $1, \ldots, m_1m_2$.
\end{problem}

A key result~\cite{szegedy2023rubik}, which we denote here as the (labeled) Grid Rearrangement Algorithm in 2D (\gls{rtatwo}), shows that a colored \gls{rtptwo} can be solved using $m_2$ column shuffles followed by $m_1$ row shuffles, implying a low-polynomial time computation time.
Additional $m_1$ row shuffles then solve the labeled \gls{rtptwo}. We illustrate how \gls{rtatwo} works on an $m_1 \times m_2$ table with $m_1 =4$ and $m_2 =3$ (\ref{fig:rubik}); we refer the readers to~\cite{szegedy2023rubik} for details. The main supporting theoretical result is given in~\ref{t:rta}, which depends on~\ref{t:hall}.
%

\begin{theorem}[Hall’s Matching theorem with parallel edges~\cite{hall2009representatives,szegedy2023rubik}]\label{t:hall}
Let $B$ be a $d$-regular ($d$ > 0) bipartite graph on $n + n$ nodes,
possibly with parallel edges. Then $B$ has a perfect
matching.
\end{theorem}

\begin{theorem}[{\normalfont\bf Grid Rearrangement Theorem~\cite{szegedy2023rubik}}]\label{t:rta}
An arbitrary Grid Rearrangement problem on an $m_1\times m_2$ table can be solved using $m_1 + m_2$ shuffles. The labeled Grid Rearrangement problem can be solved using $2m_1 + m_2$ shuffles.
\end{theorem}

\gls{rtatwo} operates in two phases. In the first, a bipartite graph $B(T, R)$ is constructed based on the initial table where the bipartite set $T$ are the colors/types of items, and the set $R$ the rows of the table  (see~\ref{fig:rubik} (b)). An edge is added to $B$ between $t \in T$ and $r \in R$ for every item of color $t$ in row $r$. 
From $B(T,R)$, which is a \emph{regular bipartite graph}, $m_2$ \emph{perfect matchings} can be computed as guaranteed by~\ref{t:hall}. Each matching, containing $m_1$ edges, dictates how a table column should look like after the first phase. For example, the first set of matching (solid lines in~\ref{fig:rubik} (b)) says the first column should be ordered as yellow, cyan, red, and green, shown in~\ref{fig:rubik} (c). 
After all matchings are processed, we get an intermediate table,~\ref{fig:rubik} (c). Notice each row of~\ref{fig:rubik} (a) can be shuffled to yield the corresponding row of~\ref{fig:rubik} (c); a key novelty of \gls{rtatwo}.
After the first phase of $m_1$ row shuffles, the intermediate table (\ref{fig:rubik} (c))  can be rearranged with $m_2$ column shuffles to solve the colored \gls{rtptwo} (\ref{fig:rubik} (d)). Another $m_1$ row shuffles then solve the labeled \gls{rtptwo} (\ref{fig:rubik} (e)). It is also possible to perform the rearrangement using $m_2$ column shuffles, followed by $m_1$ row shuffles, followed by another $m_2$ column shuffles. 

\begin{figure}[h]
        \centering
        \begin{overpic}[width=\linewidth]{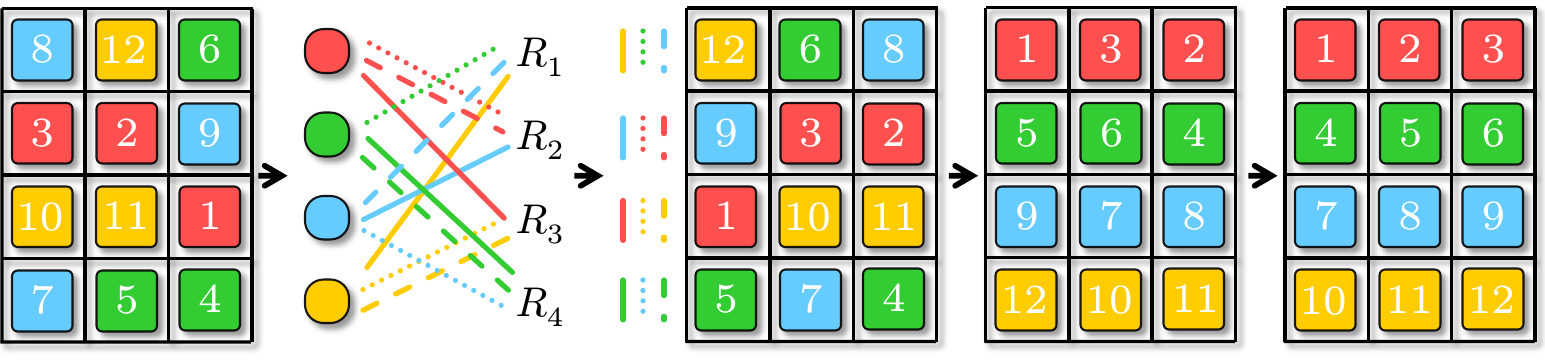}
             \footnotesize
             \put(6.5,  -2) {(a)}
             \put(26,  -2) {(b)}
             \put(50.5,  -2) {(c)}
             \put(70,  -2) {(d)}
             \put(89.5,  -2) {(e)}
        \end{overpic}
        \vspace{2mm}
        \caption[Illustration of applying the \emph{$11$ shuffles}]{Illustration of applying the \emph{$11$ shuffles}: (a) The initial $4\times 3$ table with a random arrangement of 12 items that are colored and labeled. The labels are consistent with the colors; (b) The constructed bipartite graph. It contains $3$ perfect matchings, determining the $3$ columns in (c); only color matters in this phase; (c) Applying $4$ row shuffles to (a), according to the matching results, leads to an intermediate table where each column has one color appearing exactly once; (d) Applying $3$ column shuffles to (c) solves a colored \gls{rtptwo}; (e) $4$ additional row shuffles fully sort the labeled items.}\label{fig:rubik}
    \end{figure}
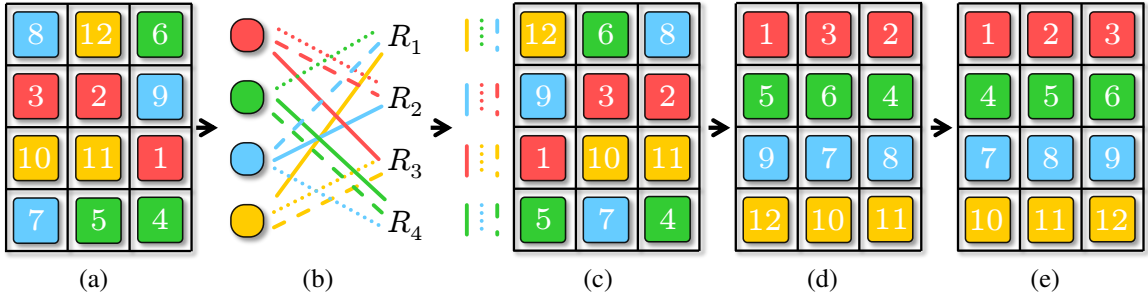

\gls{rtatwo} runs in $O(m_1m_2\log m_1)$ (notice that this is nearly linear with respect to $n = m_1m_2$, the total number of items) expected time or $O(m_1^2m_2)$ deterministic time. 

\section{Solving MRPP at Maximum Density Leveraging RTA2D}\label{sec:1:1}
The built-in global coordination capability of \gls{rtatwo} naturally applies to solving makespan-optimal \gls{mpp}.
Since \gls{rtatwo} only requires \emph{three rounds} of shuffles and each round involves either parallel row shuffles or parallel column shuffles, if each round of shuffles can be realized with makespan proportional to the size of the row/column, then a makespan upper bound of $O(m_1 + m_2)$ can be guaranteed. 
This is in fact achievable even when all of $\mathcal G$'s vertices are occupied by robots, by recursively applying a \emph{labeled line shuffle algorithm}~\cite{yu2018constant}, which can arbitrarily rearrange a line of $m$ robots embedded in a grid using $O(m)$ makespan. 
\begin{lemma}[Basic Simulated Labeled Line Shuffle~\cite{yu2018constant}]\label{l:line-shuffle} For $m$ labeled robots on a straight path of length $m$, embedded in a 2D grid, they may be arbitrarily ordered in $O(m)$ steps. Moreover, multiple such reconfigurations can be performed on parallel paths within the grid. 
\end{lemma}

The key operation is based on a localized, $3$-step pair swapping routine, shown in~\ref{fig:figure8}. For more details on the line shuffle routine, see~\cite{yu2018constant}. 

\begin{figure}[h!]
        \centering
        \includegraphics[width=\linewidth]{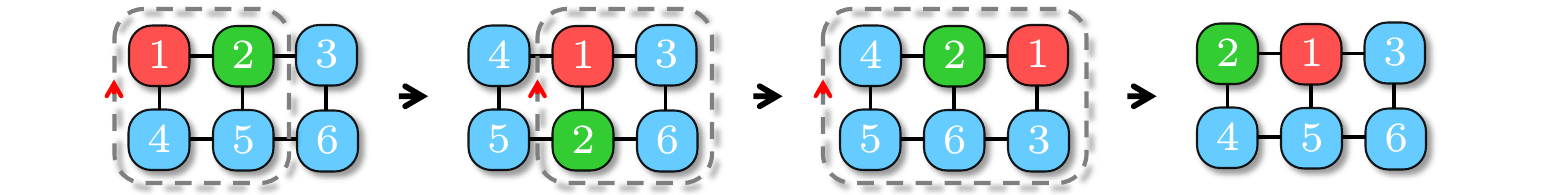}
\caption{On a $2 \times 3$ grid, swapping two robots may be performed in three steps with three cyclic rotations.}\label{fig:figure8}
\end{figure}
    
However, the basic simulated labeled line-shuffle algorithm has a large constant factor. 
Each shuffle takes 3 steps; doing arbitrary shuffling of a $2 \times 3$ takes $20+$ steps in general. The constant factor further compounds as we coordinate the shuffles across multiple lines. 
Borrowing ideas from \emph{parallel odd-even sort}~\cite{bitton1984taxonomy}, we can greatly reduce the constant factor in ~\ref{l:line-shuffle}. We will do this in several steps. First, we need the following lemma. By an \emph{arbitrary horizontal swap} on a grid, we mean an arbitrary reconfiguration of a grid row.

\begin{lemma}\label{l:reconfigure}
It takes at most $7$, $6$, $6$, $7$, $6$, and $8$ steps to perform arbitrary combinations of arbitrary horizontal swaps on $3\times 2$, $4\times 2$, $2 \times 3$, $3\times 3$, $2\times 4$, and $3\times 4$ grids, respectively.
\end{lemma}
\begin{proof}
Using integer linear programming~\cite{yu2016optimal}, we exhaustively compute makespan-optimal solutions for arbitrary horizontal reconfigurations on $3\times 2$ ($8 = 2^3$ possible cases), $4 \times 2$ grids ($2^4$ possible cases), $2 \times 3$ grids ($6^2$ possible cases), $3 \times 3$ grids ($6^3$ possible cases), $2 \times 4$ grids ($24^2$ possible cases), and $3 \times 4$ grids ($24^3$ possible cases), which confirms the claim. 
\end{proof}
As an example, it takes seven steps to horizontally ``swap'' all three pairs of robots on a $3\times 2$ grid, as shown in~\ref{fig:six}. 

\begin{figure}[htb]
        \centering
        \includegraphics[width=1\linewidth]{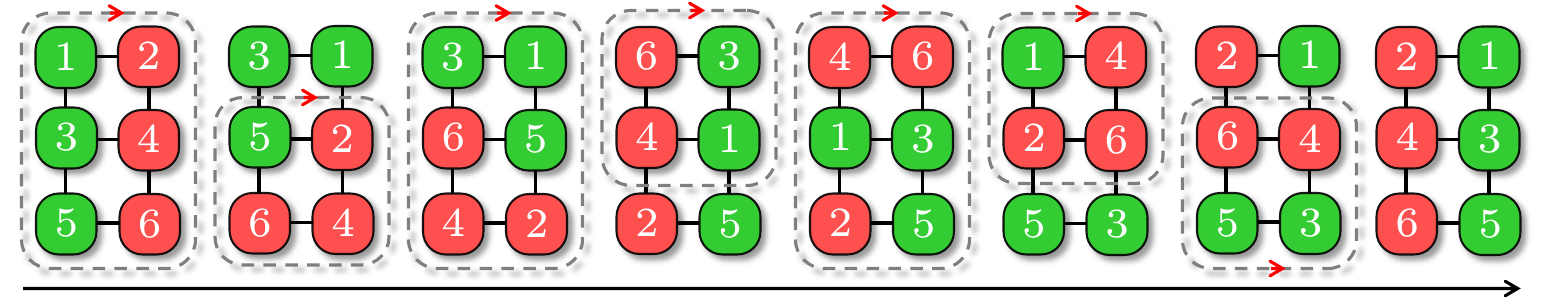}
\caption{An example of a full horizontal ``swap'' on a $3\times 2$ grid that takes seven steps, in which all three pairs are swapped. 
}\label{fig:six}
\end{figure}

\begin{lemma}[Fast Line Shuffle]\label{l:fast-line-shuffle}
Embedded in a 2D grid, $m$ robots on a straight path of length $m$ may be arbitrarily ordered in $7m$ steps. Moreover, multiple such reconfigurations can be executed in parallel within the grid. 
\end{lemma}
\begin{proof}
Arranging $m$ robots on a straight path of length $m$ may be realized using parallel odd-even sort~\cite{bitton1984taxonomy} in $m-1$ rounds, which only requires the ability to simulate potential pairwise ``swaps'' interleaving odd phases (swapping robots located at positions $2k + 1$ and $2k + 2$ on the path for some $k$) and even phases (swapping robots located at positions $2k + 2$ and $2k + 3$ on the path for some $k$). Here, it does not matter whether $m$ is odd or even. 
To simulate these swaps, we can partition the grid embedding the path into $3 \times 2$ grids in two ways for the two phases, as illustrated in~\ref{fig:odd-even}. 
\begin{figure}[h!]
        \centering
        \includegraphics[width=1\linewidth]{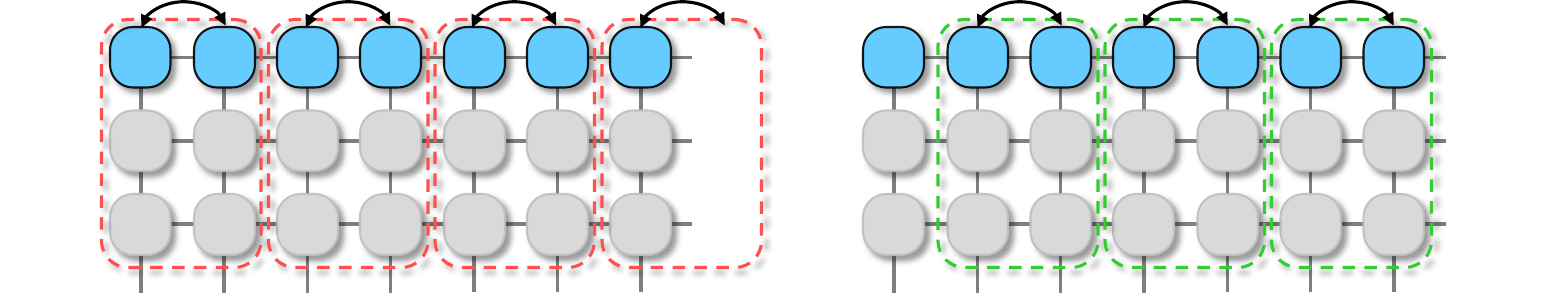}
\caption{Partitioning a grid into disjoint $3 \times 2$ grids in two ways for simulating odd-even sort. Highlighted robot pairs may be independently ``swapped'' within each $3\times 2$ grid as needed.} 
        \label{fig:odd-even}
\end{figure}

A perfect partition requires that the second dimension of the grid, perpendicular to the straight path, be a multiple of $3$. If not, some partitions at the bottom can use $4 \times 2$ grids. By ~\ref{l:reconfigure}, each odd-even sorting phase can be simulated using at most $7m$ steps. Shuffling on parallel paths is directly supported. 
\end{proof}

After introducing a $7m$ step line shuffle, we further show how it can be dropped to $4m$, using similar ideas. The difference is that an updated parallel odd-even sort will be used with different sub-grids of sizes different from $2\times 3$ and $2\times 4$. 

The updated parallel odd-even sort operates on blocks of \emph{four} (\ref{fig:odd-even-2}) instead of \emph{two} (\ref{fig:odd-even}), which cuts down the number of parallel sorting operations from $m-1$ to about $m/2$. 
Here, if $m$ is not even, a partition will leave either $1$ or $3$ at the end. For example, if $m= 11$, it can be partitioned as $4, 4, 3$ and $2, 4, 4, 1$ in the two parallel sorting phases.  

\begin{figure}[h]
        \centering
        \includegraphics[width=0.9\linewidth]{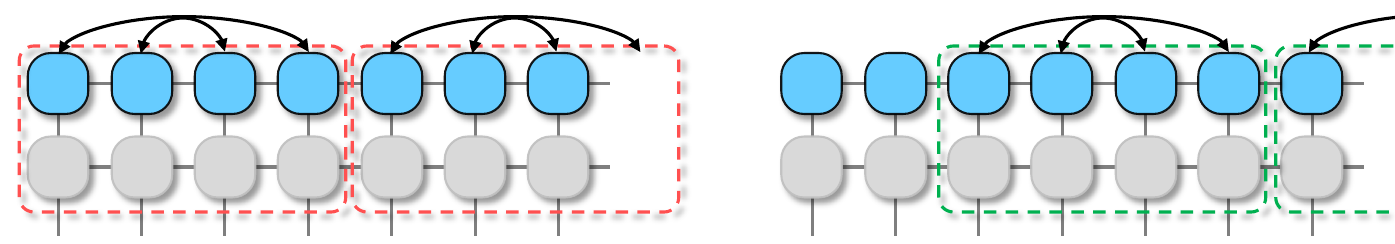}
\caption{We can make the parallel odd-even sort work twice faster by increasing the swap block size from two to four.}\label{fig:odd-even-2}
    \end{figure}

With the updated parallel odd-even sort, we must be able to make swaps on blocks of four. We do this by partition an $m_1\times m_2$ grid into $2\times 4$, which may have leftover sub-grids of sizes $2\times 3$, $3\times 4$, and $3 \times 3$. 
Using the same reasoning in proving ~\ref{l:fast-line-shuffle} and with ~\ref{l:reconfigure}, we have

\begin{lemma}[Faster Line Shuffle]\label{l:faster-line-shuffle}
Embedded in a 2D grid, $m$ robots on a straight path of length $m$ may be arbitrarily ordered in approximately $4m$ steps. Moreover, multiple such reconfigurations can be executed in parallel within the grid. 
\end{lemma}
\begin{proof}[Proof sketch]
The updated parallel odd-even sort requires a total of $m/2$ steps. Since each step operated on a $2\times 4$, $2 \times 3$, $3 \times 4$, or $3\times 3$ grid, which takes at most $8$ steps, the total makespan is approximately $4m$. 
\end{proof}

Combining \gls{rtatwo} and fast line shuffle (~\ref{l:faster-line-shuffle}) yields a polynomial time \gls{mpp} algorithm for fully occupied grids with a makespan of $4m_1 + 8m_2$. 

\begin{theorem}[\gls{mpp} on Grids under Maximum Robot Density, Upper Bound]\label{t:rtm-makespan}
\gls{mpp} on an $m_1\times m_2$ grid, $m_1 \ge m_2 \ge 3$, with each grid vertex occupied by a robot, can be solved in polynomial time in a makespan of $4m_1 + 8m_2$.
\end{theorem}

\begin{proof}[Proof Sketch]
By combining the \gls{rtatwo} and the line shuffle algorithms, the problem can be efficiently solved through two row-shuffle phases and one column-shuffle phase. During the row-shuffle phases, all rows can be shuffled in $4m_2$ steps, and similarly, during column-shuffle phases, all columns can be shuffled in $4m_1$ steps. Summing up these steps, the entire problem can be addressed in $4m_1 + 8m_2$ steps. The primary computational load lies in computing the perfect matchings, a task achievable in $O(m_1m_2\log m_1)$ expected time or $O(m_1^2m_2)$ deterministic time.
\end{proof}

It is clear that, by exploiting the idea further, smaller makespans can potentially be obtained for the full-density setting, but the computation required for extending  ~\ref{l:reconfigure} will become more challenging. It took about two days to compute all $24^3$ cases for the $3\times 4$ grid, for example.
We call the resulting algorithm from the process \gls{rtmdd} (for 2D), with ``M'' denoting \emph{maximum density}, regardless of the used sub-grids. The straightforward pseudo-code is given in~\ref{alg:rubik}. The comments in the main \gls{rtmdd} routine indicate the corresponding \gls{rtatwo} phases. For \gls{mpp} on an $m_1 \times m_2$ grid with row-column coordinates $(x, y)$, we say robot $i$ belongs to color $1 \le j \le m_1$ if $g_i.y=j$.
Function $\texttt{Prepare()}$ in the first phase finds intermediate states $\{\tau_i\}$ for each robot through perfect matchings and routes them towards the intermediate states by (simulated) column shuffles. 
If the robot density is smaller than required, we may fill the table with ``virtual'' robots~\cite{han2018sear,yu2018constant}.
For each robot $i$, we have $\tau_i.y=s_i.y$. 
Function $\texttt{ColumnFitting()}$ in the second phase routes the robots to their second intermediate states $\{\mu_i\}$ through row shuffles where $\mu_{i}.x=\tau_{i}.x$ and $\mu_i.y=g_i.y$. 
In the last phase, function $\texttt{RowFitting()}$ routes the robots to their final goal positions using additional column shuffles.

\begin{algorithm}
\DontPrintSemicolon
\SetKwProg{Fn}{Function}{:}{}
\SetKwFunction{Fprepare}{Prepare}
\SetKwFunction{Fcolumnfitting}{ColumnFitting}
\SetKwFunction{Frowfitting}{RowFitting}
\SetKwFunction{FRtMapf}{RTM}
  \caption{Labeled Grid Rearrangement Based \gls{mpp} Algorithm for 2D (\gls{rtatwo})\label{alg:rubik}}
  \KwIn{Start and goal vertices $S=\{s_i\}$ and $G=\{g_i\}$}
  \Fn{\textsc{RTA2D}($S,G$)}{
$\texttt{Prepare}(S,G)$ \quad\quad\quad\hspace{-0.5mm} \Comment{Computing~\ref{fig:rubik} (b)}\;
$\texttt{ColumnFitting}(S,G)$ \Comment{\ref{fig:rubik} (a) $\to$~\ref{fig:rubik} (c)}\;
$\texttt{RowFitting}(S,G)$ \quad\quad\hspace{-1.5mm} \Comment{\ref{fig:rubik} (c) $\to$~\ref{fig:rubik} (d)}\;
}
\Fn{\Fprepare{$S,G$}}{
    $A\leftarrow [1,\ldots,m_1m_2]$\;
  \For{$(t,r)\in [1,\ldots,m_1]\times[1,\ldots,m_1]$}{
  \If{$\exists i\in A$ where $s_i.x=r\wedge g_i.y=t$}
  {
 add edge $(t,r)$ to $B(T,R)$\;
 remove $i$ from $A$ \;
  }
    }
 compute matchings $\mathcal{M}_1,\ldots,\mathcal{M}_{m_2}$ of $B(T, R)$\;
 $A\leftarrow [1,\ldots,m_1m_2]$\;
 \ForEach{$\mathcal{M}_r$ and $(t,r)\in \mathcal{M}_r$}{
 \If{$\exists i\in A$ where $s_i.x=r\wedge g_i.y=t$}{
 $\tau_i\leftarrow (r, s_i.y)$ and remove $i$ from $A$\;
  mark robot $i$ to go to $\tau_i$\;
 }
 }
 perform simulated column shuffles in parallel 
 }
 \vspace{1mm}
  \Fn{\Fcolumnfitting{$S,G$}}{
        \ForEach{$i\in [1,\ldots,m_1m_2]$}{
        $\mu_i\leftarrow (\tau_i.x,g_i.y)$ and mark robot $i$ to go to $\mu_i$\;
        }
  perform simulated row shuffles in parallel      
  }
\vspace{1mm}
  \Fn{\Frowfitting{$S,G$}}{
        \ForEach{$i\in [1,\ldots,m_1m_2]$}{
        mark robot $i$ to go to $g_i$\;
        }
  perform simulated column shuffles in parallel 
}
\end{algorithm}

We now establish the optimality guarantee of \gls{rtmdd} on 2D grids, assuming \gls{mpp} instances are randomly generated. 
For this, a precise lower bound is needed. 

\begin{proposition}[Precise Makespan Lower Bound of \gls{mpp} on 2D Grids]\label{p:makespan-lower}
The minimum makespan of random \gls{mpp} instances on an $m_1 \times m_2$ grid with $\Theta(m_1m_2)$ robots is $m_1 + m_2 - o(m_1)$ with arbitrarily high probability as $m_1\to \infty$.
\end{proposition}
\begin{proof}
Without loss of generality, let the constant in $\Theta(m_1m_2)$ be some $c \in (0, 1]$, i.e., there are $cm_1m_2$ robots. 
We examine the top left and bottom right corners of the $m_1 \times m_2$ grid $\mathcal G$. In particular, let $\gtl$ (resp.,  $\gbr$) be the top left (resp., bottom right) $\alpha m_1\times \alpha m_2$ sub-grid of $\mathcal G$, for some positive constant $\alpha \ll 1$.
For $u \in V(\gtl)$ and $v \in V(\gbr)$, assuming each grid edge has unit distance, then the Manhattan distance between $u$ and $v$ is at least $(1-2\alpha)(m_1 + m_2)$. 
Now, the probability that some $u  \in V(\gtl)$ and $v \in V(\gbr)$ are the start and goal, respectively, for a single robot, is $\alpha^4$. For $cm_1m_2$ robots, the probability that at least one robot's start and goal fall into $\gtl$ and $\gbr$, respectively, is $p =1 - (1 - \alpha^4)^{cm_1m_2}$. 

Because $(1 - x)^y < e^{-xy}$ for $0 < x < 1$ and $y > 0$ {\footnote{This is because $\log(1-x) < -x$ for $0 < x < 1$}; 
multiplying both sides by a positive $y$ and exponentiate with base $e$ then yield the inequality.}, $p > 1 - e^{-\alpha^4cm_1m_2}$. 
Therefore, for arbitrarily small $\alpha$, we may choose $m_1$ such that $p$ is arbitrarily close to $1$. 
For example, we may let $\alpha = m_1^{-\frac{1}{8}}$, which decays to zero as $m_1 \to \infty$, then it holds that the makespan is $(1 - 2\alpha)(m_1 + m_2) = m_1 + m_2 - 2m_1^{-\frac{1}{8}}(m_1 + m_2) = m_1 + m_2 - o(m_1)$ with probability $p > 1 - e^{-c\sqrt{m_1}m_2}$. 
\end{proof}

Comparing the upper bound established in~\ref{t:rtm-makespan} and the lower bound from Proposition~\ref{p:makespan-lower} immediately yields

\begin{theorem}[Optimality Guarantee of \gls{rtmdd}]\label{t:rtmapf}
For random \gls{mpp} instances on an $m_1 \times m_2$ grid with $\Omega(m_1m_2)$ robots, as $m_1\to \infty$, \gls{rtmdd} computes in polynomial time solutions that are $4(1 + \frac{m_2}{m_1+m_2})$-makespan optimal, with high probability.
\end{theorem}

\begin{proof}[Proof Sketch]
By Proposition~\ref{p:asymptotic_lowerbound} and~\ref{t:rtm-makespan}, the asymptotic optimality ratio is $4 \left(1 + \frac{m_2}{m_1+m_2}\right)$.
\end{proof}

 \gls{rtmdd} always runs in polynomial time and has the same running time as \gls{rtatwo}; the high probability guarantee only concerns solution optimality. 
 The same is true for other high-probability algorithms proposed in this paper. We also note that high-probability guarantees imply guarantees in expectation. 

\section{Near-Optimally Solving MRPP with up to One-Third and One-Half Robot Densities}\label{sec:1:2}
Though \gls{rtmdd} runs in polynomial time and provides constant factor makespan optimality in expectation, the constant factor is $4+$ due to the extreme density. In practice, a robot density of around $\frac{1}{3}$ (i.e., $n = \frac{m_1m_2}{3}$) is already very high. As it turns out,  with $n = cm_1m_2$ for some constant $c > 0$ and $n \le \frac{m_1m_2}{2}$, the constant factor can be dropped to close to 1.  

\subsection{Up to One-Third Density: Shuffling with Highway Heuristics}\label{subsec:1:3}
For $\frac{1}{3}$ density, we work with random \gls{mpp} instances in this subsection; arbitrary instances for up to $\frac{1}{2}$ density are addressed in later subsections.
Let us assume for the moment that $m_1$ and $m_2$ are multiples of three; we partition $\mathcal G$ into $3\times 3$ cells (see, e.g.,~\ref{fig:jd_center} (b) and~\ref{fig:example}).
We use~\ref{fig:example}, where~\ref{fig:example} (a) is a random start configuration and~\ref{fig:example} (f) is a random goal configuration, as an example to illustrate \gls{rth} --- Grid Rearrangement with Highways, targeting robot density up to $\frac{1}{3}$. 
\gls{rth} has two phases: \emph{unlabeled reconfiguration} and \emph{\gls{mpp} resolution with Grid Rearrangement and highway heuristics}. 

\begin{figure}[htb]
\centering

        

              \begin{overpic}[width=\linewidth]{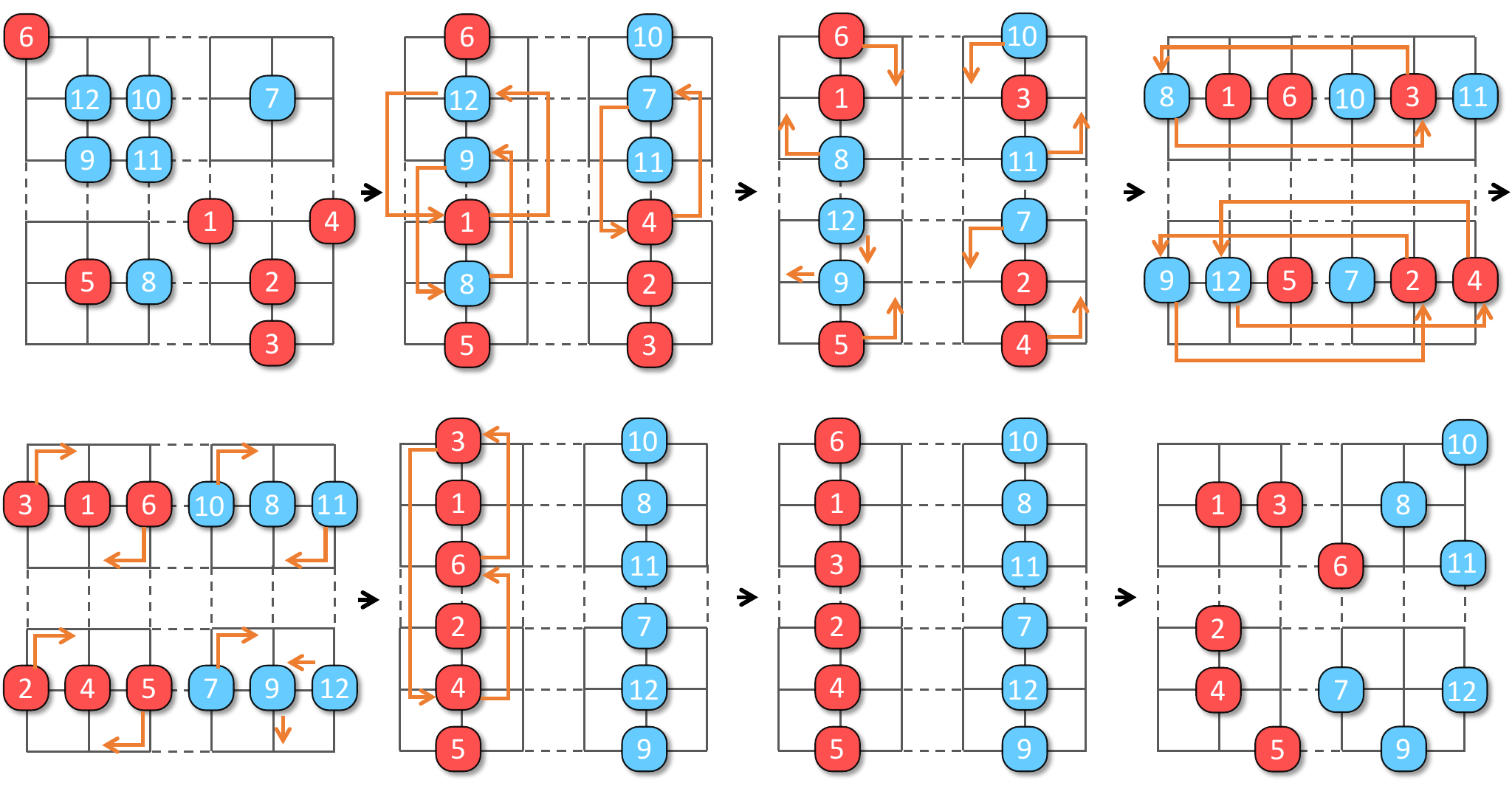}
             \footnotesize
             \put(10.5, 25) {(a)}
             \put(35.5, 25) {(b)}
             \put(60.5, 25) {(c)}
             \put(85.5, 25) {(d)}
             \put(10.5, -1) {(e)}
             \put(35.5, -1) {(f)}
             \put(60.5, -1) {(g)}
             \put(85.5, -1) {(h)}
        
        \end{overpic}
        \caption[An example of applying \gls{rth} to solve an \gls{mpp} instance.]{An example of applying \gls{rth} to solve an \gls{mpp} instance. (a) The start configuration; (b) The start balanced configuration obtained from (a);  (c) The intermediate configuration obtained from the Grid Rearrangement preparation phase; (d). The intermediate configuration obtained from (c); (e) The intermediate configuration obtained from the column fitting phase; (f) Apply additional column shuffles for labeled items; (g) The goal balanced configuration obtained from the goal configuration; (h) The goal configuration.} 
        \label{fig:example}
    \end{figure}
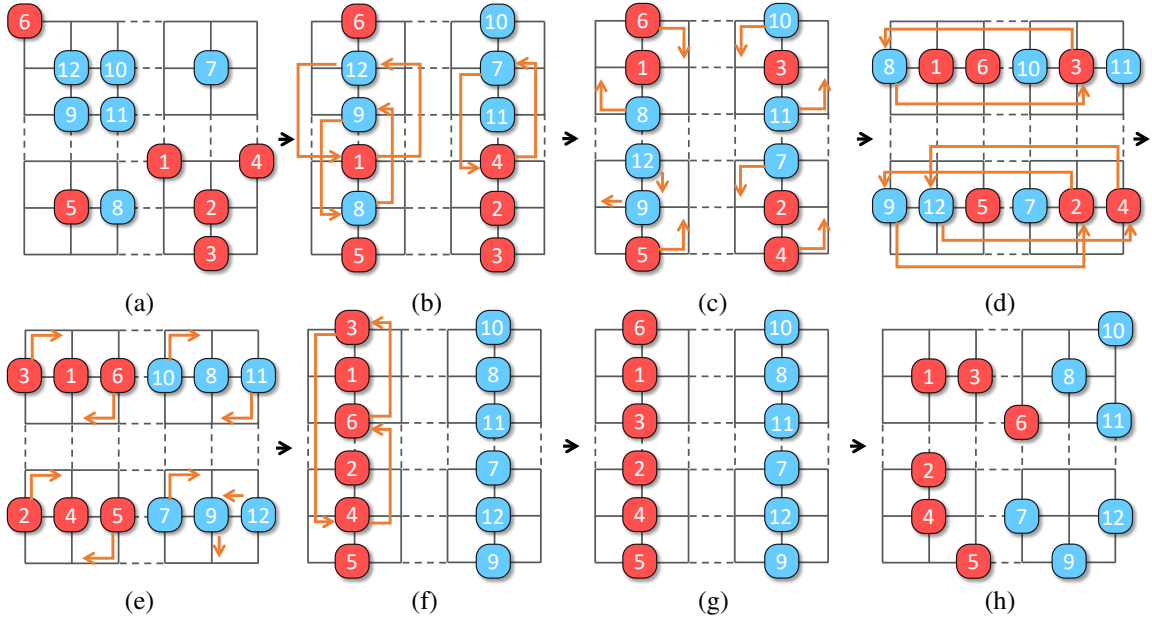

In unlabeled reconfiguration, robots are treated as being indistinguishable. Arbitrary start and goal configurations (under $\frac{1}{3}$ robot density) are converted to intermediate configurations where each $3 \times 3$ cell contains no more than $3$ robots. We call such configurations \emph{balanced}. With high probability, random \gls{mpp} instances are not far from being balanced. To establish this result (Proposition~\ref{p:phase:1}), we need the following. 

\begin{theorem}[Minimax Grid Matching~\cite{leighton1989tight}]\label{t:minimax}
Consider an $m \times m$ square containing $m^2$ points following the uniform distribution. Let $\ell$ be the minimum length such that there exists a perfect matching of the $m^2$ points to the grid points in the square for which the distance between every pair of matched points is at most $\ell$. Then $\ell = O(\log^{\frac{3}{4}}m)$ with high probability.
\end{theorem}

\ref{t:minimax} applies to rectangles with the longer side being $m$ as well (Theorem 3 in~\cite{leighton1989tight}). 

\begin{proposition}\label{p:phase:1}
On an $m_1\times m_2$ grid, with high probability, a random configuration of $n = \frac{m_1m_2}{3}$ robots is of distance $o(m_1)$ to a balanced configuration. 
\end{proposition}
\begin{proof}
We prove for the case of $m_1 = m_2 = 3m$ using the minimax grid matching theorem (\ref{t:minimax}); generalization to $m_1 \ge m_2$ can be then seen to hold using the generalized version of~\ref{t:minimax} that applies to rectangles (Theorem 3 of~\cite{leighton1989tight}, which applies to arbitrarily simply-connected region within a square region).

Now let $m_1 = m_2 = 3m$. We may view a random configuration of $m^2$ robots on a $3m\times 3m$ grid as randomly placing $m^2$ continuous points in an $m\times m$ square with scaling (by three in each dimension) and rounding. 
By~\ref{t:minimax}, a random configuration of $m^2$ continuous points in an $m \times m$ square can be moved to the $m^2$ grid points at the center of the $m^2$ disjoint unit squares within the $m \times m$ square, where each point is moved by a distance no more than $O(\log^{\frac{3}{4}}m)$, with high probability. 
Translating this back to a $3m \times 3m$ gird, $m^2$ randomly distributed robots on the grid can be moved so that each $3\times 3$ cell contains exactly one robot, and the maximum distance moved for any robot is no more than $O(\log^{\frac{3}{4}}m)$, with high probability. Applying this argument three times yields that a random configuration of $\frac{m_1^2}{3}$ robots on an $m_1\times m_1$ gird can be moved so that each $3\times 3$ cell contains exactly three robots, and no robot needs to move more than a $O(\log^{\frac{3}{4}}{m_1})$ steps, with high probability. Because the robots are indistinguishable, overlaying three sets of reconfiguration paths will not cause an increase in the distance traveled by any robot. \end{proof}

In the example, unlabeled reconfiguration corresponds to~\ref{fig:example}(a)$\to$\ref{fig:example}(b) and~\ref{fig:example}(f)$\to$\ref{fig:example}(e) (\gls{mpp} solutions are time-reversible). 
We simulated the process of unlabeled reconfiguration for $m_1 = m_2 = 300$, i.e., on a $300 \times 300$ grids. For $\frac{1}{3}$ robot density, the actual number of steps averaged over $100$ random instances is less than $5$.
We call configurations like ~\ref{fig:example}(b)-(e), which have all robots concentrated vertically or horizontally in the middle of the $3\times 3$ cells, \emph{centered balanced} or simply \emph{centered}. 
Completing the first phase requires solving two unlabeled problems~\cite{YuLav12CDC,Ma2016OptimalTA}, doable in polynomial time. 

In the second phase, \gls{rta} is applied with a highway heuristic to get us from~\ref{fig:example}(b) to~\ref{fig:example}(e), transforming between vertical centered configurations and horizontal centered configurations. 
To do so, \gls{rta} is applied (e.g., to~\ref{fig:example}(b) and (e)) to obtain two intermediate configurations (e.g.,~\ref{fig:example}(c) and (d)).
To go between these configurations, e.g.,~\ref{fig:example}(b)$\to$\ref{fig:example}(c), we apply a heuristic by moving robots that need to be moved out of a $3\times 3$ cell to the two sides of the middle columns of~\ref{fig:example}(b), depending on their target direction. If we do this consistently, after moving robots out of the middle columns, we can move all robots to their desired goal $3\times 3$ cell without stopping nor collision. 
Once all robots are in the correct $3\times 3$ cells, we can convert the balanced configuration to a centered configuration in at most $3$ steps, which is necessary for carrying out the next simulated row/column shuffle. 
Adding things up, we can simulate a shuffle operation using no more than $m + 5$ steps where $m = m_1$ or $m_2$. 
The efficiently simulated shuffle leads to low makespan \gls{mpp} algorithms. It is clear that all operations take polynomial time; a precise running time is given at the end of this subsection.

\begin{theorem}[Makespan Upper Bound for Random \gls{mpp}, $\le \frac{1}{3}$ Density]\label{t:rtm-ramdom}
For random \gls{mpp} instances on an $m_1 \times m_2$ grid, where $m_1 \ge m_2$ are multiples of three, for $n \le \frac{m_1m_2}{3}$ robots, an $m_1 + 2m_2 + o(m_1)$ makespan solution can be computed in polynomial time, with high probability. 
\end{theorem}

First, we introduce two important lemmas about obstacle-free 2D grids which have been proven in~\cite{yu2018constant}.

\begin{lemma}\label{lemma: anonymous_mkpn_2d}
On an $m_1\times m_2$ grid,  any unlabeled \gls{mpp} can be solved using $O(m_1+m_2)$ makespan. 
\end{lemma}
\begin{lemma}\label{lemma: makespan_2d}
On an $m_1\times m_2$ grid, if the underestimated makespan of an \gls{mpp} instance is $d_g$ (usually computed by ignoring the inter-robot collisions), then this instance can be solved using $O(d_g)$ makespan. 
\end{lemma}

\begin{proof}
By Proposition~\ref{p:phase:1}, unlabeled reconfiguration requires distance $o(m_1)$ with high probability. This implies that a plan can be obtained for unlabeled reconfiguration that requires $o(m_1)$ makespan (for detailed proof, refer to~\cite{yu2018constant}).
For the second phase, by~\ref{t:rta}, we need to perform $m_1$ parallel row shuffles with a row width of $m_2$, followed by $m_2$ parallel column shuffles with a column width of $m_1$, followed by another $m_1$ parallel row shuffles with a row width of $m_2$. Simulating these shuffles require $m_1 + 2m_2 + O(1)$ steps. Altogether, a makespan of $m_1 + 2m_2 + o(m_1)$ is required, with a very high probability. 
\end{proof}

Contrasting~\ref{t:rtm-ramdom} and Proposition~\ref{p:makespan-lower} yields

\begin{theorem}[Makespan Optimality for Random \gls{mpp}, $\le \frac{1}{3}$ Density]\label{t:rth-ratio}
For random \gls{mpp} instances on an $m_1 \times m_2$ grid, where $m_1 \ge m_2$ are multiples of three, for $n = cm_1m_2$ robots with $c \le \frac{1}{3}$, as $m_1 \to \infty$, a $(1 + \frac{m_2}{m_1 + m_2})$ makespan optimal solution can be computed in polynomial time, with high probability. 
\end{theorem}

Since $m_1\ge m_2$, $1 + \frac{m_2}{m_1 + m_2} \in (1, 1.5]$. In other words, in polynomial running time, \gls{rth} achieves $(1 + \delta)$ asymptotic makespan optimality for $\delta \in (0, 0.5]$, with high probability. 

From the analysis so far, if $m_1$ and/or $m_2$ are not multiples of $3$, it is clear that all results in this subsection continue to hold for robot density $\frac{1}{3} - \frac{(m_1\bmod 3)(m_2\bmod 3)}{m_1m_2}$, which is arbitrarily close to $\frac{1}{3}$ for large $m_1$ and $m_2$. It is also clear that the same can be said for grids with certain patterns of regularly distributed obstacles (\ref{fig:jd_center}(b)), i.e., 

\begin{corollary}[Random \gls{mpp}, $\frac{1}{9}$ Obstacle and $\frac{2}{9}$ Robot Density]
For random \gls{mpp} instances on an $m_1 \times m_2$ grid, where $m_1 \ge m_2$ are multiples of three and there is an obstacle at coordinates $(3k_1 + 2, 3k_2 + 2)$ for all applicable $k_1$ and $k_2$, for $n = cm_1m_2$ robots with $c \le \frac{2}{9}$, a solution can be computed in polynomial time that has makespan $m_1 + 2m_2 + o(m_1)$ with high probability.
As $m_1 \to \infty$, the solution approaches $1 + \frac{m_2}{m_1 + m_2}$ optimal, with high probability. 
\end{corollary}

\begin{proof}
Because the total density of robots and obstacles is no more than $1/3$, if~\ref{t:minimax} extends to support regularly distributed obstacles, then~\ref{t:rth-ratio} applies because the highway heuristics do not pass through the obstacles. This is true because each obstacle can only add a constant length of path detour to a robot's path. In other words, the length $\ell$ in~\ref{t:rth-ratio} will only increase by a constant factor and will remain as $\ell = O(\log^{\frac{3}{4}}m)$. Similar arguments hold for Proposition~\ref{p:phase:1}.
\end{proof}

We now give the running time of \gls{rtmdd} and \gls{rth}. 

\begin{proposition}[Running Time, \gls{rth}]\label{p:time}
For $n \le \frac{m_1m_2}{3}$ robots on an $m_1 \times m_2$ grid, \gls{rth} runs in $O(nm_1^2m_2)$ time.
\end{proposition}

\begin{proof}
The running time of \gls{rth} is dominated by the matching computation and solving unlabeled \gls{mpp}. 
The matching part takes $O(m_1^2m_2)$ in deterministic time or $O(m_1m_2\log m_1)$ in expected time~\cite{goel2013perfect}. 
Unlabeled \gls{mpp} may be tackled using the max-flow algorithm~\cite{ford1956maximal} in $O(nm_1m_2T)=O(nm_1^2m_2)$ time, where $T=O(m_1+m_2)$ is the expansion time horizon of a time-expanded graph that allows a routing plan to complete. 
\end{proof}

\subsection{One-Half Robot Density: Shuffling with Linear Merging}\label{subsec:half-density}
Using a more sophisticated shuffle routine, $\frac{1}{2}$ robot density can be supported while retaining most of the guarantees for the $\frac{1}{3}$ density setting; obstacles are no longer supported. 

To best handle $\frac{1}{2}$ robot density, we employ a new shuffle routine called \emph{linear merge}, based on merge sort, and denote the resulting algorithm as Grid Rearrangement with Linear Merge or \gls{rtlm}. 
The basic idea is straightforward: for $m$ robots on a $2 \times m$ grid,  we iteratively sort the robots first on $2\times 2$ grids, then $2\times 4$ grids, and so on, much like how merge sort works. An illustration of the process on a $2 \times 8$ grid is shown in~\ref{fig:merge}.

\begin{algorithm}
\begin{small}
\DontPrintSemicolon
\SetKwProg{Fn}{Function}{:}{}
\SetKwFunction{LineMerge}{LineMerge}
\SetKwFunction{Merge}{Merge}

\caption{Line Merge Algorithm \label{alg:lineMerge}}

\KwIn{An array $arr$ representing the vertices labeled from $1$ to $n$ with current and intermediate locations of $n$ robots.}

\Fn{\LineMerge{$arr$}}{
    \If{$\text{length of } arr > 1$}{
        $mid \leftarrow \text{length of } arr \div 2$\;
        $left \leftarrow arr[0 \ldots mid-1]$\;
        $right \leftarrow arr[mid \ldots \text{length of } arr - 1]$\;
        
        \LineMerge{$left$}\;
        \LineMerge{$right$}\;
        
        \Merge{$arr$}\;
    }
}

\Fn{\Merge{$arr$}}{
      Sort robots located at the vertices in $arr$ to obtain intermediate states\;
    \For{$i\in arr$}{
        $a_i \leftarrow$ robot at vertex $i$\;
        \If{$a_i$.intermediate $> i$}{
            Route $a_i$ moving rightward using the bottom line while avoiding blocking those robots that are moving leftward\;
        }
        \Else{
            Route $a_i$ moving leftward using the upper line without stopping\;
        }
    }
    Synchronize the paths\;
}
\end{small}
\end{algorithm}

\begin{figure}[h!]
\centering
    \begin{overpic}[width=\linewidth]{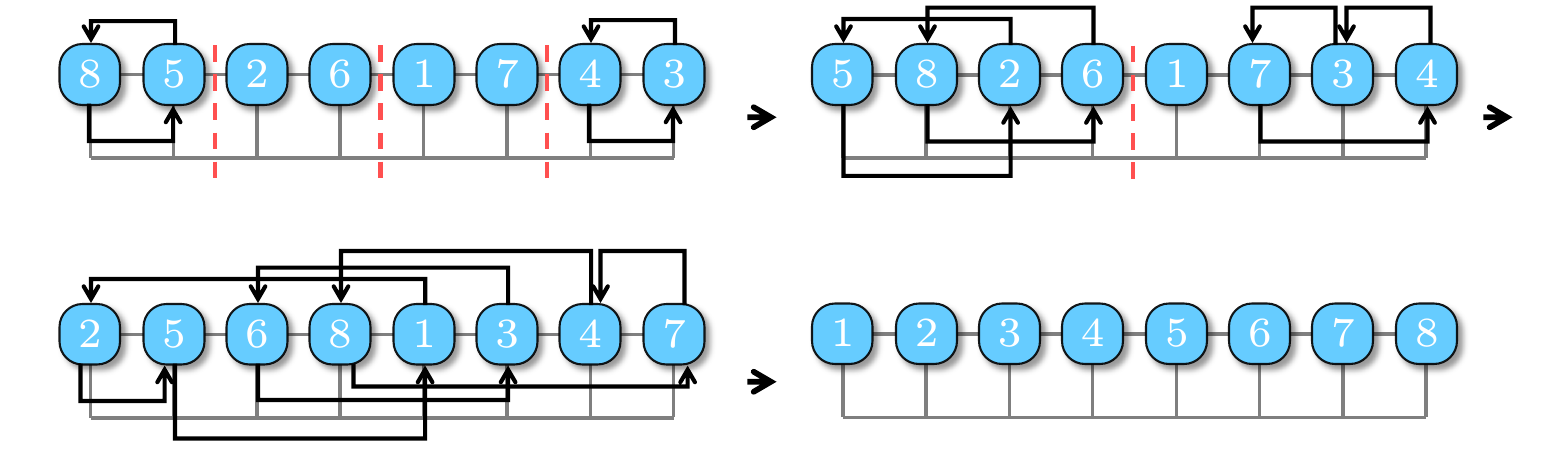}
             \footnotesize
             \put(22.8, 14.2) {(a)}
             \put(71.5, 14.2) {(b)}
             \put(22.8, -2.5) {(c)}
             \put(71.5, -2.5) {(d)}
    \end{overpic}
    \caption{A demonstration of the linear merge shuffle primitive on a $2\times 8$ grid. Robots going to the left always use the upper channel while robots going to the right always use the lower channel.} 
    \label{fig:merge}
\end{figure}
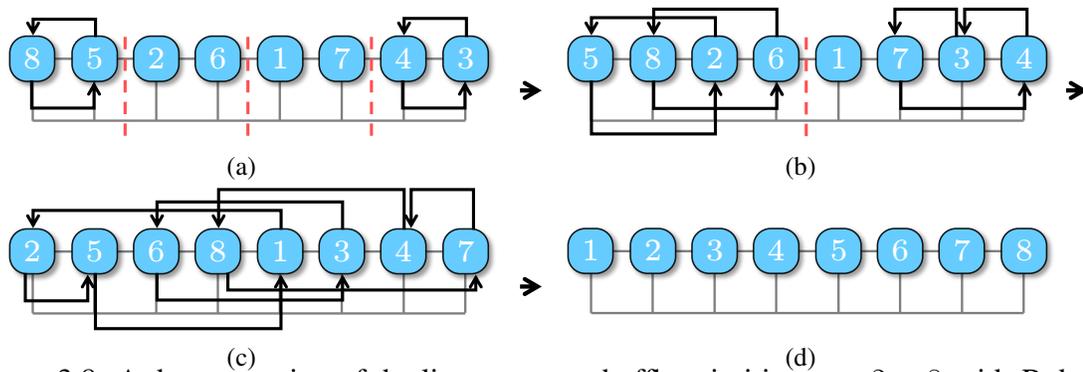


\begin{lemma}[Properties of Linear Merge]\label{l:lm}
On a $2\times m$ grid, $m$ robots, starting on the first row, can be arbitrarily ordered using $m + o(m)$ steps, using the~\ref{alg:lineMerge}, inspired from merge sort. The motion plan can be computed in polynomial time. 
\end{lemma}
\begin{proof}
We first show \emph{feasibility}. The procedure takes $\lceil \log m \rceil$ phases; in a phase, let us denote a section of the $2\times m$ grid where robots are treated together as a \emph{block}. For example, the left $2 \times 4$ grid in~\ref{fig:merge}(b) is a block. It is clear that the first phase, involving up to two robots per block, is feasible (i.e., no collision). Assuming that phase $k$ is feasible, we look at phase $k + 1$. We only need to show that the procedure is feasible on one block of length up to $2^{k+1}$. For such a block, the left half-block of length up to $2^k$ is already fully sorted as desired, e.g., in increasing order from left to right. For the $k+1$ phase, all robots in the left half-block may only stay in place or move to the right. These robots that stay must be all at the leftmost positions of the half-block and will not block the motions of any other robot. For the robots that need to move to the right, their relative orders do not need to change and, therefore, will not cause collisions among themselves. Because these robots that move in the left half-block will move down on the grid by one edge, they will not interfere with any robot from the right half-block. 
Because the same arguments hold for the right half-block (except the direction change), the overall process of merging a block occurs without collision. 

Next, we examine the \emph{makespan}. For any single robot $r$, at phase $k$, suppose it belongs to block $b$ and block $b$ is to be merged with block $b'$. It is clear that the robot cannot move more than $len(b') + 2$ steps, where $len(b')$ is the number of columns of $b'$ and the $2$ extra steps may be incurred because the robot needs to move down and then up the grid by one edge. This is because any move that $r$ needs to do is to allow robots from $b'$ to move toward $b$.
Because there are no collisions in any phase, adding up all the phases, no robot moves more than $m + 2(\log m + 1) = m + o(m)$ steps. 

Finally, the merge sort-like linear merge shuffle primitive runs in $O(m\log m)$ time since it is a standard divide-and-conquer routine with $\log m$ phases. 
\end{proof}
We distinguish between $\frac{1}{3}$ and $\frac{1}{2}$ density settings because the overhead in \gls{rtlm} is larger. Nevertheless, with the linear merge, the asymptotic properties of \gls{rth} for $\frac{1}{3}$ robot density mostly carry over to \gls{rtlm}. 

\begin{theorem}[Random \gls{mpp}, $\frac{1}{2}$ Robot Density]
For random \gls{mpp} instances on an $m_1 \times m_2$ grid, where $m_1 \ge m_2$ are multiples of two, for $\frac{m_1m_2}{3} \le n  \le \frac{m_1m_2}{2}$ robots, a solution can be computed in polynomial time that has makespan $m_1 + 2m_2 + o(m_1)$ with high probability.
As $m_1 \to \infty$, the solution approaches an optimality of $1 + \frac{m_2}{m_1 + m_2} \in (1, 1.5]$, with high probability. 
\end{theorem}

\begin{proof}[Proof Sketch]
The proof follows by combining Proposition~\ref{p:asymptotic_lowerbound} and ~\ref{l:lm}.
\end{proof}

\subsection{Supporting Arbitrary \gls{mpp} Instances on Grids}\label{subsec:arbi}
We now examine applying \gls{rth} to arbitrary \gls{mpp} instances up to $\frac{1}{2}$ robot density.
If an \gls{mpp} instance is arbitrary, all that changes to \gls{rth} is the makespan it takes to complete the unlabeled reconfiguration phase. On an $m_1\times m_2$ grid, by computing a matching, it is straightforward to show that it takes no more than $m_1 + m_2$ steps to complete the unlabeled reconfiguration phase, starting from an arbitrary start configuration.  Since two executions of unlabeled reconfiguration are needed, this adds $2(m_1 + m_2)$ additional makespan. Therefore, the following results hold.

\begin{theorem}[Arbitrary \gls{mpp}, $\le \frac{1}{2}$ Density]\label{t:arbi}
For arbitrary \gls{mpp} instances on an $m_1 \times m_2$ grid, $m_1 \ge m_2$, for $n \le \frac{m_1m_2}{2}$ robots, a $3m_1 + 4m_2 + o(m_1)$ makespan solution can be computed in polynomial time. 
\end{theorem}

    

\section{Optimality-Boosting Heuristics}\label{sec:opt-boost}
\subsection{Reducing Makespan via Optimizing Matching}\label{subsec:heuristics}
Based on \gls{rta}, \gls{rth} has three simulated shuffle phases. The makespan is dominated by the robot, which needs the longest time to move, as a sum of moves in all three phases. As a result, the optimality of Grid Rearrangement methods is determined by the first preparation/matching phase. 
Finding arbitrary perfect matchings is fast but the process can be improved to reduce the overall makespan. 

For improving matching, we propose two heuristics; the first is based on \emph{integer programming} (IP). 
We create binary variables $\{x_{ri}\}$ where $r$ represents the row number and $i$ the robot. 
robot $i$ is assigned to row $r$ if $x_{ri}=1$. 
Define single robot cost as $C_{ri}(\lambda)=\lambda |r-s_i.x|+(1-\lambda)|r-g_i.x|$. 
We optimize the makespan lower bound of the first phase by letting $\lambda=0$ or the third phase by letting $\lambda=1$. 
The objective function and constraints are given by
\begin{equation}
\label{eq:objective}
    \max_{r,i} \{{C_{ri}(\lambda=0)x_{ri}}\}+\max_{r,i}\{C_{ri}(\lambda=1)x_{ri}\}
\end{equation}
\vspace{-2mm}
\begin{equation}
\label{eq:constraint1}
    \sum_{r}x_{ri}=1, \mathrm{for\ each\ robot\ } i
\end{equation}
\vspace{-2mm}
\begin{equation}
\label{eq:constraint2}
    \sum_{g_i.y=t}x_{ri}\leq 1, {\small\mathrm{for\ each\ row\ }r\mathrm{\ and\ each\ color\ } t}  \end{equation}
\vspace{-1mm}
\begin{equation}
\label{eq:constraint3}
\sum_{s_i.x=c}x_{ri}=1, \mathrm{for\ each\ column\ }c \mathrm{and\ each\ row\ }r
\end{equation}

\ref{eq:objective} is the summation of makespan lower bound of the first phase and the third phase. Note that the second phase can not be improved by optimizing the matching.
\ref{eq:constraint1} requires that robot $i$ be only present in one row.
\ref{eq:constraint2} specifies that each row should contain robots that have different goal columns.
\ref{eq:constraint3} specifies that each vertex $(r,c)$ can only be assigned to one robot. 
The IP model represents a general assignment problem which is NP-hard in general.
It has limited scalability but provides a way to evaluate how optimal the matching could be in the limit.

A second matching heuristic we developed is based on \gls{lba}~\cite{burkard2012assignment}, which takes polynomial time.
\gls{lba} differs from the IP heuristic in that the bipartite graph is weighted.
For the matching assigned to row $r$, the edge weight of the bipartite graph is computed greedily.
If column $c$ contains robots of color $t$, we add an edge $(c,t)$ and its edge cost is
\vspace{-1.5mm}
\begin{equation}
\vspace{-1.5mm}
    C_{ct}=\min_{g_i.y=t}C_{ri}(\lambda=0)
\end{equation}
We choose $\lambda=0$ to optimize the first phase. Optimizing the third phase ($\lambda=1$) would give similar results.
After constructing the weighted bipartite graph, an $O(\frac{m_1^{2.5}}{\log m_1})$ \gls{lba} algorithm~\cite{burkard2012assignment} is applied to get a minimum bottleneck cost matching for row $r$. Then we remove the assigned robots and compute the next minimum bottleneck cost matching for the next row. 
After getting all the matchings $\mathcal{M}_r$, we can further use \gls{lba} to assign $\mathcal{M}_r$ to a different row $r'$ to get a smaller makespan lower bound. The cost for assigning matching $\mathcal{M}_r$ to row $r'$ is defined as 
\vspace{-1.5mm}
\begin{equation}
\vspace{-1.5mm}
    C_{\mathcal{M}_rr'}=\max_{i\in\mathcal{M}_r}C_{r'i}(\lambda=0)
\end{equation}
The total time complexity of using \gls{lba} heuristic for matching is $O(\frac{m_1^{3.5}}{\log m_1})$.

We denote \gls{rth} with IP and \gls{lba} heuristics as \gls{rth}ip and \gls{rth}lba, respectively.
We mention that \gls{rtmdd}, which uses the line swap motion primitive, can also benefit from these heuristics to re-assign the goals within each group. This can lower the bottleneck path length and improve optimality. 

\subsection{Path Refinement}\label{sec: path_refine}
Final paths from \gls{rta}-based algorithms are concatenations of paths from multiple planning phases. 
This means robots are forced to ``synchronize'' at phase boundaries, which causes unnecessary idling for robots finishing a phase early. 
Noticing this, we de-synchronize the planning phases, which yields significant gains in makespan optimality. 

Our path refinement method does something similar to Minimal Communication Policy (\gls{mcp})~\cite{ma2017multi}, a multi-robot execution policy proposed to handle unexpected delays without stopping unaffected robots.
During  execution, \gls{mcp} lets robots execute their next non-idling move as early as possible while preserving the relative execution orders between robots,
e.g., 
when two robots need to enter the same vertex at different times, that ordering is preserved. 
%
%
%
%

We adopt the principle used in \gls{mcp}  to refine the paths generated by \gls{rta}-based algorithms as shown in~\ref{alg:refinement}, with~\ref{alg:mcp_move} as a sub-routine.  
The algorithms work as follows. All idle states are removed from the initial robot but the order of visits for each vertex is kept (Line 2-3).
Then we enter a loop executing the plans following the \gls{mcp} principle (Line 8-10).
In~\ref{alg:mcp_move}, if  $i$ is the next robot that enters vertex $v_i$ according to the original order, we check if there is a robot currently at $v_i$.
If there is not, we let $i$ enter $v_i$.
If another robot $j$ is currently occupying $v_j$, we examine if $j$ is moving to its next vertex $v_j$ in the next step by recursively calling the function 
$\texttt{Move}$. 
We check if there is any cycle in the robot movements in the next original plan.
If $i$ is in a cycle, we move all the robots in this cycle to their next vertex recursively (Line 20-28).
If no cycle is found and $j$ is to enter its next vertex $v_j$, we let $i$ also move to its next vertex $v_i$. 
Otherwise, $i$ should  wait at $u_i$. 
The algorithm is \emph{deadlock-free} by construction ~\cite{ma2017multi}.

\begin{algorithm}
\DontPrintSemicolon
\SetKwProg{Fn}{Function}{:}{}
\SetKw{Continue}{continue}
 \caption{Path Refinement \label{alg:refinement}}
  \KwIn{Paths $\mathcal{P}$ generated by \gls{rta} }
  \Fn{Refine($\mathcal{P}$)}{
  \textbf{foreach} $v\in V$, $VOrder[v]\leftarrow Queue()$\;
  \texttt{Preprocess($InitialPlans,VOrder$)}\;
      \While{true}{
          \For{$i=1...n$}{
            $Moved\leftarrow Dict()$\;
            $CycleCheck\leftarrow Set()$\;
            $\texttt{Move}(i)$\;
              \If{$\texttt{AllReachedGoals()}$=true}{break\;}
        }
    }
}
\end{algorithm}

\begin{algorithm}
\DontPrintSemicolon
\SetKwProg{Fn}{Function}{:}{}
\SetKw{Continue}{continue}
\caption{Move \label{alg:mcp_move}}
\Fn{$\texttt{Move}$(i)}{
\If{$i$ in $Moved$}{
\Return $Moved[i]$\;
}
$u_i\leftarrow$ current position of $i$\;
$v_i\leftarrow$ next position of $i$\;
\If{$i=VOrder[v_i].front()$}{
    $j\leftarrow$ the robot currently at $v_i$\;
    \If{$i$ is in $CycleCheck$}{
        $\texttt{MoveAllRobotsInCycle}(i)$\;
        \Return true\;
    }
    $CycleCheck.add(i)$\;
    \If{$\texttt{Move}(j)=true$ or $j=$None}{
       let $i$ enter $v_i$\;
       $VOrder[v_i].popfront()$\;
    $Moved[i]\leftarrow$true\;
    \Return true\;
    }
}
let $i$ wait at $u_i$\;
$Moved[i]$=false\;
\Return false\;
}
\Fn{$\texttt{MoveAllRobotsInCycle}(i)$}{
$Visited\leftarrow Set()$\;
$j\leftarrow i$\;
\While{$j$ is not in visited}{
    let $j$ enter its next vertex $v_j$\;
    $Visited.add(j)$\;
    $VOrder[v_j].popfront()$\;
    $Moved[j]\leftarrow$ true\;
    $j\leftarrow$ the robot currently at vertex $v_j$\; 
}
}
\end{algorithm}
%
Let $T$ be the makespan of the initial paths, the makespan of the solution obtained by running $\texttt{Refine}$ is clearly no more than $T$.
In each loop of $\texttt{Refine}$, we essentially run DFS on a graph that has $n$ nodes and traverse all the nodes, for which the time complexity is $O(n)$,
Therefore the time complexity of the path refinement is bounded by $O(nT)$.

Other path refinement methods, such as~\cite{li2021anytime,okumura2021iterative}, can also be applied in principle, which iteratively chooses a small group of robots and re-plan their paths holding other robots as dynamic obstacles. 
However, in dense settings that we tackle, re-planning for a small group of robots has little chance of finding better paths this way. 

\section{Generalization to 3D}\label{sec: three_d}
We now explore the 3D setting. To keep the discussion focused, we mainly show how to generalize \gls{rth} to 3D, but note that \gls{rtmdd} and \gls{rtlm} can also be similarly generalized. 
High-dimensional \gls{rtp}~\cite{szegedy2023rubik} is defined as follows. 

\begin{problem}[{\normalfont \bf Grid Rearrangement in $k$D (\gls{rtp}k)}]\label{p:rtkd}
Let $M$ be an $m_1 \times \ldots \times m_k$ table, $m_1\ge \ldots \ge m_k$, containing $\prod_{i=1}^km_i$ items, one in each table cell. 
In a \emph{shuffle} operation, the items in a single column in the $i$-th dimension of $M$, $1 \le i \le k$, may be permuted arbitrarily. 
Given two arbitrary configurations $S$ and $G$ of the items on $M$, find a sequence of shuffles that take $M$ from $S$ to $G$.
\end{problem}

We may solve \gls{rtpthree} using \gls{rtatwo} as a subroutine by treating \gls{rtp} as a \gls{rtptwo}, which is straightforward if we view a 2D slice of \gls{rtatwo} as a ``wide'' column. For example, for the $m_1 \times m_2 \times m_3$ grid, we may treat the second and third dimensions as a single dimension.
Then, each wide column, which is itself an $m_2 \times m_3$ 2D problem, can be reconfigured by applying \gls{rtatwo}.
With some proper counting, we obtain the following 3D version of 
\ref{t:rta} as follows.
\begin{theorem}[{\normalfont \bf Grid Rearrangement Theorem, 3D}]\label{p:rta3d}
An arbitrary Grid Rearrangement problem on an $m_1\times m_2\times m_3$ table can be solved using $m_1m_2+m_3(2m_2+m_1)+m_1m_2$ shuffles. 
\end{theorem}

Although~\cite{szegedy2023rubik} offers a broad framework for addressing \gls{rtp}k, it is not directly applicable to multi-robot routing in high-dimensional spaces. To overcome this limitation, we have developed the high-dimensional \gls{mpp} algorithm by incorporating and elaborating on its underlying principles similar to the 2D scenarios.
Denote the corresponding algorithm for~\ref{p:rta3d} as 
\gls{rtathree}, we illustrate how it works on a $3 \times 3\times 3$ table (\ref{fig:rubik3d}). \gls{rta} operates in three phases. First, a bipartite graph $B(T, R)$ is constructed based on the initial table where the partite set $T$ are the colors of items representing the desired $(x,y)$ positions, and the set $R$ are the set of $(x,y)$ positions of items  (\ref{fig:rubik}(b)). Edges are added as in the 2D case. 
From $B(T,R)$, $m_3$ \emph{perfect matchings} are computed. Each matching contains $m_1m_2$ edges and connects all of $T$ to all of $R$. processing these matchings yields an intermediate table (\ref{fig:rubik3d}(c)), serving a similar function as in the 2D case.
After the first phase of $m_1m_2$ $z$-shuffles, the intermediate table (\ref{fig:rubik3d}(c)) that can be reconfigured by applying \gls{rtatwo} with $m_1$ $x$-shuffles and $2m_2$ $y$-shuffles.
This sorts each item in the correct $x$-$y$ positions (\ref{fig:rubik}(d)).
Another $m_1m_2$  $z$-shuffles can then sort each item to its desired $z$ position (\ref{fig:rubik}(e)). 

 \begin{figure}[htb]
        \centering
        \includegraphics[width=\linewidth]{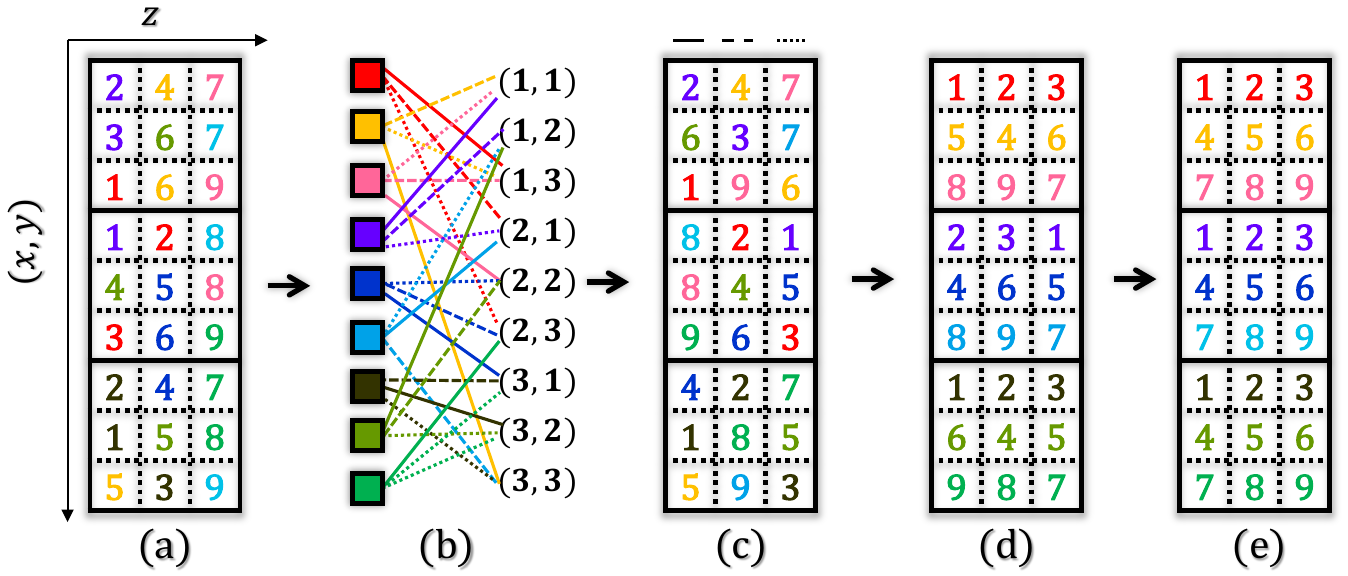}
                \caption[Illustration of applying \emph{\gls{rtathree}}]{Illustration of applying \emph{\gls{rtathree}}. (a) The initial $3\times 3\times 3$ table with a random arrangement of 27 items that are colored and labeled. Color represents the $(x,y)$ position of an item. (b) The constructed bipartite graph. The right partite set contains all the possible $(x,y)$ positions. It contains $3$ perfect matchings, determining the $3$ columns in (c). (c) Applying $z$-shuffles to (a), according to the matching results, leads to an intermediate table where each $x$-$y$ plane has one color appearing exactly once. (d) Applying wide shuffles to (c) correctly places the items according to their $(x, y)$ values (or colors). (e)  Additional $z$-shuffles fully sort the labeled items.} 
        \label{fig:rubik3d}
    \end{figure}
\subsection{Extending \gls{rth} to 3D}
The \ref{alg:rth3d}, which calls~\ref{alg:matching3d} and~\ref{alg:xyfit}, outlines the high-level process of extending \gls{rth} to 3D. 
In each $x$-$y$ plane, $\mathcal G$ is partitioned into $3\times 3$ cells (e.g.,\ref{fig:example}). 
Without loss of generality, we assume that $m_1,m_2,m_3$ are multiples of 3 and there are no obstacles.
First, to make \gls{rtathree} applicable, we convert the arbitrary start/goal
configurations to intermediate \emph{balanced} configurations $S_1$ and $G_1$, treating robots as unlabeled, as we have done in \gls{rth}, wherein each 2D plane, each $3 \times 3$ cell contains no more than $3$ robots (\ref{fig:random_to_balanced}). 

 \begin{figure}[htb]
        \centering
        \includegraphics[width=0.9\linewidth]{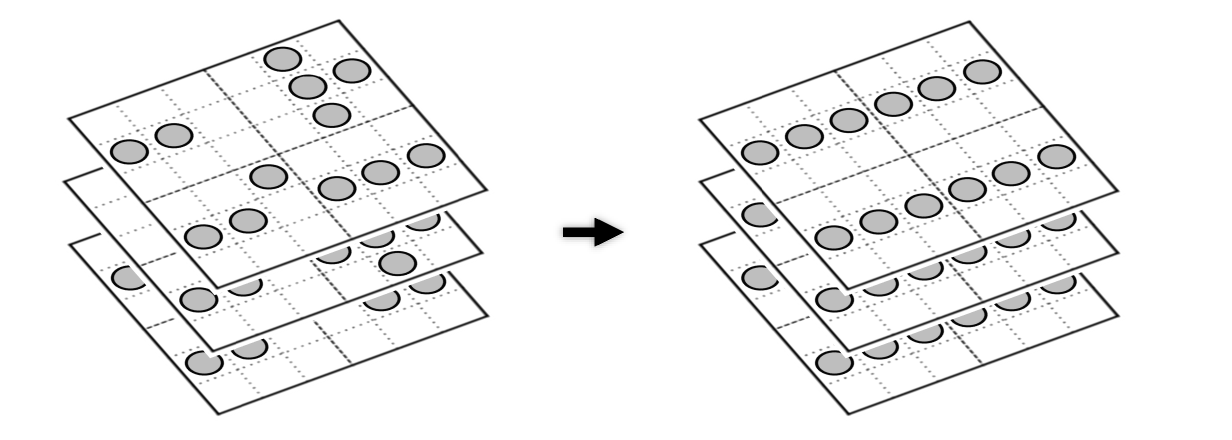}
\caption{Applying unlabeled \gls{mpp} to convert a random configuration to a balanced centering one on  $6\times 6 \times 3$ grids.} 
        \label{fig:random_to_balanced}
    \end{figure}

\gls{rtathree} can then be applied to coordinate the robots moving toward their intermediate goal positions $G_1$.
Function $\texttt{MatchingXY}$ finds a feasible intermediate configuration $S_2$ and routes the robots to $S_2$ by simulating shuffle operations along the $z$ axis.
Function $\texttt{XY-Fitting}$ apply shuffle operations along the $x$ and $y$ axes to route each robot $i$ to its desired $x$-$y$ position  $(g_{1i}.x,g_{1i}.y)$.
In the end, the function $\texttt{Z-Fitting}$ is called, routing each robot $i$ to the desired $g_{1i}$ by performing shuffle operations along the $z$ axis and concatenating the paths computed by unlabeled \gls{mpp} planner $\texttt{UnlabeledMRPP}$. 

\begin{algorithm}
\DontPrintSemicolon
\SetKwProg{Fn}{Function}{:}{}
\SetKwFunction{Fampp}{UnlabeledMRPP}
\SetKwFunction{Frthddd}{GRH3D}
\SetKwFunction{Fmatchingthreed}{MatchingXY}
\SetKwFunction{Frthdd}{XY-Fitting}
\SetKwFunction{Fzfitting}{Z-Fitting}

  \caption{\gls{rthddd}\label{alg:rth3d}}
  \KwIn{Start and goal vertices $S=\{s_i\}$ and $G=\{g_i\}$}
  \Fn{\Frthddd{$S,G$}}{
$S_1,G_1\leftarrow$\Fampp{$S,G$}\;
\Fmatchingthreed{}\;
\Frthdd{}\;
\Fzfitting{}\;
}
\vspace{1mm}
\end{algorithm}  
We now explain each part of \gls{rthddd}.
$\texttt{MatchingXY}$ uses an extended version of \gls{rta} to find perfect matching that allows feasible shuffle operations. 
Here, the ``item color" of an item $i$ (robot) is the tuple $(g_{1i}.x,g_{1i}.y)$, which is the desired $x$-$y$ position it needs to go.
After finding the $m_3$ perfect matchings, the intermediate configuration $S_2$ is determined.
Then, shuffle operations along the $z$ direction can be applied to move the robots to $S_2$.
\begin{algorithm}
\DontPrintSemicolon
\SetKwProg{Fn}{Function}{:}{}
\SetKwFunction{Fmatchingthreed}{MatchingXY}
  \caption{MatchingXY \label{alg:matching3d}}
  \KwIn{Balanced start and goal vertices $S_1=\{s_{1i}\}$ and $G_1=\{g_{1i}\}$}
  \Fn{\Fmatchingthreed{$S_1,G_1$}}{
$A\leftarrow[1,...,n]$\;
$\mathcal{T}\leftarrow$ the set of $(x,y)$ positions in $S_1$\;
\For{$(r,t)\in \mathcal{T}\times \mathcal{T}$}{
    \If{$\exists i\in A$ where $(s_{1i}.x,s_{1i}.y)=r\wedge (g_{1i}.x,g_{1i}.y)=t$}{
        add edge $(r,t)$ to $B(T,R)$\;
        remove $i$ from $A$ \;
    }
}
    compute matchings $\mathcal{M}_1,...,\mathcal{M}_{m_3}$ of $B(T, R)$\;
     $A\leftarrow [1,...,n]$\;
 \ForEach{$\mathcal{M}_c$ and $(r,t)\in \mathcal{M}_c$}{
 \If{$\exists i\in A$ where $(s_{1i}.x,s_{1i}.y)=r\wedge (g_{1i}.x,g_{1i}.y)=t$}{
 $s_{2i}\leftarrow (s_{1i}.x, s_{1i}.y,c)$ and remove $i$ from $A$\;
  mark robot $i$ to go to $s_{2i}$\;
 }
 }
   perform simulated $z$-shuffles in parallel \;
  
}
\vspace{1mm}

\end{algorithm}
The robots in each $x$-$y$ plane will be reconfigured by applying $x$-shuffles and $y$-shuffles.
We need to apply \gls{rtatwo} for these robots in each plane, as demonstrated in~\ref{alg:xyfit}.
In \gls{rth}, for each 2D plane,  the ``item color" for robot $i$ is its desired $x$ position $g_{1i}.x$.
For each plane, we compute the $m_2$ perfect matchings to determine the intermediate position $g_{2}$.
Then each robot $i$ moves to its $g_{2i}$ by applying $y$-shuffle operations.
In Line 18, each robot is routed to its desired $x$ position by performing $x$-shuffle operations.
In Line 19, each robot is routed to its desired $y$ position by performing $y$-shuffle operations.
\begin{algorithm}
\begin{small}
\DontPrintSemicolon
\SetKwProg{Fn}{Function}{:}{}
\SetKwFunction{Fxyfitting}{XY-Fitting}
\SetKwFunction{Frthdd}{GRH2D}
  \caption{XY-Fitting \label{alg:xyfit}}
  \KwIn{Current positions $S_2$ and balanced goal positions $G_1$}
  \Fn{\Fxyfitting{}}{
  \For{$z\leftarrow[1,...,m_3]$}{
$A\leftarrow \{i|s_{2i}.y=z\}$\;
\Frthdd{$A,z$}\;
}
}
  \Fn{\Frthdd{$A,z$}}{

$\mathcal{T} \leftarrow$ the set of $x$ positions of $S_2$\;
\For{$(r,t)\in \mathcal{T}\times\mathcal{T}$}{
       \If{$\exists i\in A$ where $s_{2i}.x=r\wedge g_{1i}.x=t$}{
       \If{robot $i$ is not assigned}{
       add edge $(r,t)$ to $B(T,R)$\;
        mark $i$ assigned \;
       }      
    }
      compute matchings $\mathcal{M}_1,...,\mathcal{M}_{m_2}$ of $B(T, R)$\;
}
$A'\leftarrow A$\;
   \ForEach{$\mathcal{M}_c$ and $(r,t)\in \mathcal{M}_c$}{
 \If{$\exists i\in A'$ where $s_{2i}.x=c\wedge g_{1i}.x=t$}{
 $g_{2i}\leftarrow (s_{2i}.x, c,z)$ and remove $i$ from $A'$\;
  mark robot $i$ to go to $g_{2i}$\;
 }
 }
  route each robot $i\in A$ to $g_{2i}$\;
  route each robot $i\in A$ to  $(g_{2i}.x,g_{1i}.y,z)$\;
  route each robot $i\in A$ to  $(g_{1i}.x,g_{1i}.y,z)$\;
}

\vspace{1mm}
\end{small}
\end{algorithm} 
After all the robots reach the desired $x$-$y$ positions, another round of $z$-shuffle operations in $\texttt{Z-Fitting}$ can route the robots to the balanced goal configuration computed by an unlabeled \gls{mpp} planner.
In the end, we concatenate all the paths as the result.

\gls{rtlm} and \gls{rtmdd} can be extended to 3D in similar ways by replacing the solvers in 2D planes.
For \gls{rtmdd}, there is no need to use unlabeled \gls{mpp} for balanced reconfiguration, which yields a makespan upper bound of $4m_1 + 8m_2 + 8m_3$ (as a direct combination of~\ref{t:rtm-makespan} and~\ref{p:rta3d}). For arbitrary instances under half density in 3D, the makespan guarantee in~\ref{t:arbi} updates to  $3m_1 + 4m_2 + 4m_3 + o(m_1)$.

\subsection{Properties of \gls{rthddd}}
First, we introduce two important lemmas about obstacle-free 3D grids which have been proven in~\cite{yu2018constant}.
\begin{lemma}\label{lemma: anonymous_mkpn}
On an $m_1\times m_2 \times m_3$ grid,  any unlabeled \gls{mpp} can be solved using $O(m_1+m_2+m_3)$ makespan. 
\end{lemma}
\begin{lemma}\label{lemma: makespan}
On an $m_1\times m_2\times m_3$ grid, if the underestimated makespan of an \gls{mpp} instance is $d_g$ (usually computed by ignoring the inter-robot collisions), then this instance can be solved using $O(d_g)$ makespan. 
\end{lemma}
%
%
We proceed to analyze the time complexity of \gls{rthddd} on $m_1\times m_2\times m_3$ grids.
The dominant parts of \gls{rthddd} are computing the perfect matchings and solving unlabeled \gls{mpp}.
First we analyze the running time for computing the matchings.
Finding $m_3$ ``wide column" matchings runs in $O(m_3 m_1^2m_2^2)$ deterministic time or $O(m_3m_1m_2\log(m_1m_2))$ expected time.
We apply \gls{rtatwo} for simulating ``wide column" shuffle, which requires $O(m_3 m_1^2m_2)$  deterministic time or $O(m_1m_2m_3\log m_1)$ expected time. 
Therefore, the total time complexity for Rubik Table part is $O(m_3 m_1^2m_2^2+m_1^2m_2m_3)$.
For unlabeled \gls{mpp}, we can use the max-flow based algorithm~\cite{yu2013multi} to compute the makespan-optimal solutions. 
The max-flow portion can be solved in $O(n|E|T)=O(n^2(m_1+m_2+m_3))$  where $|E|$ is the number of edges and $T=O(m_1+m_2+m_3)$ is the time horizon of the time expanded graph~\cite{ford1956maximal}.
Note that any unlabeled \gls{mpp} can be applied, for example,  the algorithm in~\cite{yu2012distance} is distance-optimal but with convergence time guarantee. 
%

Note that the choice of 2D plane can also be $x$-$z$ plane or $y$-$z$ plane.
In addition, one can also perform two ``wide column" shuffles plus one $z$ shuffle, which yields $2(2m_1+m_2)+m_3$ number of shuffles and $O(m_1m_2m_3^2+m_3m_1m_2^2)$. 
This requires more shuffles but shorter running time.

Next, we derive the optimality guarantee.

\begin{proposition}[Makespan Upper Bound]
\gls{rthddd} returns  solution with worst makespan $3m_1+4m_2+4m_3+o(m_1)$.
\end{proposition}

\begin{proof}
The Rubik Table portion has a makespan of $(2m_3+2m_2+m_1)+o(m_1)$.
For the unlabeled \gls{mpp} portion, we use $m_1+m_2+m_3$, which is the maximum grid distance between any two nodes on the $m_1 \times m_2 \times m_3$ grid, as a conservative makespan upper bound.
Therefore, the total makespan upper bound is  $2(m_1+m_2+m_3)+(2m_3+2m_2+m_1)+o(m_1)=3m_1+4m_2+4m_3+o(m_1).$
When $m_1=m_2=m_3$ for cubic grids, it turns out to be $11m_3+o(m_3).$
\end{proof}

We note that the unlabeled \gls{mpp} upper bound is actually a very  conservative estimation. 
In practice, for such dense instances, the unlabeled \gls{mpp} usually requires much fewer steps.
Next, we analyze the makespan when start and goal configurations are uniformly randomly generated based on the well-known \emph{minimax grid matching} result.
\begin{theorem}[Multi-dimensional Minimax Grid Matching~\cite{shor1991minimax} ]\label{t:minimax3d}
For $k\geq 3,$
consider $N$ points following the uniform
distribution in $[0,N^{1/k}]^k$. 
Let $\mathcal{L}$ be the minimum length such that there exists
a perfect matching of the $N$ points to the grid points  that are regularly spaced in the
$[0,N^{1/k}]^k$ for which the distance between every pair of matched
points is at most $\mathcal{L}$. Then $\mathcal{L} = O(\log^{1/k}N)$ with high probability.
\end{theorem}

\begin{proposition}[Asymptotic Makespan]\label{p:expected_makespan}
\gls{rthddd} returns solutions with $m_1+2m_2+2m_3+o(m_1)$ asymptotic makespan for \gls{mpp} instances with $\frac{m_1m_2m_3}{3}$ random start and goal configurations on 3D grids, with high probability.
Moreover, if $m_1=m_2=m_3$, \gls{rthddd} returns $5m_3+o(m_3)$ makespan-optimal solutions.
\end{proposition}

\begin{proof}
First, we point out that the theorem of minimax grid matching (\ref{t:minimax3d}) can be generalized to any rectangular cuboid as the matching distance mainly depends on the discrepancy of the start and goal distributions, which does not depend on the shape of the grid. 
Using the minimax grid matching result, the matching distance from a random configuration to a centering configuration, which is also the underestimated makespan of unlabeled \gls{mpp}, scales as $O(\log^{1/3}m_1)$.
Using the \ref{lemma: makespan}, it is not difficult to see that the matching obtained this way can be readily turned into an unlabeled \gls{mpp} plan without increasing the maximum per robot travel distance by much, which remains at $o(m)$.
Therefore, the asymptotic makespan of \gls{rthddd} is $m_1+2m_2+2m_3+O(\log^{1/3}m_1)=m_1+2m_2+2m_3+o(m_1)$ with high probability.
If $m_1=m_2=m_3$, the asymptotic makespan  is $5m_3+o(m_3)$, with high probability.
\end{proof}

\begin{proposition}[Asymptotic Makespan Lower Bound]\label{p:asymptotic_lowerbound}
For \gls{mpp} instances on $ m_1\times  m_2\times m_3$ grids with $\Theta(m_1m_2m_3)$ random start and goal configurations on 3D grids, the makespan lower bound is asymptotically approaching $m_1+m_2+m_3$, with high probability.
\end{proposition}

\begin{proof}
We examine two opposite corners of the $m_1\times  m_2\times m_3$ grid. 
At each corner, we examine an $\alpha  m_1\times \alpha  m_2\times \alpha m_3$ sub-grid for some constant $\alpha \ll 1$. 
The Manhattan distance between a point in the first sub-grid and another point in the second sub-grid is larger than $(1-6\alpha)(m_1+m_2+m_3)$. 
As $m\to \infty$, with $\Theta(m^3)$ robots,  a start-goal pair falling into these two sub-grids
 can be treated as an event following binomial distribution $B(k,\alpha^6)$ when $m\rightarrow \infty$.
 The probability of having at least one success trial among $n$ trials is $p=1-(1-\alpha^6)^n$, which goes to one when $m\rightarrow \infty$.

Because $(1 - x)^y < e^{-xy}$ for $0 < x < 1$ and $y > 0$,
$p > 1 - e^{-\alpha^6n}$. Therefore, for arbitrarily small $\alpha$, we may choose $m_1$ such that $p$ is arbitrarily close to $1$. 
The probability of a start-goal pair falling into these two sub-grids is asymptotically approaching one 
, meaning that the makespan goes to $(1-6\alpha)(m_1+m_2+m_3)$. 
\end{proof}

\begin{corollary}[Asymptotic Makespan Optimality Ratio]
\gls{rthddd} yields asymptotic $1+\frac{m_2+m_3}{m_1+m_2+m_3}$ makespan optimality ratio for \gls{mpp} instances with $\Theta(m_1m_2m_3)\le \frac{m_1m_2m_3}{3}$ random start and goal configurations on 3D grids, with high probability.
\end{corollary} 

\begin{proof}
If $n < \frac{m_1m_2m_3}{3}$, we add virtual robots with randomly generated start and use the same for the goal, until we reach $n =\frac{m_1m_2m_3}{3}$ robots, which then allows us to invoke Proposition~\ref{p:expected_makespan}.
In viewing Proposition~\ref{p:expected_makespan} and Proposition~\ref{p:asymptotic_lowerbound}, solution computed by \gls{rthddd} guarantees a makespan optimality ratio of $\frac{m_1+2m_2+2m_3}{m_1+m_2+m_3}=1+\frac{m_2+m_3}{m_1+m_2+m_3}$ as $m_3\to \infty$.
Moreover, if $m_1=m_2=m_3$, \gls{rthddd} returns $\frac{5}{3}$ makespan-optimal solution.
\end{proof}


\begin{corollary}[Asymptotic Optimality, Fixed Height]\label{c:fh}
For an \gls{mpp} on an $ m_1\times  m_2\times K$ grid, $m_1>m_2\gg K$, and $\frac{1}{3}$ robot density, the \gls{rthddd} algorithm yields $1+\frac{m_1}{m_1+m_2}$ optimality ratio, with high probability.
If $m_1=m_2$, the asymptotic makespan optimality ratio is $1.5$.
\end{corollary}

\begin{theorem}
 Consider an $k$-dimensional cubic grid with grid size $m$. If robot density is less than $1/3$ and start and goal configurations are uniformly distributed, generalizations to \gls{rtatwo} can solve the instance with asymptotic makespan optimality being $\frac{2^{k-1}+1}{k}$.
\end{theorem}

\begin{proof}
By theorem \ref{t:minimax3d}, the unlabeled \gls{mpp} takes $o(m)$ makespan (note for $k=1,2$, the minimax grid matching distance is still $o(m)$~\cite{leighton1989tight}).
Extending proposition \ref{p:asymptotic_lowerbound} to $k$-dimensional grid, the asymptotic lower bound is $mk$.
We now prove that the asymptotic makespan $f(k)$ is $(2^{k-1}+1)m+o(m)$ by induction.
The Rubik Table algorithm solves a $d$-dimensional problem by using two 1-dimensional shuffles and one $(k-1)$-dimensional ``wide column" shuffle.
Therefore, we have $f(k)=2m+f(k-1)$.
It's trivial to see $f(1)=m+o(m),f(2)=3m+o(m)$, which yields that $f(k)=2^{k-1}m+m+o(m)$ and makespan optimality ratio being $\frac{2^{k-1}+1}{k}$.
\end{proof}

\section{Simulation Experiments}\label{sec:eval}
We thoroughly evaluated \gls{rta}-based algorithms and compared them with many similar algorithms. 
We mainly highlight comparisons with \gls{eecbs}($w$=1.5)~\cite{li2021eecbs}, \gls{lacam}~\cite{okumura2023lacam} and \gls{ddm}~\cite{han2020ddm}. 
These two methods are, to our knowledge, two of the fastest near-optimal \gls{mpp} solvers.
Beyond \gls{eecbs} and \gls{ddm}, we considered a state-of-the-art polynomial algorithm, push-and-swap~\cite{luna2011push}, which gave fairly sub-optimal results: the makespan optimality ratio is often above 100 for the densities we examine. 
%

As a reader's guide to this section, in ~\ref{e:1}, as a warm-up, we show the 2D performance of all \gls{rta}-based algorithms at their baseline, i.e., without any efficiency-boosting heuristics mentioned in ~\ref{sec:opt-boost}. 
In ~\ref{e:2}, for 2D square grids, we show the performance of all \gls{rtatwo}-based algorithms with and without the two heuristics discussed in ~\ref{sec:opt-boost}. We then thoroughly evaluate the performance of all versions of the \gls{rtatwo} algorithm at $\frac{1}{3}$ robot density in ~\ref{subsec:rth}. Some special 2D patterns are examined in ~\ref{e:4}. 3D settings are briefly discussed in ~\ref{e:5}. 

All experiments are performed on an Intel\textsuperscript{\textregistered} Core\textsuperscript{TM} i7-6900K CPU at 3.2GHz. Each data point is an average of over 20 runs on randomly generated instances unless otherwise stated.
A running time limit of $300$ seconds is imposed over all instances. 
The optimality ratio is estimated as compared to conservatively estimated makespan lower bounds.
All the algorithms are implemented in C++.
%
We choose Gurobi~\cite{gurobi} as the mixed integer programming solver and  ORtools~\cite{ortools} as the max-flow solver.

\subsection{Optimality of Baseline Versions of \gls{rta}-Based Methods}\label{e:1}

First, we provide an overall evaluation of the optimality achieved by basic versions of \gls{rtmdd}, \gls{rtlm}, and \gls{rth} over randomly generated 2D instances at their maximum designed robot density. 
That is, these methods do not contain the heuristics from ~\ref{sec:opt-boost}.
We test over three $m_1:m_2$ ratios: $1:1$, $3:2$, and $5:1$. 
For \gls{rtmdd}, different sub-grid sizes for dividing the $m_1 \times m_2$ grid are evaluated. 
%
The result is plotted in~\ref{fig:RTM-RTLM-RTH}. 
Computation time is not listed; we provide the computation time later for \gls{rth}; the running times of \gls{rtmdd}, \gls{rtlm}, and \gls{rth} are all similar. 
The optimality ratio is computed as the ratio between the solution makespan and the longest Manhattan distance between any pair of start and goal, which is conservative. 
\begin{figure}[htb]
        \centering
        \includegraphics[width=1\linewidth]{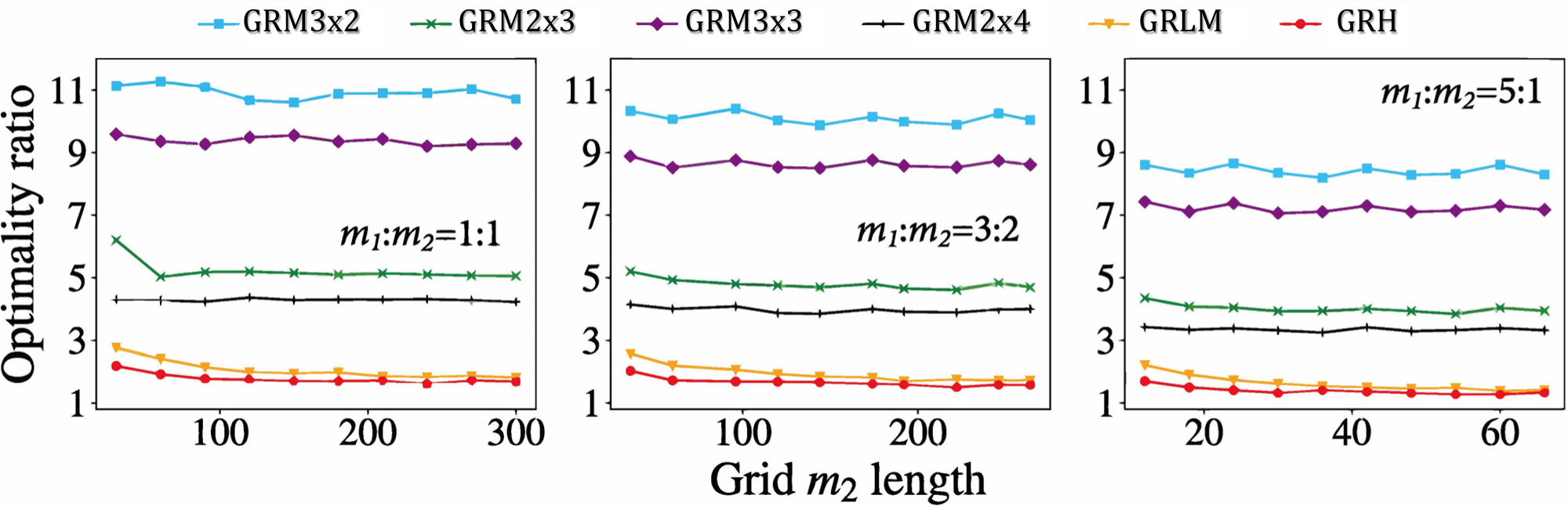}
\caption{Makespan optimality ratio for  \gls{rtmdd} (3x2, 2x3, 3x3, 2x4), \gls{rtlm}, and \gls{rth} at their maximum design density, for different $m_2$ values and $m_1:m_2$ ratios. The largest \gls{rtmdd} problems have $90,000$ robots on a $300 \times 300$ grid.} 
        \label{fig:RTM-RTLM-RTH}
    \end{figure}

\gls{rtmdd} does better and better on the optimality ratio as the sub-grid size ranges from $3\times 2$, $3\times 3$, $2\times 3$, and $2\times 4$, dropping to just above $3$.
%
In general, using ``longe'' sub-grids for the shuffle will decrease the optimality ratio because there are opportunities to reduce the overhead.
However, the time required for computing the solutions for all possible configurations grows exponentially as the size of the sub-grids increases.
%

On the other hand, both \gls{rtlm} and \gls{rth} achieve a sub-2 optimality ratio in most test cases, with the result for \gls{rth} dropping below $1.5$ on large grids. For all settings, as the grid size increases, there is a general trend of improvement of optimality across all methods/grid aspect ratios. This is due to two reasons: (1) the overhead in the shuffle operations becomes relatively smaller as grid size increases, and (2) with more robots, the makespan lower bound becomes closer to $m_1 + m_2$. Lastly, as $m_1:m_2$ ratio increases, the optimality ratio improves as predicted. For many test cases, the optimality ratio for \gls{rth} at $m_1:m_2=5$  is around $1.3$.

\begin{table}[h]
\small
  \centering
  \begin{tabular}{|c|c|c|c|c|c|}
    \hline
    \textbf{$\#$ of Robots} & 5 & 10  &15  &20 &25 \\
\hline
\textbf{Optimality Gap} & 1 &1.0025 &1.004 &1.011  &1.073\\
\hline
  \end{tabular}
  \caption{Optimality gap investigation on $5\times 5$ grids }
  \label{tab:OptimalityGap}
\end{table}

The exploration of optimality gaps is conducted on $5\times 5$ grids, as shown in~\ref{tab:OptimalityGap}. For every specified number of robots, we create 100 random instances and employ the ILP solver to solve them. The optimality gap is then assessed by calculating the average ratio between the optimal makespan and the makespan lower bound. 
The optimality gap widens with higher robot density.

\subsection{Evaluation and Comparative Study of \gls{rth}}\label{subsec:rth}
\subsubsection{Impact of Grid Size}
For our first detailed comparative study of the performance of \gls{rta},\gls{rtlm} and \gls{rth} at $100\%$, $\frac{1}{2}$ and $\frac{1}{3}$ density respectively, we set $m_1:m_2 = 3:2$ in terms of computation time and makespan optimality ratio.
We compare with \gls{ddm}~\cite{han2020ddm}, EECBS ($w=1.5$)~\cite{li2021eecbs}, Push and Swap\cite{luna2011push}, \gls{lacam}~\cite{okumura2023lacam} in~\ref{fig:revise_full}-\ref{fig:revise_third}. For \gls{eecbs}, we turn on all the available heuristics and reasonings.

\begin{figure}[htb]
        \centering
        \includegraphics[width=1\linewidth]{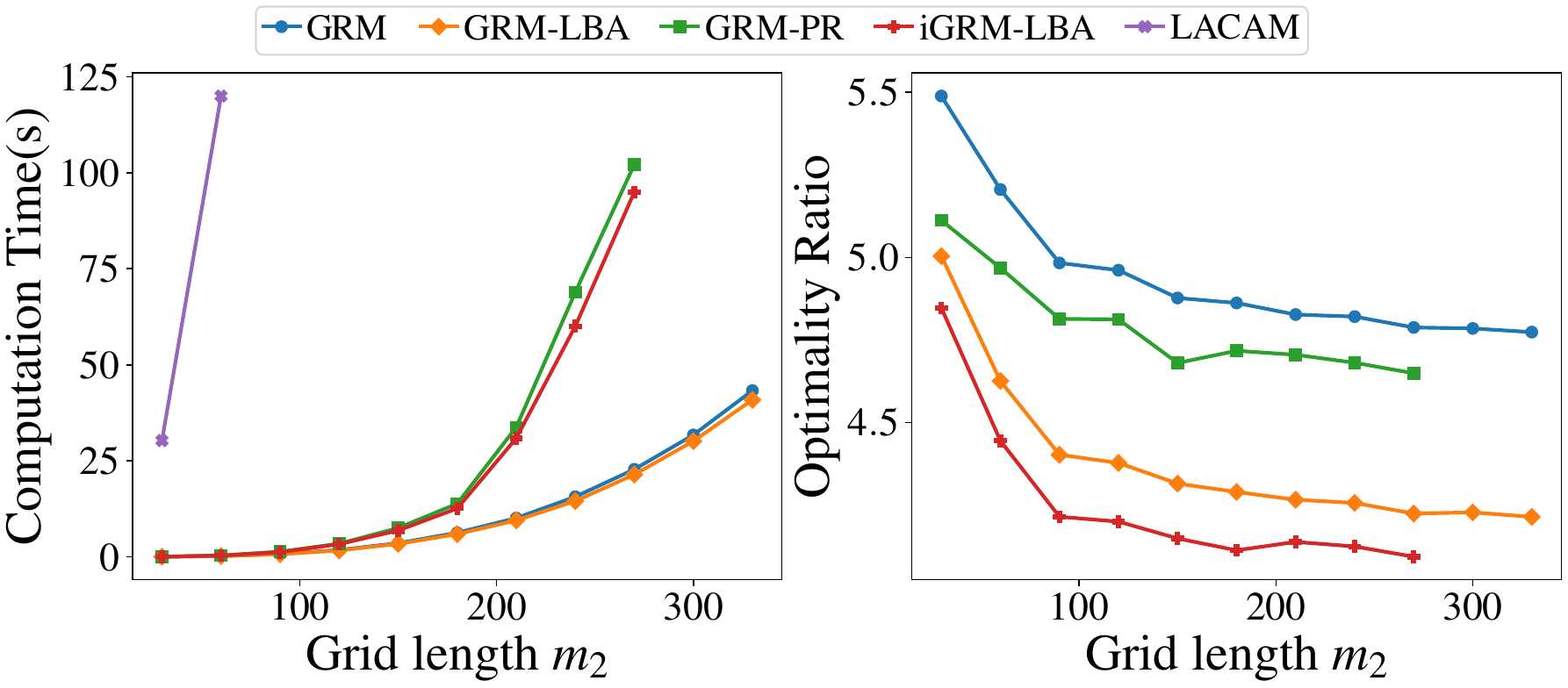}
        \caption{Computation time and optimality ratios on $m_1 \times m_2$ grids of varying sizes with $m_1:m_2 = 3:2$ and robot density at $100\%$ density.} 
        \label{fig:revise_full}
\end{figure}

\begin{figure}[htb]
        \centering
        \includegraphics[width=1\linewidth]{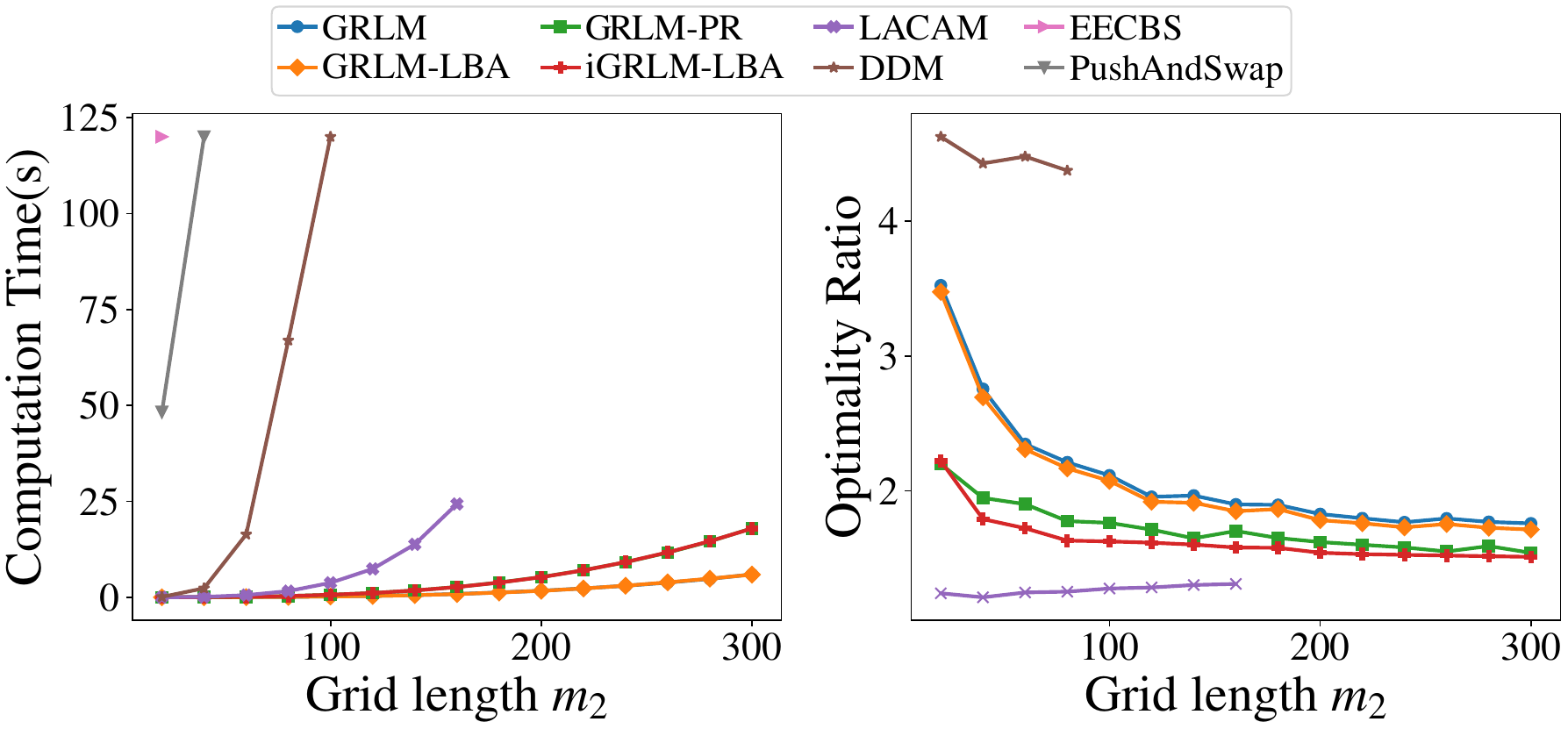}
        \caption{Computation time and optimality ratios on $m_1 \times m_2$ grids of varying sizes with $m_1:m_2 = 3:2$ and robot density at $\frac{1}{2}$ density.} 
        \label{fig:revise_half}
    \end{figure}

\begin{figure}[htb]
        \centering
        \includegraphics[width=1\linewidth]{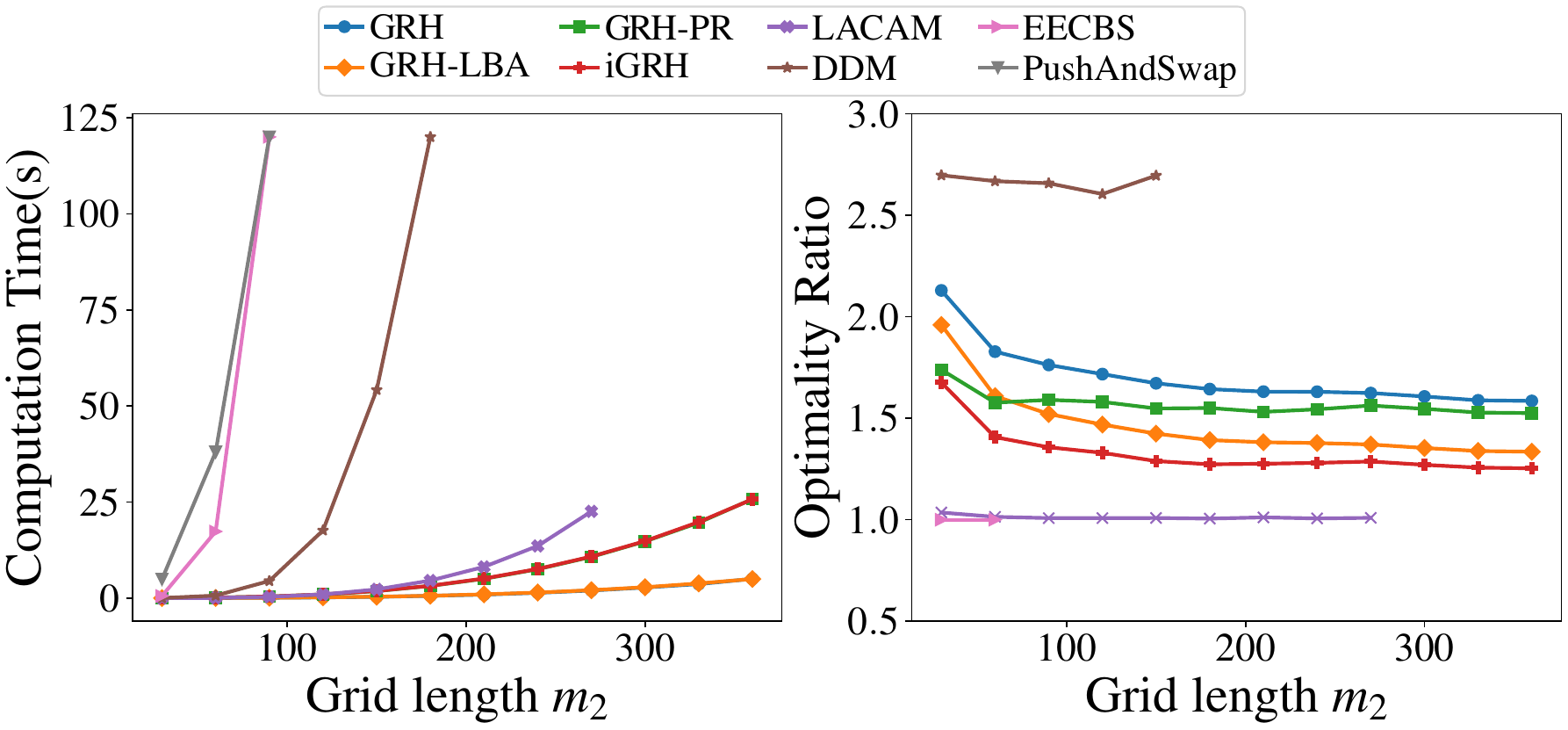}
        \caption{Computation time and optimality ratios on $m_1 \times m_2$ grids of varying sizes with $m_1:m_2 = 3:2$ and robot density at $1/3$ density.} 
        \label{fig:revise_third}
\end{figure}
    
%
When at $100\%$ robot density, \gls{rtmdd} methods can solve huge instances, e.g., on $450 \times 300$ grids with $135,000$ robots in about $40$ seconds while none of the other algorithms can.
\gls{lacam} can only solve instances when $m_2=30$, though the resulting makespan optimality is around $2073$ and thus is not shown in the figure.
When at $\frac{1}{2}, \frac{1}{3}$ robot density, \gls{rtlm} and \gls{rth} still are the fastest methods among all, scaling to 45,000 robots in 10 seconds while the optimality ratio is close to 1.5 and 1.3 respectively.
\gls{lacam} also achieves great scalability and is able to solve problems when $m_2 \leq 270$. However, after that, \gls{lacam} faces out-of-memory error. This is due to the fact that \gls{lacam} is a search algorithm on joint configurations, and the required memory grows exponentially as the number of robots. In contrast, \gls{rth}, \gls{rtlm} do not have the issue and solve the problems consistently despite the optimality being worse than \gls{lacam}.
\gls{eecbs}, stopped working after $m_2 = 90$ at $\frac{1}{3}$ robot density and cannot solve any instance within the time limit at $100\%$ and $\frac{1}{2}$ robot density, while \gls{ddm} stopped working after $m_2=180$.
Push and Swap also performed poorly, and its optimality ratio is more than 30 in those scenarios and thus is not presented in the figure.

\subsubsection{Handling Obstacles}
\gls{rth} can also handle scattered obstacles and is especially suitable for cases where obstacles are regularly distributed. 
For instance, problems with underlying graphs like that in Fig.~\ref{fig:jd_center}(b), where each $3 \times 3$ cell has a hole in the middle, can be natively solved without performance degradation.
Such settings find real-world applications in parcel sorting facilities in large warehouses~\cite{wan2018lifelong,li2020lifelong}.
For this parcel sorting setup, we fix the robot density at $\frac{2}{9}$ and test \gls{eecbs}, \gls{ddm}, \gls{rth}, \gls{rth}lba, \gls{rth}pr, i\gls{rth} on graphs with varying sizes. 
The results are shown in Fig~\ref{fig: sorting_random}.
Note that \gls{ddm} can only apply when there is no narrow passage. So we added additional ``borders'' to the map to make it solvable for \gls{ddm}.
The results are similar as before; we note that i\gls{rth} reaches a conservative optimality ratio estimate of $1.26$.

   \begin{figure}[htb]
        \centering
        \includegraphics[width=1\linewidth]{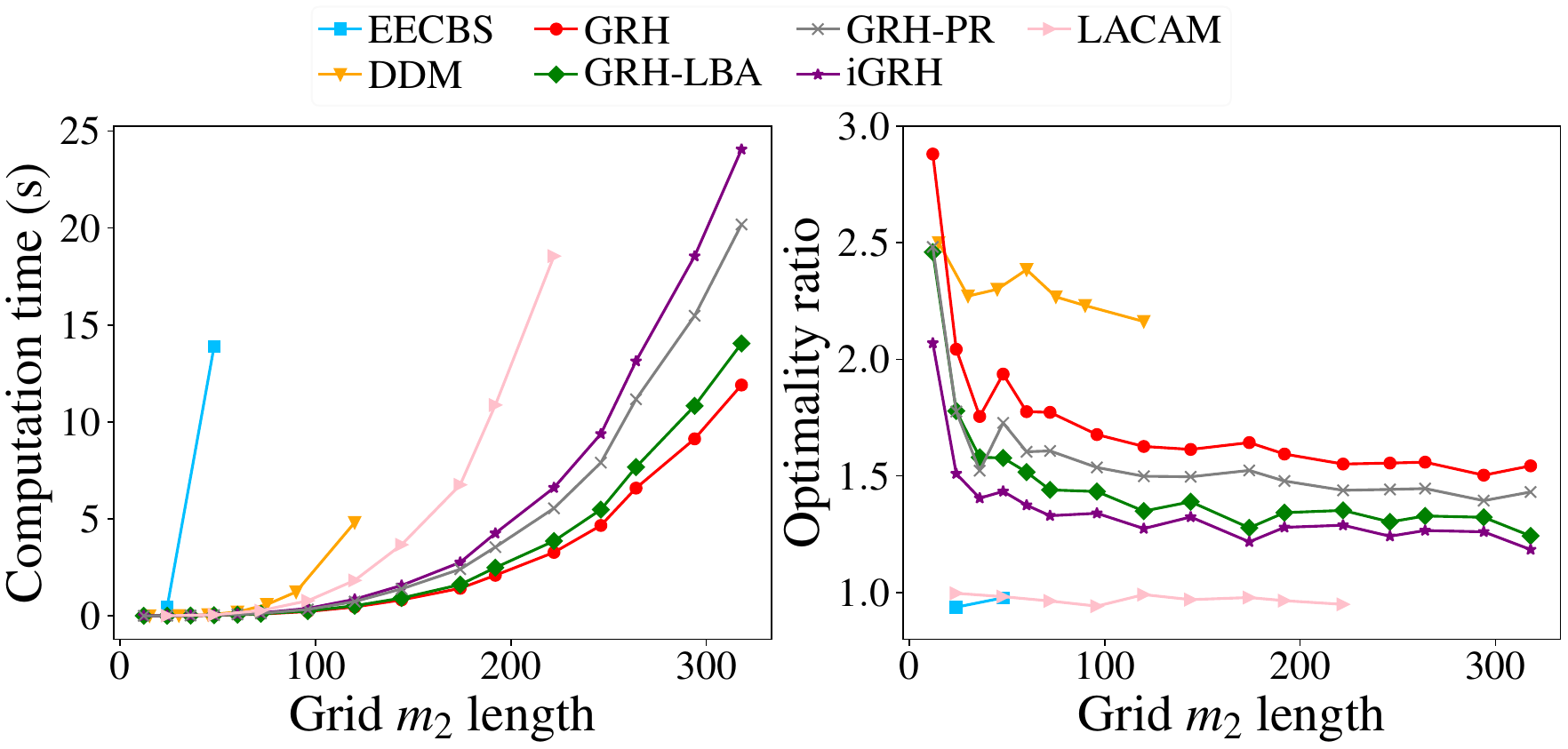}
        \caption{Computation time and optimality ratios on environments of varying sizes with regularly distributed obstacles at $\frac{1}{9}$ density and robots at $\frac{2}{9}$ density. $m_1:m_2 = 3:2$.
        } 
        \label{fig: sorting_random}
    \end{figure}

\subsubsection{Impact of Grid Aspect Ratios}
In this section, we fix $m_1m_2=90000$ and vary the $m_2:m_1$ ratio between $0$ (nearly one dimensional) and $1$ (square grids). We evaluated four algorithms, two of which are \gls{rth} and i\gls{rth}. Now recall that \gls{rtp} on an $m_1 \times m_2$ table can also be solved using $2m_2$ column shuffles and $m_1$ row shuffles. Adapting \gls{rth} and i\gls{rth} with $m_1 + 2m_2$ shuffles gives the other two variants we denote as \gls{rth}-LL and i\gls{rth}-LL, with ``LL'' suggesting two sets of longer shuffles are performed (each set of column shuffle work with columns of length $m_1$). The result is summarized in~\ref{fig:rectangle}.
  \begin{figure}[htb]

        \centering
        \includegraphics[width=1\linewidth]{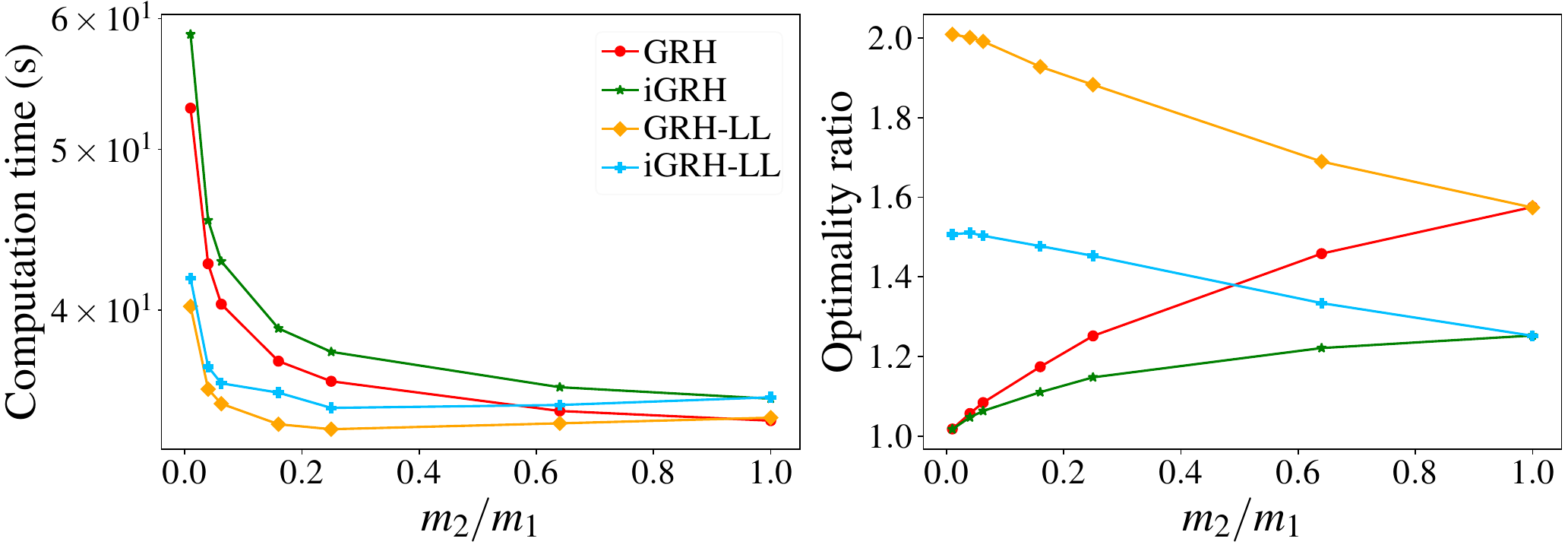}
        \caption{Computation time and optimality ratios on rectangular grids of varying aspect ratio and $\frac{1}{3}$ robot density.} 
        \label{fig:rectangle} 
    \end{figure}
    
Interestingly but not surprisingly, the result clearly demonstrates the trade-offs between computation effort and solution optimality. \gls{rth} and i\gls{rth} achieve better optimality ratio in comparison to \gls{rth}-LL and i\gls{rth}-LL but require more computation time.
Notably, the optimality ratio for \gls{rth} and i\gls{rth} is very close to 1 when $m_2:m_1$ is close to 0.
As expected, i\gls{rth} does much better than \gls{rth} across the board. 
\subsection{Special Patterns}\label{e:4}
We experimented i\gls{rth} on many ``special'' instances, two are presented here (\ref{fig:new_test}). For both settings, we set $m_1 = m_2$. In the first, the ``squares'' setting, robots form concentric square rings and each robot and its goal are centrosymmetric. 
In the second, the ``blocks'' setting, the grid is divided into smaller square blocks (not necessarily $3 \times 3$) containing the same number of robots. Robots from one block need to move to another randomly chosen block.
i\gls{rth} achieves optimality that is fairly close to 1.0 in the square setting and 1.7 in the block setting. The computation time is similar to that of~\ref{fig: sorting_random}; EECBS performs well in terms of optimality, but its scalability is limited, working only on grids with $m \leq 90$. On the other hand, \gls{lacam} exhibits excellent scalability and good optimality for block patterns, although its optimality is comparatively worse for square patterns. Other algorithms are excluded from consideration due to significantly poorer optimality; for example, DDM's optimality exceeds 9, and Push\&Swap's optimality is greater than 40.

\begin{figure}[htb]
        \centering
           \begin{overpic}               
        [width=1\linewidth]{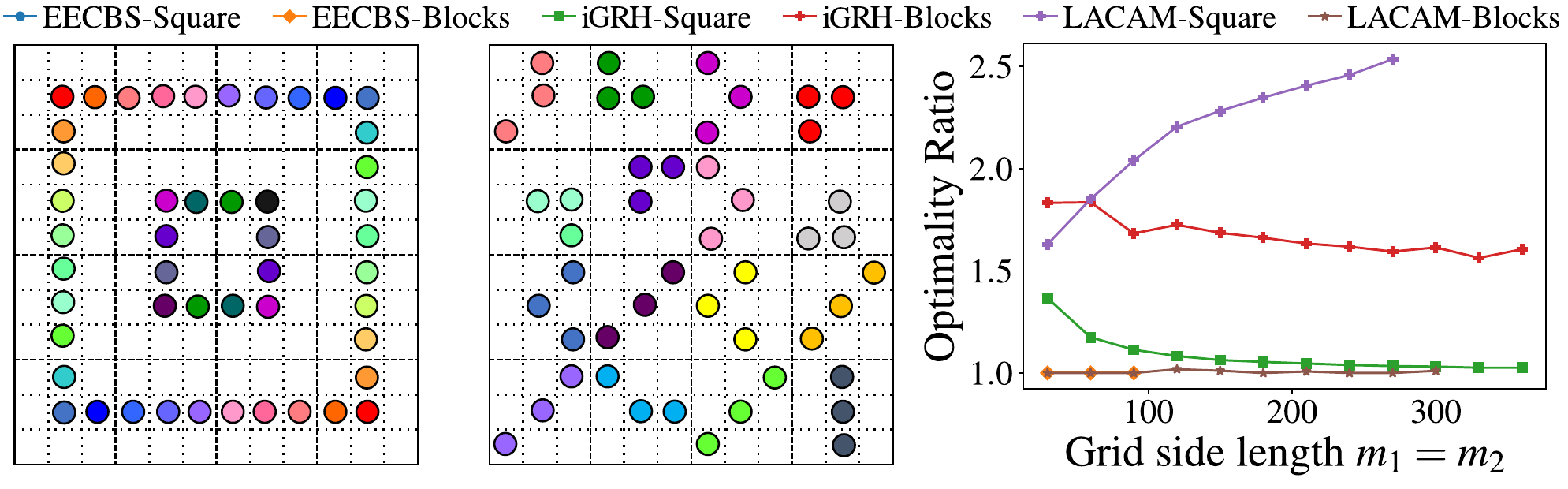}
             \footnotesize
             \put(12.5, -1) {(a)}
             \put(43.5, -1) {(b)}
            \put(81.5, -1) {(c)}
        \end{overpic}
        \caption{(a) An illustration of the ``squares'' setting. (b) An illustration of the ``blocks'' setting. (c) Optimality ratios for the two settings for \gls{eecbs}, \gls{lacam}, and \gls{rth}lba.} 
        \label{fig:new_test} 
\end{figure}

\subsection{Effectiveness of Matching and Path Refinement Heuristics}\label{e:2}
In this subsection, we evaluate the effectiveness of heuristics introduced in ~\ref{sec:opt-boost} in boosting the performance of baseline \gls{rta}-based methods (we will briefly look at the IP heuristic later). We present the performance of \gls{rtmdd}, \gls{rtlm}, and \gls{rth} at these methods' maximum design density. For each method, results on all $4$ combinations with the heuristics are included.  
For a baseline method X, X-\gls{lba}, X-PR, and iX mean the method with the \gls{lba} heuristic, the path refinement heuristic, and both heuristics, respectively. 
In addition to the makespan optimality ratio, we also evaluated \emph{sum-of-cost}  (SOC) optimality ratio, which may be of interest to some readers. The sum-of-cost is the sum of the number of steps taking individual robots to reach their respective goals. 
We tested the worst-case scenario, i.e., $m = m_1 = m_2$. 
The result is shown in Fig. ~\ref{fig:path-refinement}. 
\begin{figure}[h!]
    \centering
    \includegraphics[width=1\linewidth]{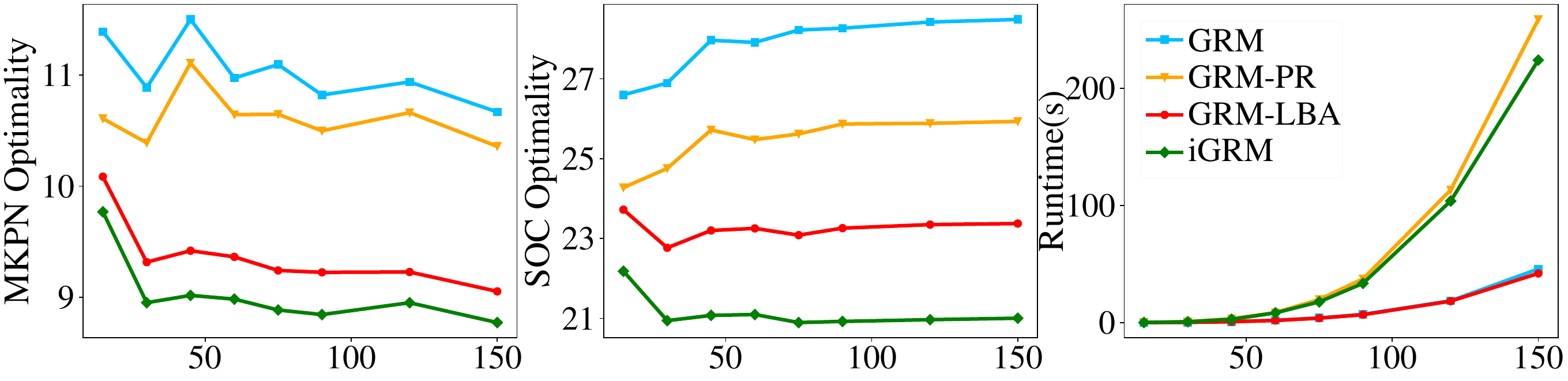}
    \includegraphics[width=1\linewidth]{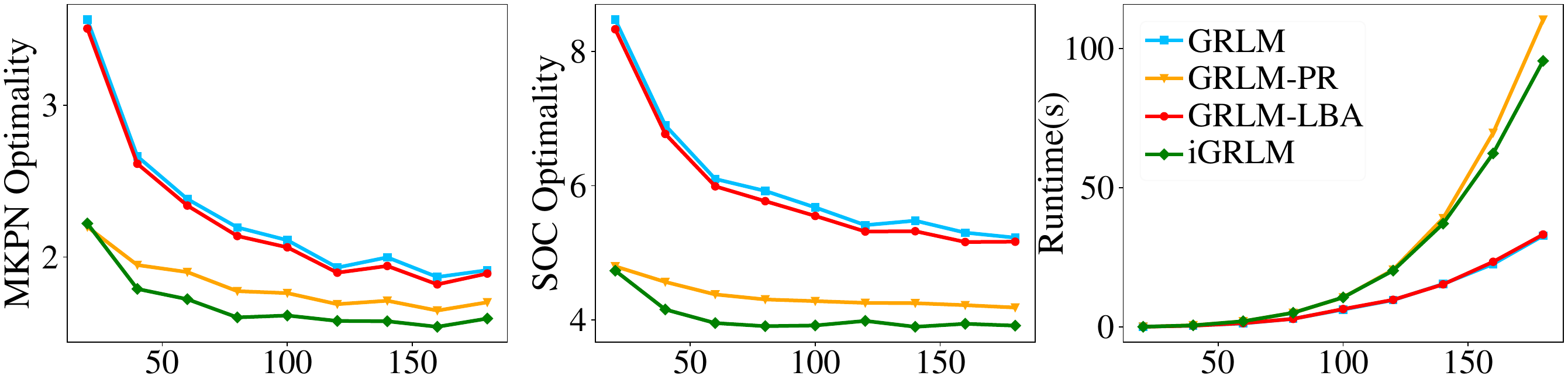}
    \includegraphics[width=1\linewidth]{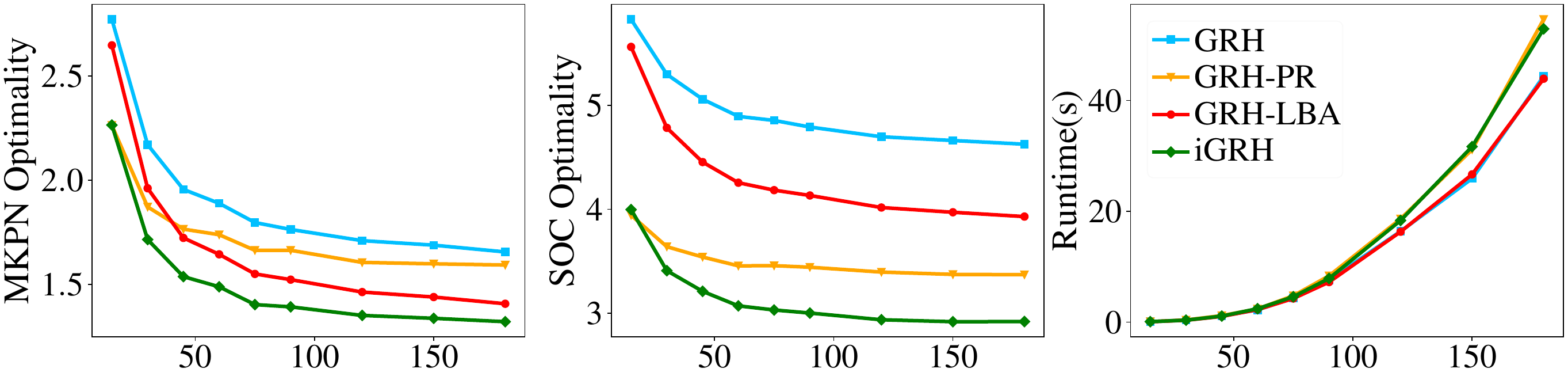}
    \caption[Effectiveness of heuristics in boosting \gls{rta}-based algorithms on $m\times m$ grids.]{Effectiveness of heuristics in boosting \gls{rta}-based algorithms on $m\times m$ grids. For all figures, the $x$-axis is the grid side length $m$. Each row shows a specific algorithm (\gls{rtmdd}, \gls{rtlm}, and \gls{rth}). From left to right, makespan(MKPN) optimality ratio, SOC optimality ratio, and the computation time required for the path refinement routine are given.} 
    \label{fig:path-refinement} 
\end{figure}

We make two key observations based on~\ref{fig:path-refinement}. First, both heuristics provide significant individual boosts to nearly all baseline methods (except \gls{rtlm}-\gls{lba}), delivering around $10\%$--$20\%$ improvement on makespan optimality and $30\%$--$40\%$ improvement on SOC optimality. Second, the combined effect of the two heuristics is nearly additive, confirming that the two heuristics are orthogonal two each other, as their designs indicate. 
%
%
The end result is a dramatic overall cross-the-board optimality improvement. As an example, for \gls{rth}, for the last data point, the makespan optimality ratio dropped from around $1.8$ for the based to around $1.3$ for i\gls{rth}. 
In terms of computational costs, the \gls{lba} heuristic adds negligible more time. The path refinement heuristic takes more time in full and $\frac{1}{2}$ density settings but adds little cost in the $\frac{1}{3}$ density setting. 


\subsection{Evaluations on 3D Grids}\label{e:5}
In the first evaluation, we fix the aspect ratio $m_1:m_2:m_3=4:2:1$ and $\frac{1}{3}$ robot density,  and examine the performance \gls{rta} methods on obstacle-free grids with varying size.
Start and goal configurations are randomly generated; the results are shown in
\ref{fig:random_rec}.
\gls{ilp} with 16-split heuristic and \gls{ecbs} computes solution with better optimality ratio but have poor scalability.
In contrast, \gls{rthddd} and \gls{rthdddlba} readily scale to grids with over $370,000$ vertices and $120,000$ robots.
Both of the optimality ratio of \gls{rthddd} and \gls{rthdddlba} decreases as the grid size increases, asymptotically approaching $\sim 1.7$ for \gls{rthddd} and $\sim 1.5$ for \gls{rthdddlba}.

\begin{figure}[htbp]
        \centering
        \includegraphics[width=1\linewidth]{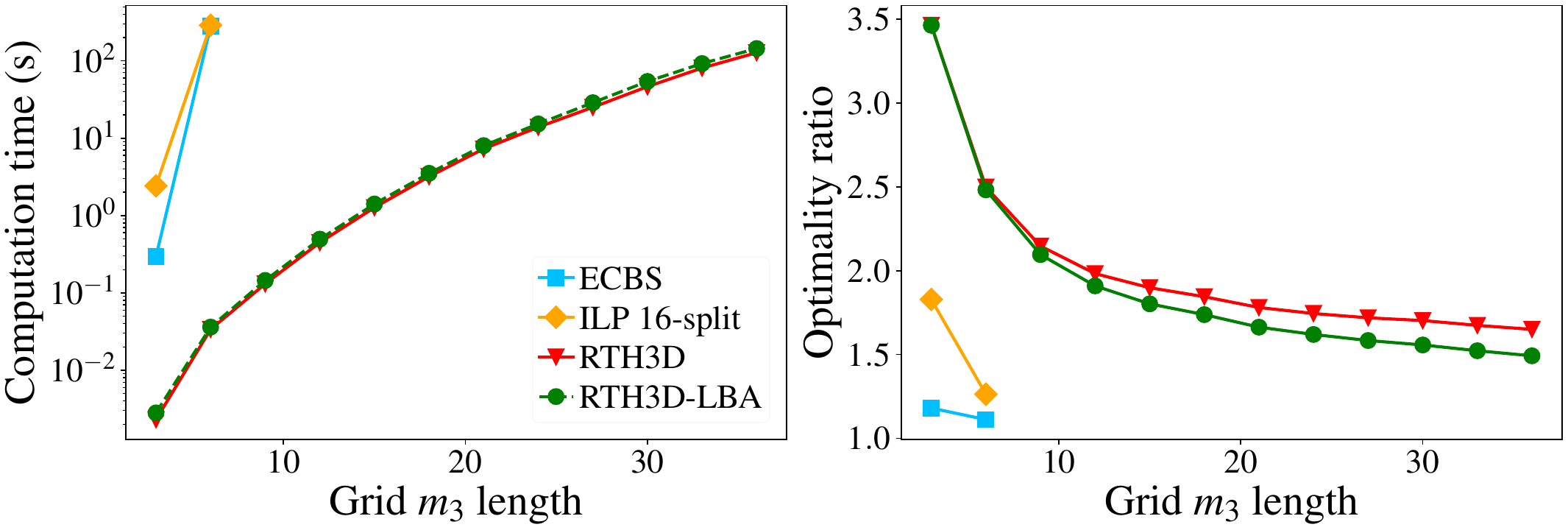}
\caption{Computation time and optimality ratio on grids with varying grid size and $m_1:m_2:m_3=4:2:1$.} 
        \label{fig:random_rec}
\end{figure}

In a second evaluation, we fix $m_1:m_2=1:1, m_3=6$ and robot density at $\frac{1}{3}$, 
and vary  $m_1$.
As shown in~\ref{fig:flat}, \gls{rthddd} and \gls{rthdddlba} show exceptional scalability and the asymptotic 1.5 makespan optimality ratio, as predicted by~\ref{c:fh} ($m_3$ is a constant and $m_1 = m_2$).
\begin{figure}[htbp]
        \centering
        \includegraphics[width=1\linewidth]{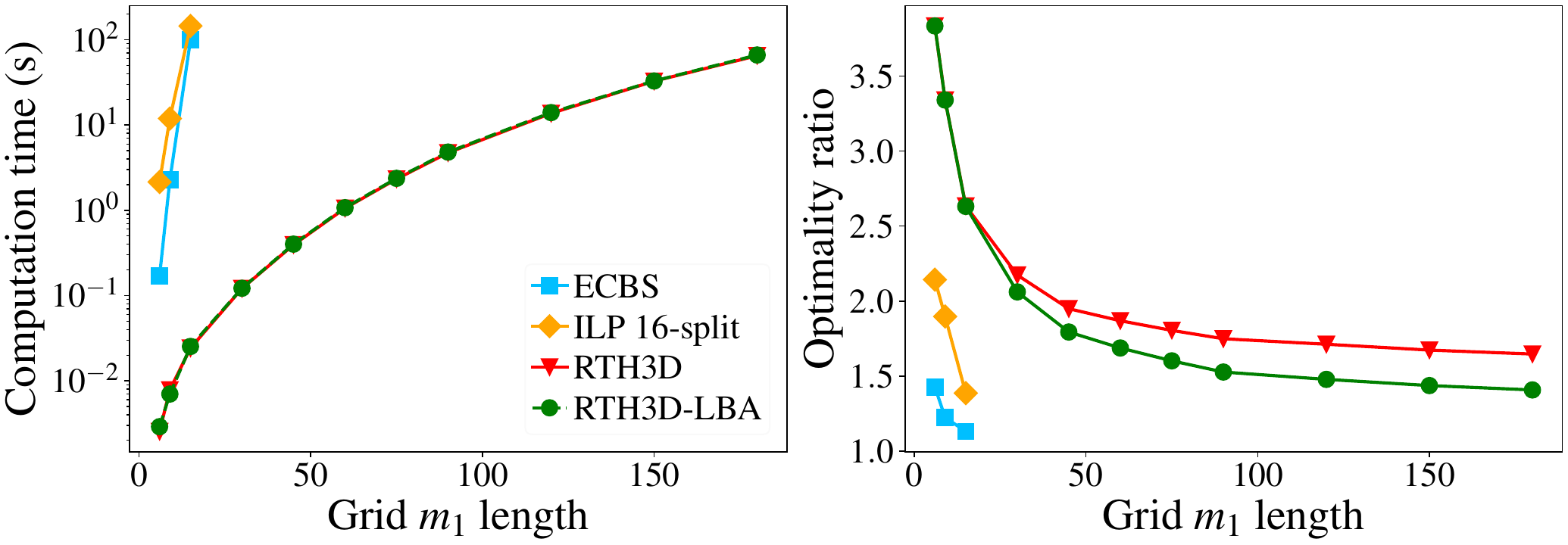}
\caption{Computation time and optimality ratio on grids with $m_1:m_2=1$ and $m_3=6$.} 
        \label{fig:flat}
\end{figure}
    
    \gls{rthddd} and \gls{rthddd}lba  support scattered obstacles where obstacles are small and regularly distributed.
As a demonstration of this capability, we simulate setting with an environment containing many tall buildings, a snapshot of which is shown  Fig.~\ref{fig:drones_in_cities}.
\begin{figure}[!htbp]
        \centering
        \includegraphics[width=1\linewidth]{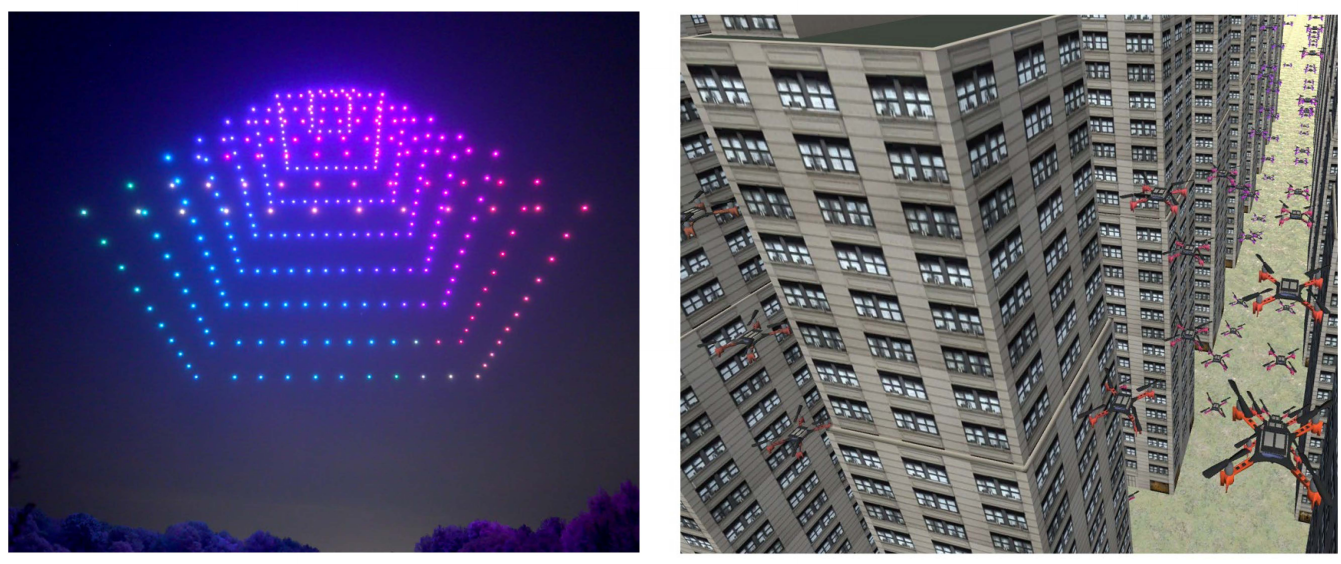}
  
        \caption{[left] A real-world drone light show where the UAVs form a 3D grid like pattern \cite{drone-flight}. [right] Our simulated large UAV swarm in the Unity environment with many tall buildings as obstacles. A video of the simulations and experiments can be found at \url{https://youtu.be/v8WMkX0qxXg}} 
        \label{fig:drones_in_cities}
\end{figure}
These ``tall building'' obstacles are located at positions $(3i+1,3j+1)$, which 
corresponds to an obstacle density of over $10\%$.
UAV swarms have to avoid colliding with the tall buildings and are not allowed to fly higher than those buildings.
We fix the aspect ratio at $m_1:m_2:m_3=4:2:1$ and robot density at $\frac{2}{9}$.
The evaluation result is shown in~\ref{fig:obs_rec}.
\begin{figure}[htbp]
\vspace{-1.5mm}
        \centering
        \includegraphics[width=1\linewidth]{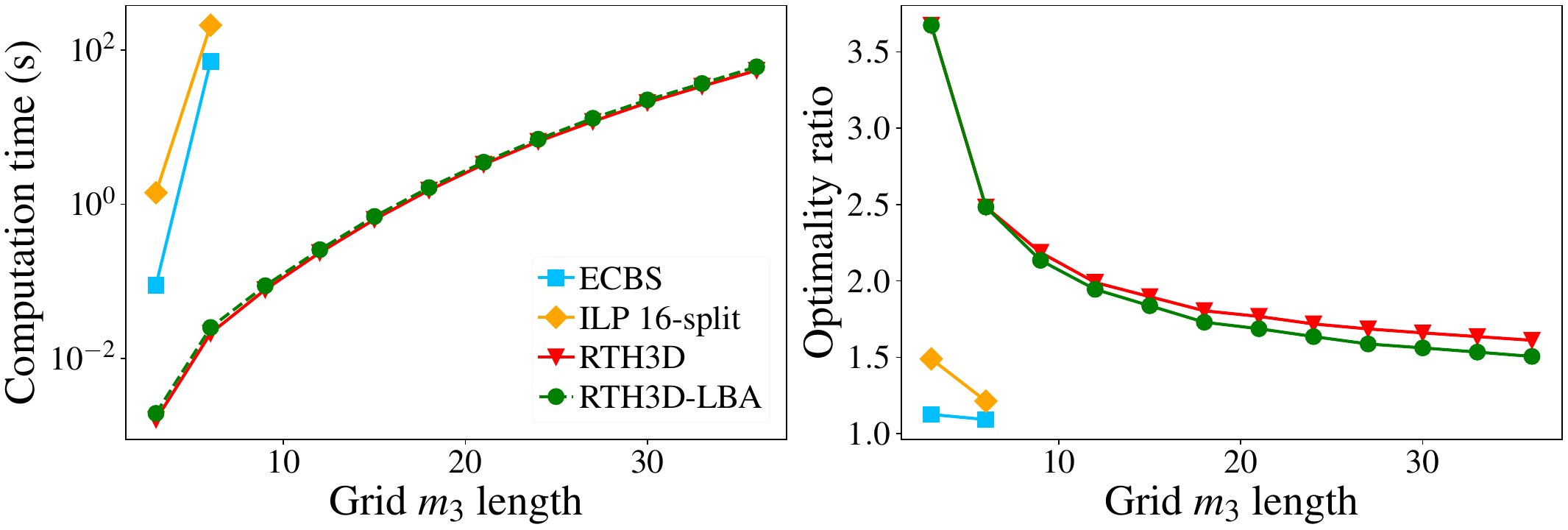}
\caption{Computation time and optimality ratio on 3D environments with obstacles. $m_1:m_2:m_3=4:2:1$.} 
        \label{fig:obs_rec}
\vspace{-1.5mm}
    \end{figure}

\subsection{Impact of Robot Density}
Next, we vary robot density, fixing the environment as a $120\times 60\times 6$ grid.
For robot density lower than $\frac{1}{3}$, we add virtual robots with random start and goal configurations for perfect matching computation. 
When applying the shuffle operations and evaluating optimality ratios, virtual robots are removed.
Therefore, the optimality ratio is not overestimated.
The result is plotted in~\ref{fig:robot_density}, showing that robot density has little 
impact on the running time. 
On the other hand, as robot density increases, the lower bound and the makespan required by unlabeled \gls{mpp} are closer to theoretical limits. Thus, the makespan optimality ratio is actually better.
\begin{figure}[htbp]
        \centering
        \includegraphics[width=1\linewidth]{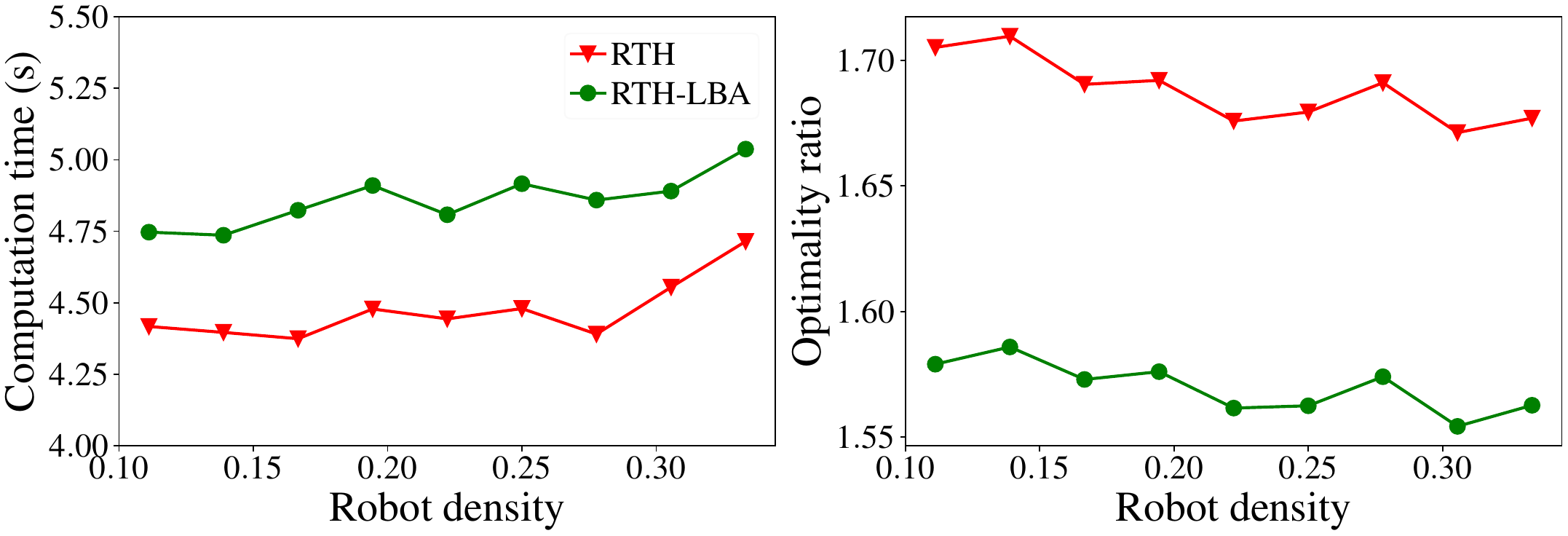}
\caption{Computation time and optimality ratio on a $120\times 60\times 6$ grid with varying robot density.} 
        \label{fig:robot_density}
    \end{figure}

\subsection{Special Patterns}
In addition to random instances, we also tested instances with start and goal configurations forming special patterns, e.g., 3D ``ring'' and ``block'' structures (\ref{fig:special_pattern}).
For both settings, $m_1 = m_2=m_3$.
In the first setting, ``rings'', robots form concentric square rings in each $x$-$y$ plane. Each robot and its goal are centrosymmetric.
In the ``block" setting, the grid is divided to $27$ smaller cubic blocks. The robots in one block need to move to another random chosen block.
\ref{fig:special_pattern}(c) shows the optimality ratio results of both settings on grids with varying size. Notice that the optimality ratio for the ring approaches $1$ for large environments.
\begin{figure}[htbp]
        \centering
        \includegraphics[width=1\linewidth]{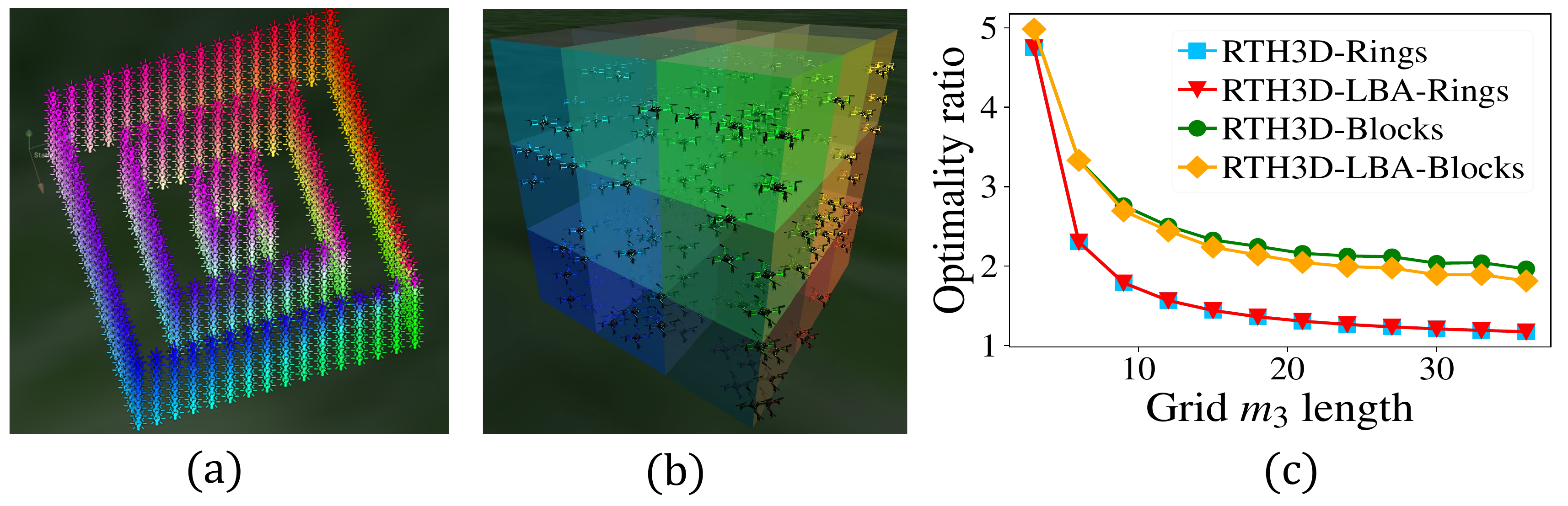}
\caption{Special patterns and the associated optimality ratios.} 
        \label{fig:special_pattern}
    \end{figure}

\subsection{Crazyswarm Experiment}
\begin{wrapfigure}[6]{r}{1.2in}
\vspace{-6.5mm}
  \begin{center}
    \includegraphics[width=1.2in]{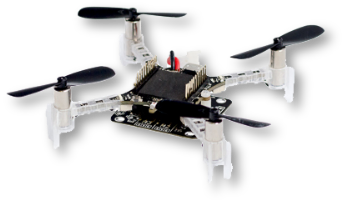}
  \end{center}
  \vspace{-4mm}
  \caption{Crazyflie 2.0 nano quadcopter.}
  \label{fig:cf2}
\end{wrapfigure}
Paths planned by \gls{rta} can also be readily transferred to real UAVs.
To 
demonstrate this, $10$ Crazyflie 2.0 nano quadcopters are choreographed to form the ``A-R-C-R-U'' pattern, letter by letter, on $6\times 6 \times 
6$ grids (\ref{fig:real_drone}).
The discrete paths are computed by \gls{rthddd}.
Due to relatively low robot density,  shuffle operations are further 
optimized to be more efficient.
Continuous trajectories are then computed  based on \gls{rthddd} plans by 
applying the method described in~\cite{honig2018trajectory}.
\begin{figure}[!htbp]
        \centering
        \includegraphics[width=1\linewidth]{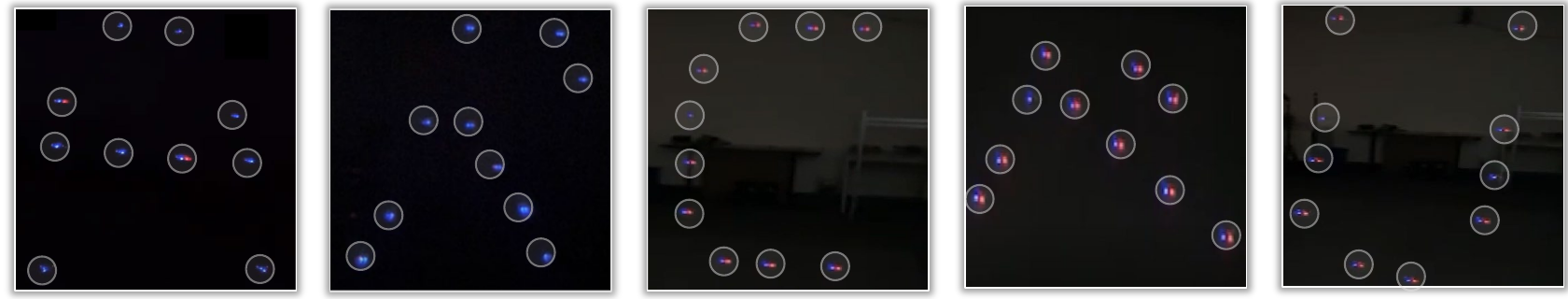}
\caption[A Crazyswarm~\cite{preiss2017crazyswarm} with 10 Crazyflies following paths computed by \gls{rta} on $6\times6\times 6$ grids, transitioning between letters ARC-RU. ]{A Crazyswarm~\cite{preiss2017crazyswarm} with 10 Crazyflies following paths computed by \gls{rta} on $6\times6\times 6$ grids, transitioning between letters ARC-RU. The figure shows five snapshots during the process. We note that some letters are skewed in the pictures which is due to non-optimal camera angles when the videos were taken.} 
        \label{fig:real_drone}
\end{figure}

\section{Conclusion and Discussion}\label{sec:conclusion}
In this study, we propose to apply Grid Rearrangements~\cite{szegedy2023rubik} to solve \gls{mpp}. 
A basic adaptation of \gls{rta}, with a more efficient line shuffle routine, enables solving \gls{mpp} on grids at maximum robot density, in polynomial time, with a previously unachievable optimality guarantee. 
Then, combining \gls{rta}, a highway heuristic, and additional matching heuristics, we obtain novel polynomial time algorithms that are provably asymptotically $1 + \frac{m_2}{m_1+m_2}$ makespan-optimal on $m_1\times m_2$ grids with up to $\frac{1}{3}$ robot density, with high probability. 
Similar guarantees are also achieved with the presence of obstacles and at an robot density of up to one-half. 
These results in 2D are then shown to readily generalize to 3D and higher dimensions. 
In practice, our methods can solve problems on 2D graphs with over $10^5$ number of vertices and $4.5 \times 10^4$ robots to $1.26$ makespan-optimal (which can be better with a larger $m_1:m_2$ ratio). Scalability is even better in 3D.
To our knowledge, no previous \gls{mpp} solvers provide dual guarantees on low-polynomial running time and practical optimality. 

\textbf{Limitation.}
While the \gls{rta} excels in achieving reasonably good optimality within polynomial time for dense instances, it does have limitations when applied to scenarios with a small number of robots or instances that are inherently easy to solve. 
In such cases, although \gls{rta} remains functional, its performance may lag behind other algorithms, and its solutions might be suboptimal compared to more specialized or efficient approaches. 
The algorithm's strength lies in its ability to efficiently handle complex, densely populated scenarios, leveraging its unique methodology. However, users should be mindful of its relative performance in simpler instances where alternative algorithms may offer superior solutions. 
This limitation underscores the importance of considering the specific characteristics of the robotic system and task at hand when choosing an optimization algorithm, tailoring the selection to match the intricacies of the given problem instance.

Our study opens the door for many follow-up research directions; we discuss a few here. 

\textbf{New line shuffle routines}. Currently, 
\gls{rtlm} and \gls{rth} only use two/three rows to perform a simulated row shuffle.
Among other restrictions, this requires that the sub-grids used for performing simulated shuffle be well-connected (i.e. obstacle-free or the obstacles are regularly spaced so that there are at least two rows that are not blocked by static obstacles in each motion primitive to simulate the shuffle).
Using more rows or irregular rows in a simulated row shuffle, it is possible to accommodate larger obstacles and/or support density higher than one-half.

\textbf{Better optimality at lower robot density}. 
It is interesting to examine whether further optimality gains can be realized at lower robot density settings, e.g., $\frac{1}{9}$ density or even 
lower, which are still highly practical. We hypothesize that this can be realized by somehow merging the different phases of \gls{rta} so that some unnecessary robot travel can be eliminated after computing an initial plan. 

\textbf{Consideration of more realistic robot models}.
The current study assumes a unit-cost model in which an robot takes a unit amount of time to travel a unit distance and allows turning at every integer time step. 
In practice, robots will need to accelerate/decelerate and also need to make turns. 
Turning can be especially problematic and cause a significant increase in plan execution time if the original plan is computed using the unit-cost model mentioned above. 
We note that \gls{rth} returns solutions where robots move in straight lines most of the time, which is advantageous compared to all existing \gls{mpp} algorithms, such as ECBS and DDM, which have many directional changes in their computed plans. 
it would be interesting to see whether the performance of \gls{rta}-based \gls{mpp} algorithms will further improve as more realistic robot models are adapted. 

\begin{remark}
    Currently, our \gls{rta}-based \gls{mpp} solves are limited to a static setting whereas e-commerce applications of multi-robot motion planning often require solving life-long settings ~\cite{Ma2017LifelongMP}. 
%
The metric for evaluating life-long \gls{mpp} is often the \emph{throughput}, namely the number of goals reached per time step.
We note that \gls{rth} also provides optimality guarantees for such settings, e.g., for the setting where $m_1 = m_2 = m$,
the direct application of \gls{rth} to large-scale life-long \gls{mpp} on square grids yields an optimality ratio of $\frac{2}{9}$ on throughput.

We may solve life-long \gls{mpp} using \gls{rth} in \emph{batches}. 
For each batch with $n$ robots, \gls{rth} takes about $3m$ steps; the throughput 
is then $\mathcal{T}_{RTH} = \frac{n}{3m}$.
As for the lower bound estimation of the throughput, the expected Manhattan 
distance in an $m\times m$ square, ignoring inter-robot collisions, is 
$\frac{2m}{3}$.
Therefore, the lower bound throughput for each batch is $\mathcal{T}_{lb}=\frac{3n}{2m}$.
The asymptotic optimality ratio is $\frac{\mathcal{T}_{RTH}}{\mathcal{T}_{lb}}=\frac{2}{9}$.
The $\frac{2}{9}$ estimate is fairly conservative because \gls{rth} supports much higher 
robot densities not supported by known life-long \gls{mpp} solvers. 
Therefore, it appears very promising to develop an optimized Grid Rearrangement-inspired 
algorithms for solving life-long \gls{mpp} problems. 
\end{remark}

%% file: chapters/chapter4.tex
\chapter{Well-Connected Set and Its Application to Multi-Robot Path Planning}~\label{chap:well-connected}

\section{Introduction}

Designing infrastructures that accommodate many mobile entities (e.g., vehicles, robots, and so on) without causing frequent congestion or deadlock is critical for improving system throughput in real-world applications, e.g., in an autonomous warehouse where many robots roam around. 
%
%
A good design generally entails good \emph{environment connectivity} in some sense. 
This paper captures the intuition of ``good connectivity'' with the concept of \gls{wcs}, presents a comprehensive study on computing \gls{lwcs}, and highlights the application of \gls{lwcs} in \gls{mpp}.
%
%

To illustrate what a \gls{wcs} is, let's consider a parking garage. It is essential to design it so vehicles can park without blocking each other, and retrieving a parked vehicle doesn't require moving other vehicles. 
Roughly speaking, the parking spots satisfying these requirements form a \gls{wcs}, and finding an \gls{lwcs} is instrumental in determining maximum parking capacity while minimizing congestion.
Well-formed infrastructures based on \gls{wcs}s are encountered in a broad array of real-world scenarios, including fulfillment warehouses, parking structures, storage systems~\cite{wurman2008coordinating,guo2023toward,azadeh2019robotized,caron2000optimal}, and so on. These infrastructures, designed properly, efficiently facilitate the movement of the enclosed entities, ensuring smooth operations and avoiding blockages.

\gls{wcs} is especially relevant to MRPP, which involves finding collision-free paths for many mobile robots ~\cite{guo2022sub, guo2022rubik,yan2013survey,sheng2006distributed,Ma2017LifelongMP,ma2019searching,vcap2015complete}.
Here, the challenge lies in finding feasible paths connecting each robot's start and target positions. 
The concept of \gls{wcs} becomes crucial, as it ensures that each robot can reach its destination without traversing other robots' positions, thereby guaranteeing a deadlock-free solution for easily realized prioritized planners.

In this paper, we provide a rigorous formulation for the \gls{lwcs} problem and establish its computational intractability. We then propose two algorithms for tackling this challenging combinatorial optimization problem. The first algorithm is an exact optimal approach that guarantees finding a largest well-connected vertex set, while the second algorithm offers a suboptimal but highly efficient solution for large-scale instances. As an application, \gls{lwcs} readily provides prioritized MRPP with completeness guarantees.

\begin{figure}[!htpb]
    \centering
  \begin{overpic}               
        [width=1\linewidth]{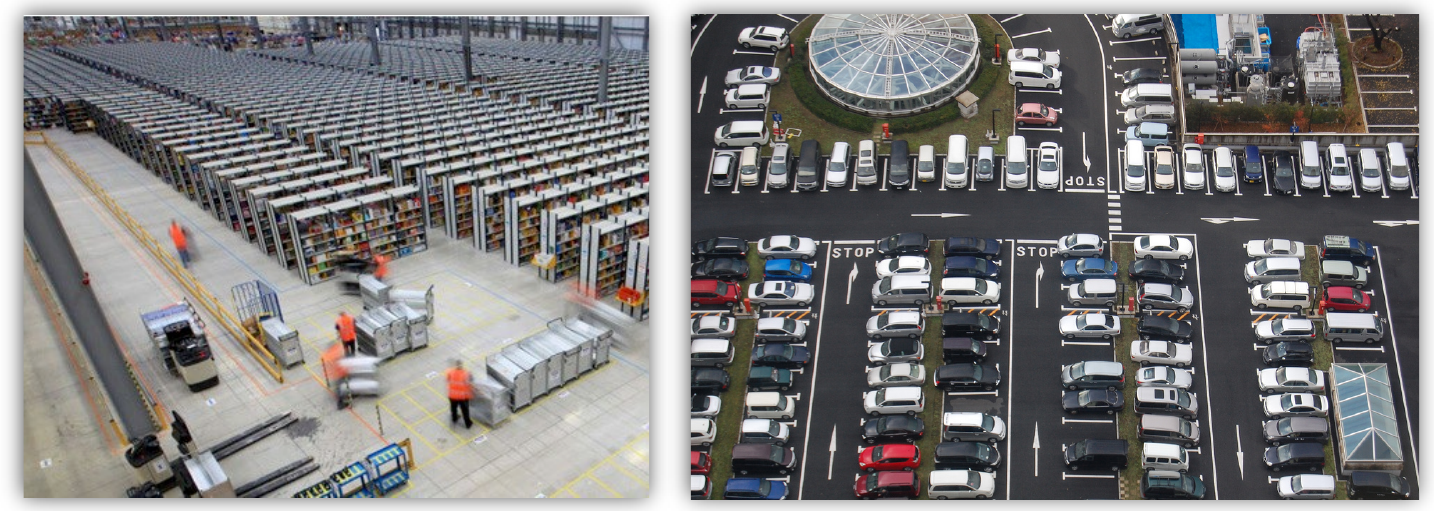}
             \small
             \put(20.5, -2) {(a)}
             \put(70.5, -2) {(b)}
        \end{overpic}
    \caption{Examples of well-formed infrastructures. (a) Amazon fulfillment warehouse. (b)  A typical parking lot.}
    \label{fig:aplications}
\end{figure}%

\textbf{Related work.} The concept of well-connected vertex sets is inspired by well-formed infrastructures \cite{vcap2015complete}. In well-formed infrastructures, such as parking lots and fulfillment warehouses \cite{wurman2008coordinating}, the endpoints are designed to allow multiple robots (vehicles) to move between them without completely obstructing each other, where the endpoint can be a parking slot, a pickup station, or a delivery station. 
Many real-world infrastructures are built in this way to benefit pathfinding and collision avoidance.
Planning collision-free paths that move robots from their start positions to target positions, known as multi-robot path planning or \gls{mpp}, is generally NP-hard to optimally solve ~\cite{surynek2010optimization,yu2013structure}.
In real applications, prioritized planning \cite{erdmann1987multiple,silver2005cooperative} is one of the most popular methods used to find collision-free paths for multiple moving robots where the robots are ordered into a sequence and planned one by one such that each robot avoids collisions with the higher-priority robots.
The method performs well in uncluttered settings but is generally incomplete and can fail due to deadlocks in dense environments.
Prior studies \cite{vcap2015complete,ma2019searching} show that prioritized planning with arbitrary ordering is guaranteed to find deadlock-free paths in well-formed environments.
When a problem is not well-formed, it may be possible to find a solution using prioritized planning with a specific priority ordering, as proposed in \cite{ma2019searching}. However, finding such a priority order can be time-consuming or even impossible. 
To our knowledge, no previous studies investigated how to efficiently design a well-formed layout that fully utilizes the workspace.
%

\begin{figure}[h]
    \centering
    
  \begin{overpic}               
        [width=1\linewidth]{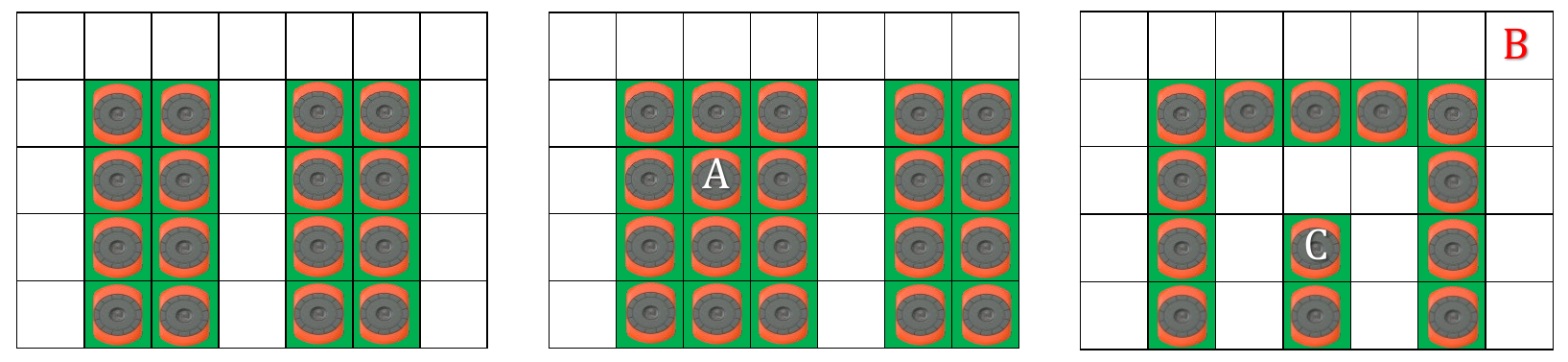}
             \small
             \put(14., -3) {(a)}
             \put(47.9,-3) {(b)}
             \put(82., -3) {(c)}
        \end{overpic}
\vspace{-2mm}
    \caption[An example illustrating the concepts of \gls{wcs}]{(a) The green cells form a \gls{wcs}. Any robot parked at one of these cells does not block others' move. (b)  An example of a non-\gls{wcs}. Retrieving one robot, $A$, requires moving some other robots. (c) An example of a S\gls{wcs}. $B$ has no access to a robot at $C$ without moving others.  }
    \label{fig:well_connected_example}
\vspace{-6mm}
\end{figure}%

\textbf{Organization.} 
%
We provide detailed problem formulations in~\ref{sec6:prelim}.
In  \ref{sec6:nphard}, we investigate the theoretical properties and establish the proof of NP-hardness. 
Next, in \ref{sec6:algorithm}, we present our algorithms for finding \gls{wcs} while optimizing the size of the set. 
In \ref{sec6:applications}, we explore the application of \gls{wcs} in multi-robot navigation.
In  \ref{sec6:eval}, we evaluate the effectiveness of our algorithms on various maps. 
Finally, we conclude the paper in  \ref{sec6:conclusion} and discuss future directions for research.
%

\section{Problem Formulation}\label{sec6:prelim}
\subsection{Well-Connected Set}
Let $G(V,E)$ be a connected undirected graph representing the environment with vertex set $V$ and edge set $E$ the edge set. 
A well-connected set (\gls{wcs}) is defined as follows.
\begin{definition}[Well-Connected Set (\gls{wcs})]\label{def:well_connected_set}
On a graph $G(V,E)$, a vertex set $M \subset V$ is well-connected if (i). $\forall u,v\in M, u\neq v$, there exists a path connecting $u,v$ without passing through any $w\in M-\{u,v\}$ , (ii) the induced subgraph of $G$ by the vertex subset $V-M$ is connected.
\end{definition}
\gls{wcs} enforces a stronger connectivity requirement. If a vertex set $M$ satisfies (i) but violates (ii), we call $M$ a \gls{swcs}. 
Any \gls{wcs} is a S\gls{wcs} set but the opposite is not true (see, e.g.,  Fig.~\ref{fig:well_connected_example}(c)).

For a given $G$, there are many \gls{wcs}s. We are particularly interested in computing a largest such set.

Toward that, We introduce two related concepts: the \gls{mwcs} and the \gls{lwcs}.

\begin{definition}[Maximal Well-Connected Set (\gls{mwcs})]
A \gls{wcs} $M$  is maximal if for any $v\in V-M$, $\{v\}\bigcup M$ is not a well-connected set.
\end{definition}

\begin{definition}[Largest Well-Connected Set (\gls{lwcs})]
 A largest well-connected set $M$  is a \gls{wcs} with maximum cardinality.    
\end{definition}

By definition, a \gls{lwcs} is also a M\gls{wcs}; the opposite is not necessarily true. 
%
%
In this paper, we focus on maximizing the ``capacity" of well-formed infrastructures, or in other words, maximizing the cardinality of the \gls{wcs}.
Besides capacity, we also introduce the \emph{path efficiency ratio} (\gls{per}) for evaluating how good a layout is from the path-length perspective. 
\begin{definition}[Well-Connected Path (WCP)]
Let $M$ be a \gls{wcs}. A path $p=(p_0,...,p_k)$ is a well-connected path connecting $p_0$ and $p_k$ if its subpath $(p_1,...,p_{k-1})$ does not pass through any vertex in $M$. 
\end{definition}
If $M$ is a \gls{wcs}, a WCP connects any two vertices $u,v\in M$. We denote $d_w(u,v)$ as the shortest WCP distance between $u,v$ and $d(u,v)$ as the shortest path distance.
\begin{definition}[Path Efficiency Ratio]
Let $M$ be a \gls{wcs} of $G$ and $u\in M$ as the reference point (i.e. I/O port), the \emph{path efficiency ratio} w.r.t vertex $u$ is defined as $\frac{\sum_{v\in M}d(u,v)}{\sum_{v\in M}d_w(u,v)}$.
\end{definition}

\section{Theoretic Study}\label{sec6:nphard}
In this section, we investigate the property of \gls{wcs} and prove that finding \gls{lwcs} is NP-hard.
%
%

A vertex in an undirected connected graph is an \emph{articulation point} (or cut vertex) if removing it (and edges through it) disconnects the graph or increases the number of connected components. 
We observe if a \gls{wcs} contains a node $v$, then $v$ should not be an articulation point so as not to violate property (ii) in the definition of \gls{wcs}.
\begin{proposition}\label{p:articulation_point}
    If $M$ is a \gls{wcs}, for any $M'\subseteq M,v\in M'$, $v$ is not an articulation point of the subgraph induced by $V-M'+\{v\}$.  
\end{proposition}
Next, we investigate the property of the node in \gls{wcs} and its neighbors. 
\begin{proposition}\label{p:orphan_neighbor}
Let $M$ be a \gls{wcs}. For any $v\in M$, if $|M|>|N(v)|+1$ where $N(v)$ denotes the set of neighbors of $v$, then at least one of its neighbors is not in $M$.
\end{proposition}
\begin{proof}
    Assume $N(v)\subset M$. Because $|M|>|N(v)|+1$, $M-(N(v)\bigcup\{v\})\neq\emptyset$. Let $w\in M-(N(v)\bigcup\{v\})$, then every path from $w$ to $v$ has to pass through a neighboring node of $v$, which contradicts property (i) of \gls{wcs}.
\end{proof}
We now prove that finding a \gls{lwcs} is NP-hard.
\begin{theorem}[Intractability]
    Finding the \gls{lwcs} is NP-hard.
\end{theorem}
\begin{proof}
    We give the proof by using a reduction from  \emph{3SAT} \cite{garey1979computers}. Let $(X,C)$ be an arbitrary instance of 3SAT with $|X|=n$ variables $x_1,...,x_n$ and $|C|=m$ clauses $c_1,...,c_m$, in which $c_j=l_j^1\vee l_j^2\vee l_j^3$.
    Without loss of generality, we may assume that the set of all literals, $l^k_j$'s, contain both unnegated and negated forms of each variable $x_i$.

    From the 3SAT instance, a \gls{lwcs} instance is constructed as follows. 
    For each variable $x_i$, we create three nodes, one node is for $x_i$, one node is for its negation $\Bar{x}_i$, and $y_i$ which is a gadget node. 
    The three nodes are connected with each other with edges and form a triangle gadget.
    For each clause variable $c_i=l_j^1\vee l_j^2\vee l_j^3$, we create a node and connect it to the three nodes that are associated with $\Bar{l}_j^1,\Bar{l}_j^2,\Bar{l}_j^3$.
    Finally, we create an auxiliary node $z$ and add an edge between $z$ and each literal node. 
    Fig.~\ref{fig:np_hardness} gives the complete graph constructed from the 3CNF formula $(x_1\vee x_2\vee x_3)\wedge(x_2\vee \Bar{x}_3\vee \Bar{x}_4)
    \wedge (\Bar{x}_1\vee x_3\vee x_4)\wedge(x_1\vee\Bar{x}_2\vee\Bar{x}_4)$.

\begin{figure}[h]
    \centering
  \begin{overpic}               
        [width=0.75\linewidth]{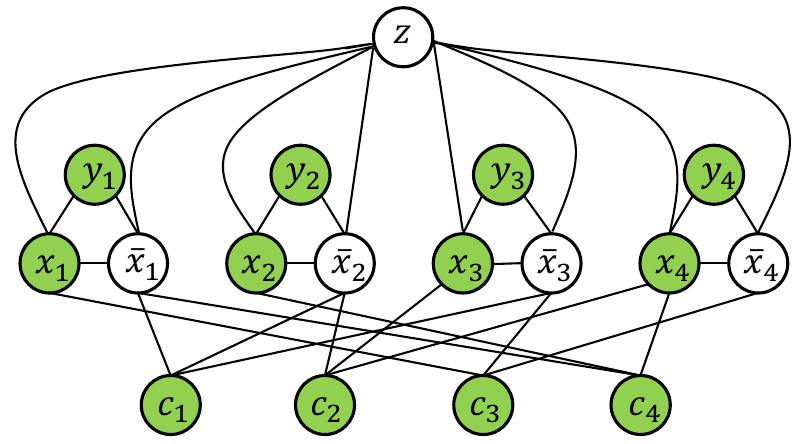}
             \small
        \end{overpic}
    \caption{The graph derived from the 3SAT instance $(x_1\vee x_2\vee x_3)\wedge(x_2\vee \Bar{x}_3\vee \Bar{x}_4)
    \wedge (\Bar{x}_1\vee x_3\vee x_4)\wedge(x_1\vee\Bar{x}_2\vee\Bar{x}_4)$. The green vertices form a \gls{lwcs} of size 12. By setting the literals of green vertices to true, the 3SAT instance is satisfied.}
    \label{fig:np_hardness}
\end{figure}%

To find a \gls{lwcs}, we observe that for each triangle gadget formed by $(x_i,\Bar{x}_i, y_i)$,  the node $x_i, \Bar{x}_i$ should not be selected at the same time.
    Otherwise, if node $y_i$ is not selected, it would be completely isolated, which violates Proposition \ref{p:articulation_point}. 
    If $y_i$ is also selected, then every path to $y_i$ from other nodes must always pass through either node $x_i$ or node $\Bar{x}_i$, violating \gls{wcs} definition.
    This implies that for each triangle, at most two nodes can be selected and added to the set, and one of the nodes must be $y_i$.
    A second observation is if the clause node $c_j=l_j^1\vee l_j^2\vee l_j^3$ is selected as a vertex of the \gls{wcs}, the nodes $\{\Bar{l}_j^1,\Bar{l}_j^2,\Bar{l}_j^3\}$ connected to $c_j$ cannot be selected simultaneously.
    Otherwise, node $c_i$ would be blocked by the three nodes. 
    
    If the 3SAT instance is satisfiable, let $\widetilde{X}=\{\widetilde{x}_1,...,\widetilde{x}_n\}$ be an assignment of the truth values to the variables. 
    Based on these observations, the \gls{mwcs} of size $2n+m$ is constructed as $\widetilde{X}\bigcup C\bigcup Y$, where $Y=\{y_1,...,y_n\}$.
    On the other hand, if the reduced graph has a \gls{wcs} of size $2n+m$, for each triangle gadget, we can only choose two, and one of them must be the gadget node.
    Thus, all the gadget nodes should be selected and this contributes $2n$ nodes.
    For the remaining $m$ nodes, we hope that all nodes in $C$ are selected.
    To not violate the well-connectedness, for each clause node $c_j=l_j^1\vee l_j^2\vee l_j^3$ to be selected, at least one of the node from  $(\Bar{l}_j^1,\Bar{l}_j^2,\Bar{l}_j^3)$ should not be selected. This is equivalent to ensuring that $c_j=l_j^1\vee l_j^2\vee l_j^3$ is true. Therefore, if the node $\widetilde{x}_i$ is selected for each $i$, we set $\widetilde{x}_i=true$, then the resulting assignment $\widetilde{X}$ satisfies the 3SAT instance.
\end{proof}

\begin{remark}
Although finding a \gls{lwcs} is generally NP-hard, for certain special classes of graphs, a \gls{lwcs} can be obtained easily.
For example, in a complete graph, the \gls{lwcs} is the set of all nodes in the graph. 
In a tree graph, the \gls{lwcs} is the set that contains all the nodes of degree one. 
In a complete bipartite graph $B(U,V)$, the \gls{lwcs} is formed by choosing any $|U|-1$ vertices from $U$ and any $|V|-1$ vertices from $V$.
Finding the maximum well-semi-connected set is also NP-hard, which can be proven by reduction from $3SAT$ in a similar manner to the proof for the \gls{lwcs}.
\end{remark}

Next, we investigate the upper bound for the \gls{lwcs} denoted as $M^*$.
\begin{proposition}\label{p:upperbound}
    Denote $\Delta$ as the maximum degree of the node of $G$. We have $|M^*|\leq \max(\frac{\Delta-1}{\Delta}|V|,\Delta+1)$.
\end{proposition}
\begin{proof}
For each $v\in M^*$, at least one of its neighbors $u$ should be in $V-M^*$.
Charge $v$ to $u$.
A node $u\in V-M^*$ can be charged at most $\Delta-1$ times.
Hence, we have $|V-M^*|\geq \frac{|M^*|}{\Delta-1}$.
Since every node is either in $M^*$ or $V -M^*$ we have $|M^*| + |V - M^*| = |V|\geq (1+\frac{1}{\Delta-1})|M^*|$.
On the other hand, it is possible that $M^*=N(v)\bigcup\{v\}$ for a node $v$ of degree $\Delta$.
Therefore, $|M^*|\leq \max(\frac{\Delta-1}{\Delta}|V|,\Delta+1)$.
\end{proof}

\section{Algorithms for finding well-connected sets}\label{sec6:algorithm}

\subsection{Algorithm for Finding a \gls{mwcs}}
We first present an algorithm to find a \gls{mwcs} in a graph, which begins by initializing two sets: $M$ and $P$. 
Set $M$ contains vertices currently included in the \gls{wcs}, while $P$ contains those available for adding to the set. 
Initially, $M$ is empty, and $P$ contains all the vertices of $G$.
The algorithm selects a vertex from $P$ and adds it to $M$ while maintaining it as a \gls{wcs}. 
At each step, we use Tarjan's algorithm~\cite{tarjan1972depth} to compute the set of articulation vertices for the subgraph induced by $V-M$ and remove all the articulation vertices from $P$. 
Assume $v\in M$ and $u\in N(v)$,  $u$ is called an orphan neighbor of $v$ if $u\not\in M$ and $N(v)-u\subset M$.
We iterate over all the neighboring vertices of $M$ to remove all orphan neighbors from $P$, according to Proposition~\ref{p:orphan_neighbor}. 
If $|P|=0$, the algorithm cannot add another vertex to $M$ while keeping it well-connected. 
To ensure that $M$ is indeed maximal, we check if a $v\in M$ exists, such as $M\subset N(v)\bigcup\{v\}$.
If so, we add the remaining neighbor of $v$ to $M$ so that $M$ is maximal.
Finding the set of articulation vertices using Tarjan's algorithm takes $O(|V|+|E|)$.
As a result, the algorithm for finding a \gls{mwcs} takes $O(|M|(|V|+|E|))=O(|V|(|V|+|E|))$.

\begin{algorithm}
\DontPrintSemicolon
\SetKwProg{Fn}{Function}{:}{}
\SetKw{Continue}{continue}
  \Fn{\textsc{MaximalWVS}({G})}{
 \caption{Maximal Well-Connected Set \label{alg:greedy}}
    $M\leftarrow \{\}, P\leftarrow V$\;
    \While{$|P|!=0$}{
        $AP\leftarrow \texttt{Tarjan}(G'(V-M))$\;
        $NB\leftarrow \texttt{FindOrphanNeighbors}(M)$\;
        $P\leftarrow P-(AP\bigcup NB)$\;
        $u\leftarrow \texttt{ChooseOne}(P)$\;
        $M.add(u)$ and $P.pop(u)$\;
    }
    $M\leftarrow \texttt{AdditionalCheck}(M)$\;
    \Return $M$\; 
}
\end{algorithm}

%
Next, we establish lower bounds for the \gls{mwcs} algorithm.
\begin{proposition}\label{p:lowerbound_one_degree}
Denote $W$ as the set of terminal nodes that contains nodes of degree one. Then $W\subseteq M$.    
\end{proposition}
\begin{proof}
    Every terminal node $u$  is neither an orphan neighbor nor an articulation point of the induced subgraph $G'(V-M+\{u\})$. Thus $W\subseteq M$. 
\end{proof}
\begin{proposition}\label{p:lowerbound}
Denote $L$ as the length of the longest induced path of $G$. Then $|M|\geq \frac{|V|}{L}$.    
\end{proposition}
\begin{proof}
    Let $u\in V-M$. 
    Then $u$ is either an orphan neighbor or an articulation point of $G'(V-M)$.
    For every $v\in M$, there exists a WCP $p_v$ connecting $u$ and $v$.
    Every vertex in $V-M$ should be passed through by at least one of those WCPs. 
    Otherwise, let $w$ be the vertex that is not passed through by any of the WCPs.
    Clearly, $w$ cannot be an articulation point or an orphan neighbor.
    Then $w$ can be added to $M$, which contradicts that $M$ is maximal.
    Let $p_v'=p_v-v$.
    Then $\bigcup_{v\in M}p'_v=V-M$. Meanwhile, $|p_v|\leq diam(G(V-M))\leq L$, where $diam(G(V-M))$ is the diameter of the induced subgraph by $V-M$.  We have $|V-M|\leq\sum_{v\in M}|p_v'|\leq |M|(L-1)$.
    On the other hand $|V-M|+|M|=|V|$. Hence $|M|\geq |V|/L$.
\end{proof}
\subsection{Exact Search-Based Algorithm}
We now establish a complete search algorithm to find a \gls{lwcs}.
%

A naive algorithm for finding the \gls{lwcs} starts by initializing a variable $k$ to 1. The algorithm then searches for a  \gls{wcs} of size $k$ in the graph. 
If such a set exists, it is considered a candidate for the \gls{lwcs}. 
The algorithm then proceeds to search for a \gls{wcs} of size $k+1$. 
If such a set is found, it is taken as the new candidate for the \gls{lwcs}. 
If no \gls{wcs} of size $k$ is found, the algorithm terminates and returns the \gls{wcs} of size $k-1$ as the \gls{lwcs}. 
This algorithm uses a simple iterative approach and keeps track of the size of the \gls{lwcs} found so far. 
By searching for larger \gls{wcs} and updating the candidate accordingly, it ensures that it finds the \gls{lwcs}.
When searching a well-connected vertex set of size $k$, there are $\binom{|V|}{k}$ possible subsets. 
For each subset $M$ with $|M|=k$, we could run $k$ times of BFS to check if the subgraph induced by $V-M+\{v\}$ is connected for each $v\in M$, in order to examine if it is a \gls{wcs}.
The naive algorithm runs in $O(2^{|V|}|V|(|V|+|E|))$.
To improve the efficiency of the algorithm, we propose a search-based algorithm. 

The algorithm is shown in \ref{alg:dfs}.
The algorithm employs a depth-first search (DFS) approach to exhaustively explore all possible vertex sets and select the one with the maximum size and the highest path efficiency ratio (PER).
The algorithm starts by defining three empty sets, $M$, $M^*$, and $visited$, 
where $M$ represents the current set of vertices being explored, $M^*$ represents the set with the maximum size and highest \gls{per} found so far, and $visited$ stores all previously explored vertex sets to avoid repeated explorations.
Then, the DFS search is initiated by calling the DFS function with the current set $M$, the set with the maximum size and highest PER found so far $M^*$, and the set of visited vertex sets visited as parameters.
The DFS function starts by checking if the current set $M$ has already been explored before. 
If it has, the function returns immediately. 
Otherwise, $M$ is added to the visited set. 
The function then checks if the current set $M$ is larger than the set with the maximum size and highest PER found so far $M^*$. If it is, then we update $M$ as the current best $M^*$. 
If $M$ has the same size as $M^*$, then the function compares their \gls{per} values, and the set with the higher \gls{per} becomes $M^*$.
%

Otherwise, the function generates a list of candidate vertices $P$ that can be added to the current set $M$ without violating the \gls{wcs} condition. 
The function then loops over the candidate vertices and recursively calls the DFS function with the current set $M$ unioned with the current candidate vertex and the same $M^*$ and visited sets.

The algorithm continues until all possible vertex sets have been explored, or the condition for returning early is met. 
Similar to the algorithm for the \gls{mwcs}, we also perform additional checks. 
For each $v\in M$, we check if $N(v)\bigcup\{v\}$ can be larger than the current solution found to ensure that the final vertex set is the maximum one.
The algorithm's time complexity is dependent on the size and density of the graph $G$, as well as the number of candidate vertices generated at each recursive call. 
In the worst-case scenario, the algorithm has a time complexity of $O(2^{|V|}(|V|+|E|))$, where $V$ is the number of vertices in the graph $G$. 
However, the early termination condition in the DFS function helps avoid exploring unnecessary vertex sets and can significantly reduce the algorithm's execution time.
\begin{algorithm}
\DontPrintSemicolon
\SetKwProg{Fn}{Function}{:}{}
\SetKw{Continue}{continue}
  \Fn{$\textsc{DfsSearch}(G)$}{
 \vspace{0.5mm}
 \caption{Largest Well-Connected Set\label{alg:dfs}}
 $M,M^{*},visited\leftarrow \emptyset,\emptyset,\emptyset$\;
 \vspace{0.5mm}
 $\texttt{DFS}(G,M,M^{*},visited)$\;
 \vspace{0.5mm}
    $M^{*}\leftarrow \texttt{AdditionalCheck}(M^*)$\;
 \vspace{0.5mm}
 \Return $M^{*}$ \;
}
\Fn{\textsc{DFS}($G,M,M^{*},visited$)}{
 \vspace{0.5mm}
\lIf{$M\in visited$}{\Return}
$visited.add(M)$\;
 \vspace{0.5mm}
\lIf{$|M|>|M^{*}|$ or ($|M|=|M^{*}|$ and $\texttt{PER}(M)<\texttt{PER}(M^{*})$)}{
    $M^{*}\leftarrow M$
}
 \vspace{0.5mm}
$AP\leftarrow \texttt{Tarjan}(G'(V-M))$\;
 \vspace{0.5mm}
$NB\leftarrow \texttt{FindOrphanNeighbor}(M)$\;
 \vspace{0.5mm}
$P\leftarrow P-(AP\bigcup NB)$\;
 \vspace{0.5mm}
\lIf{$|P|+|M|<|M^{*}|$}{
\Return
}
 \vspace{0.5mm}
\lForEach{$v$ in $P$}{
 \vspace{0.5mm}
    $\texttt{DFS}(G,M\bigcup\{v\},M^{*},visited)$
    }
}
\end{algorithm}
\vspace{-3mm}

\section{Applications in multi-robot navigation}\label{sec6:applications}
We demonstrate how \gls{wcs} benefits prioritized multi-robot path planning (\gls{mpp}) on 
graphs. In a legal move, a robot may cross an edge if the edge is not used by another robot during the same move and the target vertex is not occupied by another robot at the end of the move.
%
%
%
The task is to plan paths with legal moves for all robots to reach their respective goals.
%
The makespan (the time for all robots to reach their goals) and the total arrival time are two common criteria to evaluate the solution quality.
Previous studies \cite{vcap2015complete,ma2019searching} have established the completeness of prioritized planning in well-formed infrastructures. 
Building on the foundation, we provide algorithms with completeness guarantees for non-well-formed environments.
\begin{definition}[Well-Formed \gls{mpp}]\label{def:wf_mpp}
An \gls{mpp} instance is well-formed if, for any robot $i$, a path connects its start and goal without traversing any other robots' start or goal vertex.    
\end{definition}

\begin{theorem}\label{theorem:wf_guarantee}
Well-formed \gls{mpp} is solvable using prioritized planning with any total priority ordering \cite{vcap2015complete,ma2019searching}.
\end{theorem}

When addressing non-well-formed \gls{mpp}, we adopt a simple, effective strategy similar to \cite{guo2022sub} to convert start and goal configurations to intermediate well-connected configurations so that the resulting problems are guaranteed to be solvable by prioritized planning with any total priority ordering.
We call the algorithm \gls{unpp} - \textbf{un}labeled \textbf{p}rioritized \textbf{p}lanning. 

We compute a \gls{mwcs}/\gls{lwcs} $M$ offline.   
The first step in the algorithm is to assign the $2n$ vertices, $\mathcal{S}'',\mathcal{G}''$ in $M$ as the intermediate start vertices and goal vertices.
This is done by solving a min-cost matching problem using the Hungarian algorithm \cite{Kuhn1955}. Collision-free paths are easy to plan in the unlabeled setting with optimality guarantees on makespan and total distance \cite{yu2013multi,yu2012distance}.
%
%
%
The output of this function is a set of collision-free paths for the robots that route them to a well-formed configuration and the intermediate labeled starting and goal positions $\mathcal{S}', \mathcal{G}'$.
The PrioritizedPlanning function is then called on the resulting intermediate starting and goal positions to generate a deadlock-free path for each robot.
Finally, the paths generated by the Unlabeled Multi-Robot Path Planning and Prioritized Planning functions are concatenated to produce a final solution.
%

\begin{algorithm}
\DontPrintSemicolon
\SetKwProg{Fn}{Function}{:}{}
\SetKw{Continue}{continue}
  \KwIn{Starts $\mathcal{S}$, goals $\mathcal{G}$, \gls{lwcs}/\gls{mwcs} $M$}
  \Fn{\textsc{UNPP}({$\mathcal{S},\mathcal{G}$})}{
 \caption{ \gls{unpp}\label{alg:unpp}}
$\mathcal{S}'',\mathcal{G}''\leftarrow \texttt{Assignment}(\mathcal{S},\mathcal{G},M)$\;
$P_s,\mathcal{S}'\leftarrow \texttt{UnlabeledMRPP}(\mathcal{S},\mathcal{S}'')$\;
$P_g,\mathcal{G}'\leftarrow \texttt{UnlabeledMRPP}(\mathcal{G},\mathcal{G}'')$\;
$P_m\leftarrow \texttt{PrioritizedPlanning}(\mathcal{S}',\mathcal{G}')$\;
$solution\leftarrow \texttt{Concat}(P_s,P_g,P_m)$\;
\Return $solution$\;
}
\end{algorithm}

\begin{theorem}
Denote $n_c$ as the size of the \gls{lwcs} of graph $G$, for any \gls{mpp} instance with number of robots less than $n_c/2$, regardless of the distribution of starts and goals, \gls{unpp} is complete with respect to any priority ordering. 
\end{theorem}
\begin{proof}
When the number of robots $n\leq n_c/2$, it is always possible to select such $\mathcal{S}',\mathcal{G}'$ where $\mathcal{S}'\bigcap\mathcal{G}'=\emptyset$.
Since $\mathcal{S}'\bigcup \mathcal{G}'\subseteq M$, by the definition of \gls{wcs}, $\mathcal{S}'\bigcup \mathcal{G}'$ is also a \gls{wcs}.
Therefore, the resulting \gls{mpp} problem which requires routing robots from $\mathcal{S}'$ to $\mathcal{G}'$ is well-formed.
By Theorem.~\ref{theorem:wf_guarantee}, prioritized planning is guaranteed to solve the subproblem using any priority ordering.
\end{proof}
\gls{unpp}  runs in polynomial time for \gls{mpp} instances with $n\leq n_c/2$.
We can use the max-flow-based algorithm \cite{yu2013multi} to solve the unlabeled \gls{mpp}, which takes $O(n|E|D(G))$ where $D(G)$ is the diameter of the graph $G$, if we use \cite{ford1956maximal} (faster max-flow algorithm can also be used here) to solve the max-flow problem.
In the worst case, an unlabeled \gls{mpp} requires $n+|V|-1$ makespan to solve \cite{yu2012distance}.
For the prioritized planning applied on well-formed instances, the makespan is upper bounded by $nD(G)$. 
As we use spatiotemporal A* to plan the individual paths while avoiding collisions with higher-priority robots on a time-expanded graph with edges no more than $n|E|D(G)$ and such a solution is guaranteed to exist,
the worst time complexity is $O(n^2|E|D(G))$.
In summary, \gls{unpp} yields worst case time complexity of $O(n^2|E|D(G))$ and its makespan is upper bounded by $2(n+|V|-1)+nD(G)$.

For $\frac{n_c}{2}<n<n_c$, arbitrary priority ordering does not guarantee a solution.
For some robots, its intermediate start vertex in $\mathcal{S}'$ has to be the intermediate goal vertex in $\mathcal{G}'$ of another robot when assigned from the \gls{wcs} $M$.
The resulting subproblem is not well-formed.
However, it is possible to solve such an instance by breaking it into several sub-problems and using specific priority ordering to solve it.
To do this, we can first establish a dependency graph to determine the priority ordering of robots. 
The dependency graph consists $n$ nodes representing the robots. If $s_i'=g_j'$, we add a directed edge from node $i$ to node $j$, meaning that $j$ should have higher priority than $i$.  If the resulting dependency graph is a DAG, topological sort can be performed on it to get the priority ordering.
When encountering a cycle, as $n<n_c$, we can break the cycle by moving one of the robots in this cycle to a buffer vertex in $C-\{\mathcal{S}'\bigcup\mathcal{G}'\}$ and perform the topological sort on the remaining robots. 

Finally, we briefly illustrate the application of \gls{wcs} in multi-robot pickup and delivery (\gls{mapd}) \cite{Ma2017LifelongMP}.
In well-formed \gls{mapd},  each robot can rest in a non-task endpoint forever without preventing other robots from going to their task endpoints, i.e. pickup stations.
The layout of the endpoints forms a \gls{wcs} of the graph of the environment.
While it is desirable to increase the number of robots and the endpoints as many as possible to maximize space utilization, it is also important to keep a well-connected layout so that the robots will not block each other for better pathfinding. 
Thus, the maximum number of endpoints is equal to $n_c$, and at most $n_c-1$ robots can be used in the \gls{mapd}.

\section{Experiments}\label{sec6:eval}
In our evaluation, we first test algorithms that compute \gls{wcs} for grids and a set of benchmark maps and then perform evaluations on \gls{mpp} problems. 
Since the grids and maps used in our experiments are either 4-connected or 8-connected, the solution we find without additional checks will always be larger than the number of neighbors of a vertex $v$ plus one (i.e., $|N(v)|+1$). 
Therefore, we can safely omit the additional checks in our solution.
All experiments are performed on an Intel® CoreTM i7-6900K CPU
at 3.2GHz with 32GB RAM in Ubuntu 18.4 LTS and implemented in C++. 
\subsection{Grid Experiments}
We test the algorithms on $m\times m$ 4/8-connected grids with varying side lengths $m$. 
The result is shown in~\ref{tab:gridexp}.
In ``Random" and ``Greedy", we run the \gls{mwcs} algorithm 50 times and return the set with the maximum size.
Random randomly chooses a node from the candidates $P$ and adds it to the set.
In Greedy, to select the next candidate to add to the current \gls{wcs}, we sort the candidate nodes in ascending order based on their total shortest distance from any node in the current \gls{wcs}. We then select the candidate with the smallest total shortest distance as the next node to add to the \gls{wcs}.
To evaluate the running time of Random and Greedy, we take the average of the total execution time over the 50 runs of the algorithm.
To evaluate the path efficiency of the maximum(maximal) \gls{wcs} found by each algorithm, we treat each node in  $V$ as the reference point and compute the \gls{per} for each node. We then take the average of the \gls{per} values to obtain a measure of the overall path efficiency of the algorithms.
In DFS, we set a time limit of 600 seconds to search for a solution. If the time limit is reached before a solution is found, we report the best solution found so far. 
\begin{table*}
\vspace{2mm}
\centering
\setlength{\tabcolsep}{3pt} 
\fontsize{9}{14}\selectfont
\begin{tabular}{|c|ccc|ccc|ccc!{\vrule width 2pt}ccc|ccc|ccc|}
    \hline
    \multirow{2}{*}{Side} & \multicolumn{9}{c!{\vrule width 2pt}}{4-Connected Grid} & \multicolumn{9}{c|}{8-Connected Grid} \\
    \cline{2-19}
     & \multicolumn{3}{c|}{Rand} & \multicolumn{3}{c|}{Greedy} & \multicolumn{3}{c!{\vrule width 2pt}}{DFS} & \multicolumn{3}{c|}{Rand} & \multicolumn{3}{c|}{Greedy} & \multicolumn{3}{c|}{DFS} \\
    \cline{2-19}
     & $|M|$ & PER & T & $|M|$ & PER & T & $|M|$ & PER & T & $|M|$ & PER & T & $|M|$ & PER & T & $|M|$ & PER & T \\
    \hline
    5  & 14 & 0.89 & 0  & 11  & 0.89 & 0  & \color{red}{14} & 0.89 & 18  & 20  & 0.97 & 0  & 20  & 0.93 & 0  & \color{red}{20} & 0.97 & 336 \\
    10 & 55 & 0.60 & 0  & 52  & 0.69 & 0  & \textbf{60} & 0.67 & 600 & 72  & 0.52 & 0  & 73  & 0.96 & 0  & \textbf{74} & 0.85 & 600 \\
    15 & 124 & 0.62 & 0.02 & 122 & 0.73 & 0.04 & \textbf{138} & 0.71 & 600 & 161 & 0.60 & 0.04 & 160 & 0.85 & 0.06 & \textbf{162} & 0.85 & 600 \\
    20 & 220 & 0.73 & 0.07 & 238 & 0.74 & 0.2 & \textbf{242} & 0.67 & 600 & 283 & 0.46 & 0.1 & 280 & 0.82 & 0.3 & \textbf{285} & 0.85 & 600 \\
    25 & 238 & 0.53 & 0.2 & 372 & 0.64 & 0.7 & \textbf{378} & 0.72 & 600 & \textbf{447} & 0.46 & 0.3 & 434 & 0.86 & 0.9 & 445 & 0.64 & 600 \\
    30 & 487 & 0.53 & 0.4 & 561 & 0.62 & 2.0 & \textbf{561} & 0.68 & 600 & \textbf{645} & 0.47 & 0.7 & 622 & 0.84 & 2.3 & 642 & 0.54 & 600 \\
    35 & 667 & 0.35 & 0.8 & 765 & 0.58 & 4.7 & \textbf{765} & 0.71 & 600 & \textbf{875} & 0.62 & 1.3 & 842 & 0.79 & 5.3 & 872 & 0.62 & 600 \\
    40 & 872 & 0.47 & 1.4 & 992 & 0.72 & 9.8 & \textbf{992} & 0.70 & 600 & 1137 & 0.56 & 2.4 & 1097 & 0.79 & 11.0 & \textbf{1139} & 0.49 & 600 \\
    45 & 1098 & 0.55 & 2.4 & 1245 & 0.57 & 18.9 & \textbf{1245} & 0.69 & 600 & 1441 & 0.60 & 4.1 & 1384 & 0.82 & 22.2 & \textbf{1443} & 0.36 & 600 \\
    50 & 1360 & 0.33 & 3.9 & 1588 & 0.64 & 35.9 & \textbf{1588} & 0.66 & 600 & 1778 & 0.51 & 6.5 & 1705 & 0.82 & 39.8 & \textbf{1785} & 0.50 & 600 \\
    \hline
\end{tabular}
\caption{Grid Experiment on 4/8-connected square grids. Red values are optimal DFS solutions.}
\label{tab:gridexp}
\end{table*}

Though DFS only finds guaranteed optimal solutions on small grids, the final vertex size it returned is usually larger than the other two methods and has better \gls{per}.
Greedy finds larger \gls{wcs} on 4-connected grids than Random.
Interestingly, on 8-connected grids, the \gls{mwcs} found by Greedy is smaller.
\gls{per} of Greedy is generally better than Random.

Through linear regression, we found that on grids, the size of the maximum(maximal) \gls{wcs} $|M|$ founded by these algorithms is linearly related to the number of vertices $|V|$. 
Specifically, on 4-connected grids, $|M| \sim 0.63|V|$, and on 8-connected grids, $|M| \sim 0.72|V|$. 
This means that in a square parking lot (or other well-formed infrastructures), if it is considered a 4-connected grid, at most about $63\%$ of the space can be used for parking.

\subsection{Benchmark Maps}
We select several maps from \href{https://movingai.com/benchmarks/grids.html}{2D Pathfinding Benchmarks}~\cite{sturtevant2012benchmarks}.
Here, we use Greedy to compute the suboptimal \gls{lwcs} as most of the maps are too large to perform DFS search.
For maps that are not connected, the largest connected component is used.  
The result is presented in~\ref{tab:map_exp}.
And some examples are shown in Fig.~\ref{fig:benchmark_maps}.
Our algorithm is efficient on large and complex maps with tens of thousands of vertices.
The computed vertex set size is roughly $50\%$-$60\%$ of $|V|$ for 4-connected graphs, and $60\%$-$70\%$ for 8-connected graphs.

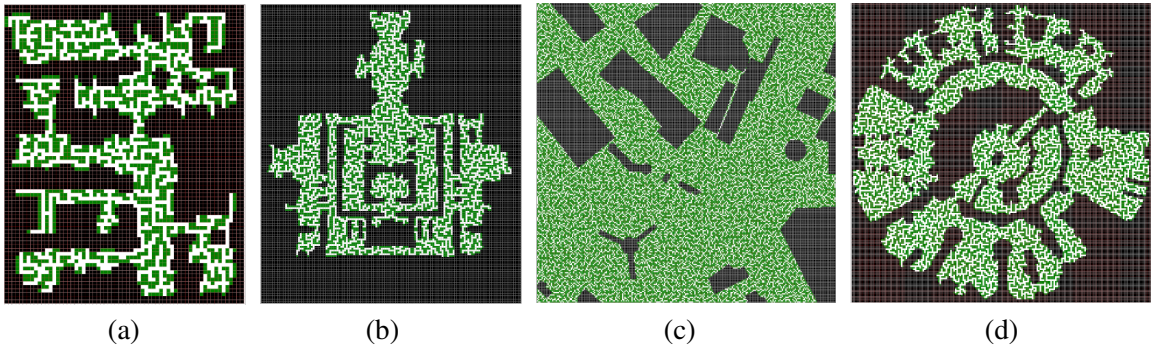
\begin{figure}[!hpbt]
\vspace{-2mm}
    \centering
  \begin{overpic}               
        [width=1\linewidth]{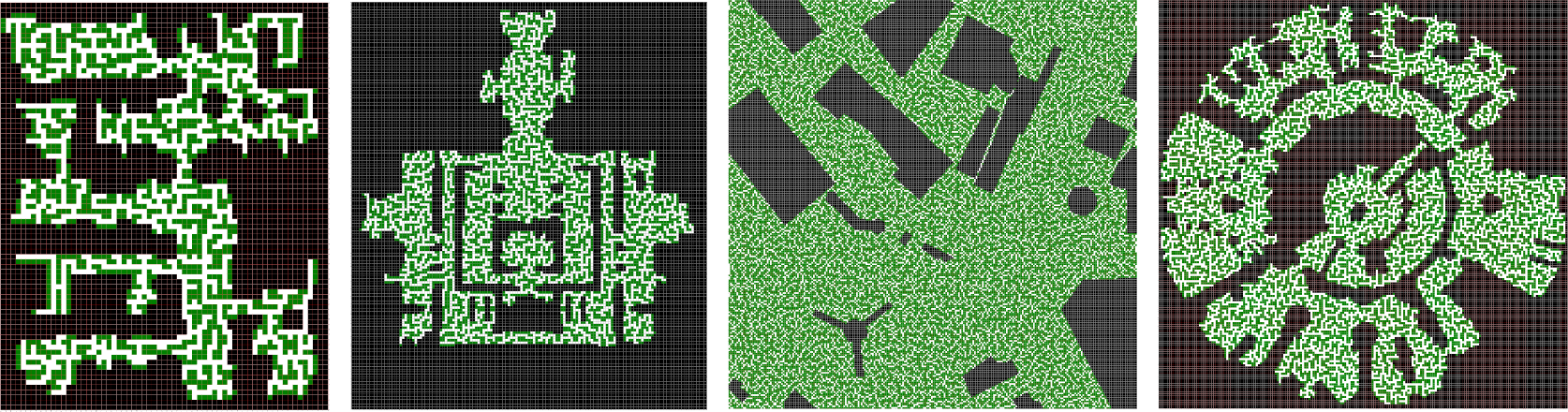}
             \small
             \put(9.0, -3) {(a)}
             \put(31.5, -3) {(b)}
             \put(57.5, -3) {(c)}
             \put(85.5, -3) {(d)}
        \end{overpic}
\vspace{-1mm}
    \caption{Examples of the computed \gls{mwcs} (colored in green) in different 4-connected grid maps. (a) den312d. (b) ht$\_$chantry. (c) Shanghai$\_$0$\_$256.  (d) lak503d. Zoom in on the digital version to see more details.}
    \label{fig:benchmark_maps}
\end{figure}%

\begin{table*}
    \centering
    \fontsize{9}{12}\selectfont 
    \setlength{\tabcolsep}{7pt} 
    \renewcommand{\arraystretch}{0.9} 
    \begin{tabular}{|l@{}|c@{}|c@{}!{\vrule width 1.5pt}c@{}|c@{}|c@{}|c@{}!{\vrule width 1.5pt}c@{}|c@{}|c@{}|c@{}|}
        \hline
        \textbf{Map Name} & \textbf{Grid Size} & $\mathbf{|V|}$ & $\mathbf{|E_4|}$ & $\mathbf{Time_4}$ & ${\mathbf{|M_4|}}$ & $\mathbf{PER_4}$ & $\mathbf{|E_8|}$ & $\mathbf{Time_8}$ & ${\mathbf{|M_8|}}$ & $\mathbf{PER_8}$ \\ \hline
    
        arena & $49\times 49$ & 2,054 & 3,955 & 2.33 & 1,113 & 0.68 & 7,813 & 3.76 & 1,455 & 0.52 \\ \hline
        brc202d & $481\times 530$ & 43,151 & 81,512 & 1,685 & 22,659 & 0.61 & 160,277 & 2,668 & 29,973 & 0.63 \\ \hline
        den001d & $80\times 211$ & 8,895 & 16,980 & 20.1 & 4,859 & 0.77 & 33,392 & 92.55 & 6,233 & 0.51 \\ \hline
        den020d & $118\times 89$ & 3,102 & 7,412 & 4.48 & 1,599 & 0.89 & 10,869 & 9.04 & 2,104 & 0.699 \\ \hline
        den312d & $81\times 65$ & 2,445 & 4,391 & 3.18 & 1,247 & 0.701 & 8,464 & 5.11 & 1,663 & 0.708 \\ \hline
        hrt002d & $50\times 49$ & 754 & 1,300 & 0.24 & 377 & 0.87 & 2,489 & 0.38 & 510 & 0.68 \\ \hline
        ht\_chantry & $141\times 162$ & 7,461 & 13,963 & 38.87 & 3,889 & 0.45 & 27,222 & 60.95 & 5,183 & 0.37 \\ \hline
        lak103d & $49\times 49$ & 859 & 1,509 & 0.32 & 438 & 0.84 & 2,869 & 0.51 & 584 & 0.58 \\ \hline
        lak503d & $194\times 194$ & 17,953 & 33,781 & 258.89 & 9,484 & 0.58 & 66,734 & 415.60 & 12,482 & 0.48 \\ \hline
        lt\_warehouse & $130\times 194$ & 5,534 & 10,397 & 18.67 & 2,895 & 0.87 & 20,306 & 31.72 & 3,858 & 0.63 \\ \hline
        NewYork\_0\_256 & $256\times 256$ & 48,299 & 94,068 & 2,000 & 26,025 & 0.40 & 186,935 & 3,469 & 34,054 & 0.35 \\ \hline
        orz201d & $45\times 47$ & 745 & 1,342 & 0.23 & 389 & 0.73 & 2,604 & 0.39 & 513 & 0.61 \\ \hline
        ost003d & $194\times 194$ & 13,214 & 24,999 & 131.82 & 7,004 & 0.88 & 49,437 & 206.30 & 9,221 & 0.59 \\ \hline
        random-32-32-20 & $32\times 32$ & 819 & 1,270 & 0.26 & 375 & 0.66 & 2,487 & 0.44 & 533 & 0.55 \\ \hline
        Shanghai\_0\_256 & $256\times 256$ & 48,708 & 95,649 & 2,119 & 26,453 & 0.32 & 190,581 & 3,542 & 34,501 & 0.29 \\ \hline
    \end{tabular}
    \caption{Computed suboptimal \gls{lwcs} size in selected 4/8-connected maps.}
    \label{tab:map_exp}
\end{table*}

\subsection{Evaluations of \gls{mpp}}
Lastly, we examine the effectiveness of our proposed \gls{mpp} method on selected benchmarks. We compare our proposed method with two other prioritized planners, HCA*\cite{silver2005cooperative} and PIBT\cite{Okumura2019PriorityIW}.
For each map, a maximal vertex set is precomputed using the Greedy method. 
To conduct the experiments, we randomly generate 50 instances for each map and the number of robots $n$.
The results are shown in Fig.~\ref{fig:orz201d}-\ref{fig:hrt002d}.
The experimental results demonstrate that our proposed method significantly improves the success rate compared to HCA* and PIBT. Furthermore, although the solution quality is not optimal, it is still reasonably good.

It should be emphasized that the proposed method applies to robots that consider kinodynamics. Specifically, in the case of car-like robots, a simple modification can be made by replacing the traditional A* algorithm with the hybrid A* algorithm for single-robot path planning. This adjustment enables the method to cater to car-like motion's unique dynamics and constraints.

\begin{figure}[!htpb]
    \centering
    \includegraphics[width=1\linewidth]{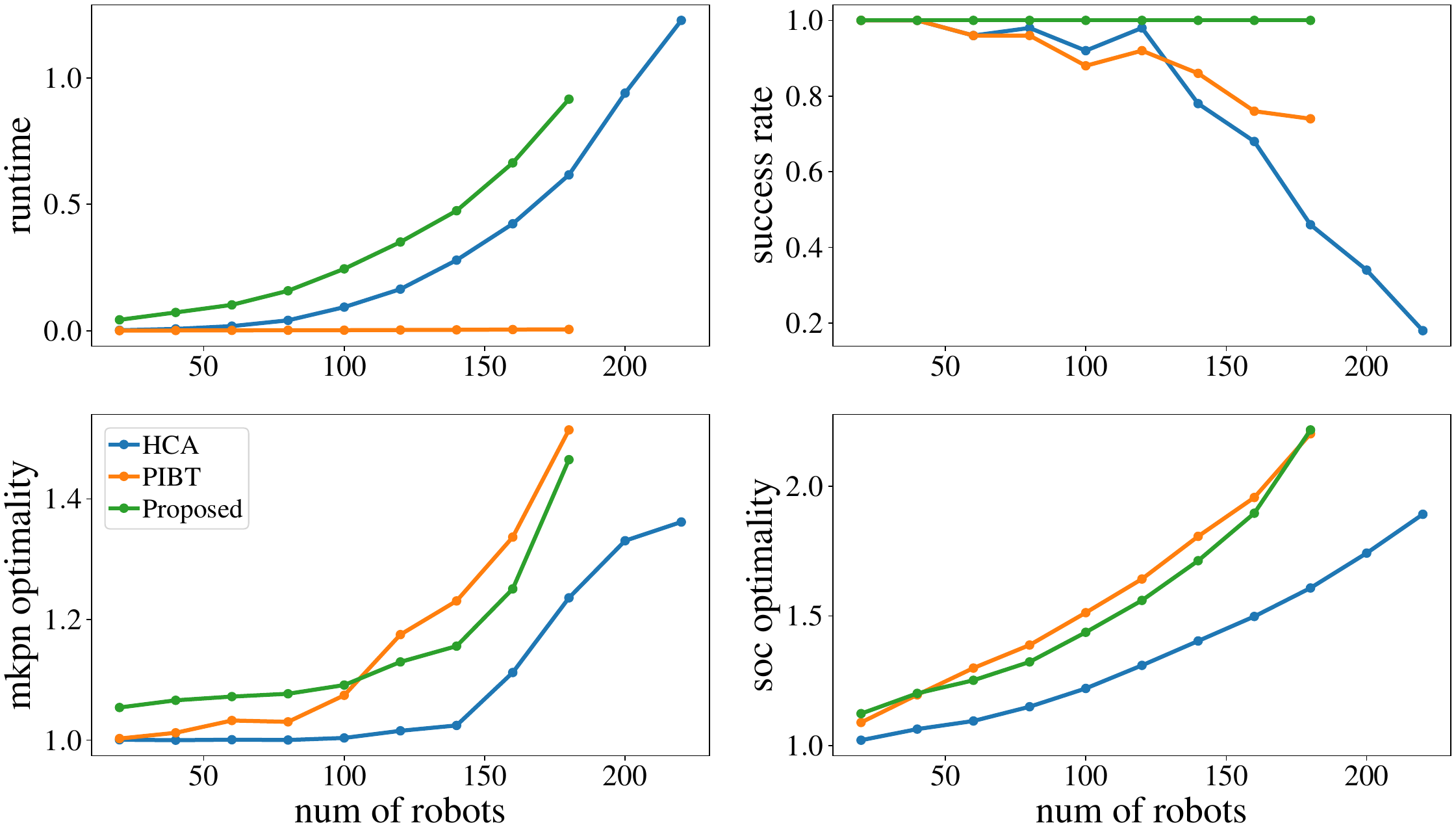}
    \put(-406, 196){\includegraphics[width=0.12\linewidth]{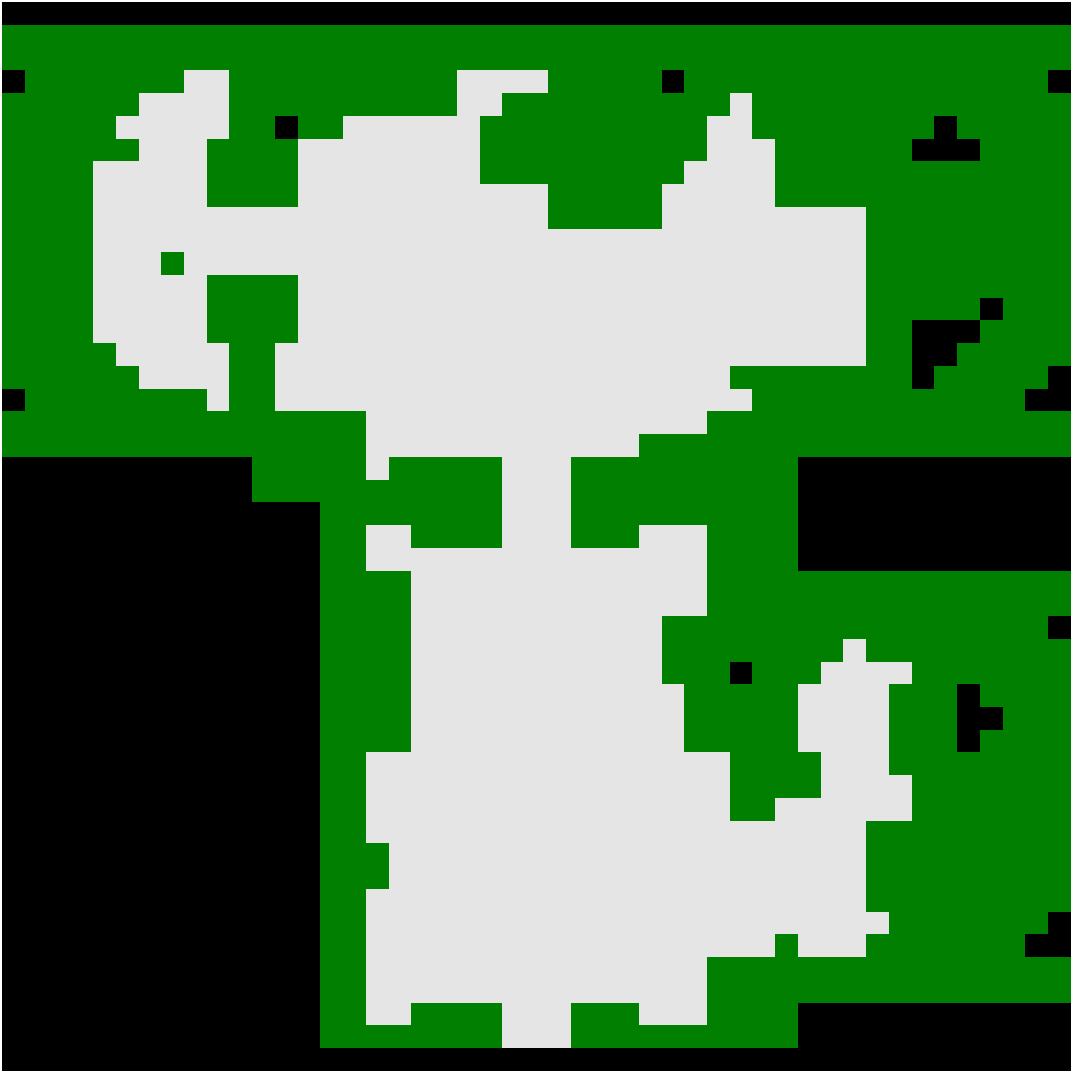}}
    \caption{Experimental results for map orz201d, including computation time, success rate, makespan optimality, and soc optimality, for HCA, PIBT, and the proposed method.}
    \vspace{-4mm}
    \label{fig:orz201d}
\end{figure}

\begin{figure}[!htpb]
    \centering
    \includegraphics[width=1\linewidth]{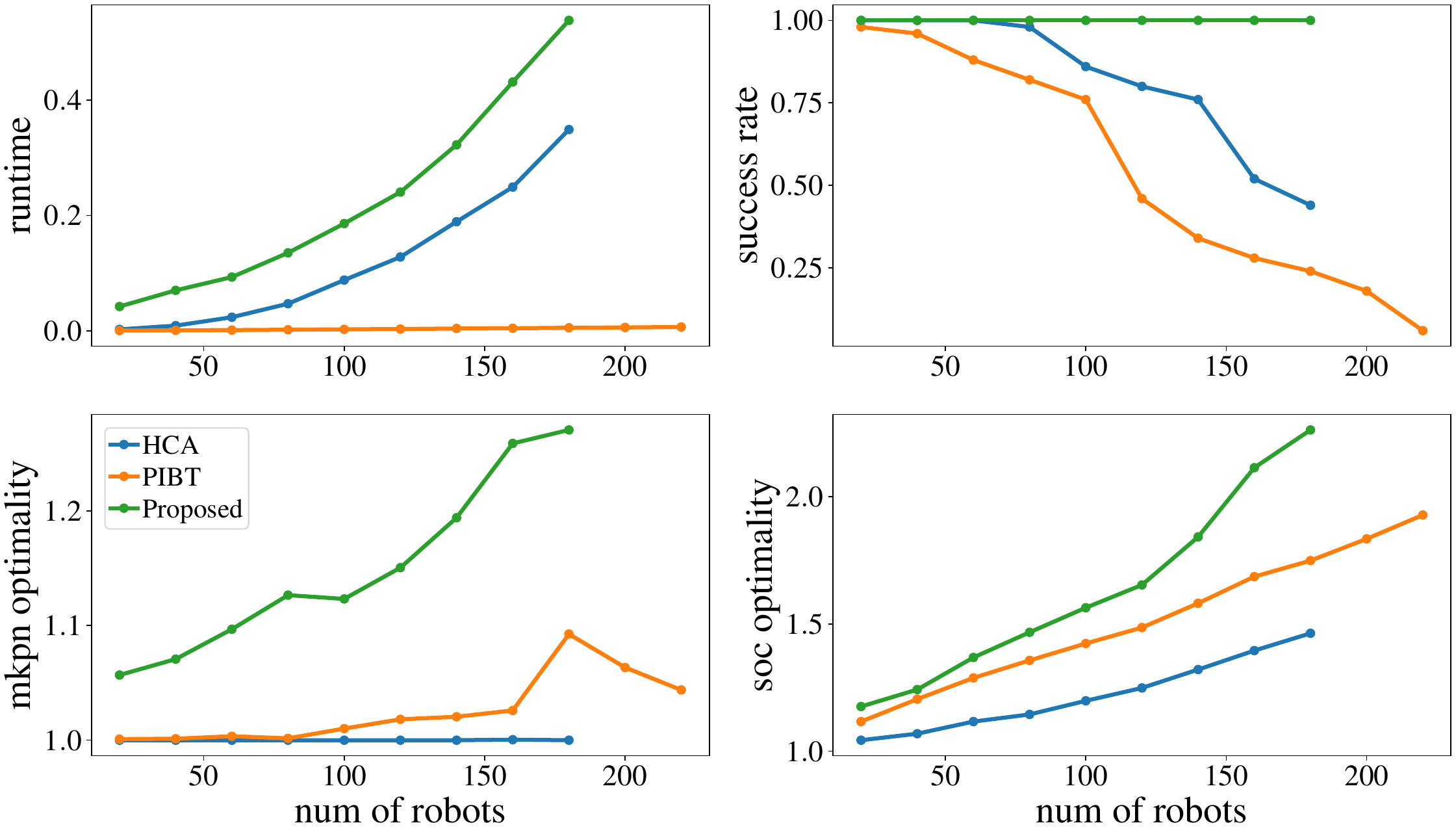}
    \put(-406, 196){\includegraphics[width=0.12\linewidth]{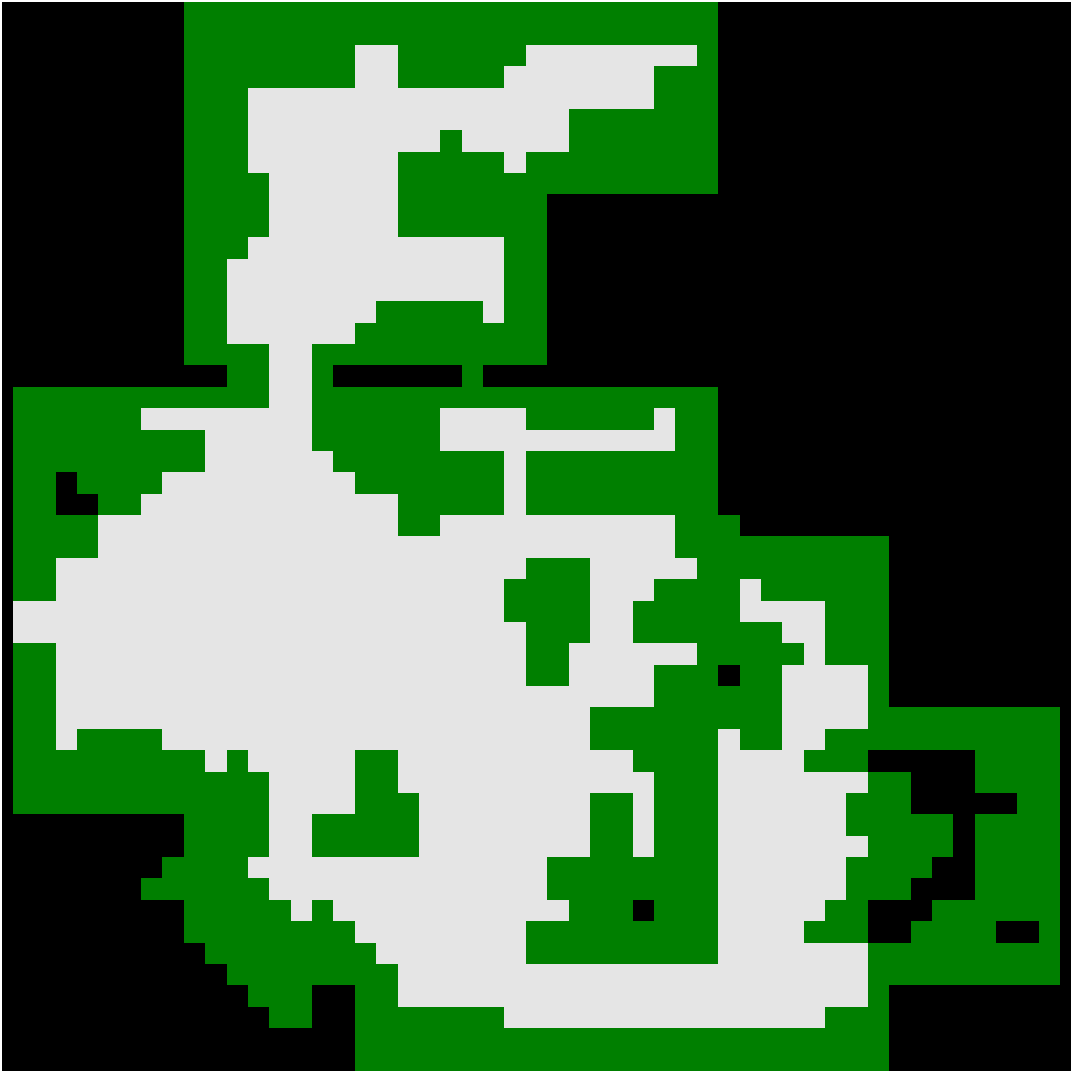}}
    \caption{Experimental results for map hrt002d, including computation time, success rate, makespan optimality, and soc optimality, for HCA, PIBT, and the proposed method.}
    \vspace{-6mm}
    \label{fig:hrt002d}
\end{figure}
\section{Conclusions}\label{sec6:conclusion}
In this paper, we have presented a comprehensive study of the \gls{lwcs} problem and its applications in multi-robot path planning. 
We provided a rigorous problem formulation and developed two algorithms, an exact optimal algorithm and a suboptimal algorithm, to solve the problem efficiently. 
We have shown that the problem has various real-world applications, such as parking and storage systems, multi-robot coordination, and path planning. 
Our algorithms have been evaluated on various maps to demonstrate their effectiveness in finding solutions. 
Moreover, we have integrated the \gls{lwcs} problem with prioritized planning to plan paths for multi-robot systems without encountering deadlocks. 
%
Our study enhances comprehension of the relationship between multi-robot path planning complexity, the number of robots, and graph topology, laying a robust groundwork for future research in this domain.
In future work, we plan to investigate the performance of our algorithms in more complex environments and explore their scalability in solving larger instances of the problem.
Additionally, we aim to explore the potential of our algorithms in real-world applications and examine their robustness against uncertainties and disruptions.

%% file: chapters/chapter5.tex
\chapter{Toward Efficient Physical and Algorithmic Design of Automated Garages}~\label{chap:garage}
\section{Introduction}

The invention of automated parking systems (garages) helps solve parking issues in areas where space carries significant premiums, such as city centers and other heavily populated areas.
Nowadays, parking space is becoming increasingly scarce and expensive; a spot in Manhattan could easily surpass $200,000$ USD.
Developing garages supporting high-density parking that save space and are more convenient is thus highly attractive for economic/efficiency reasons.

In automated garages, human drivers only need to drop off (pick up) the vehicle in a specific I/O (Input/Output) port without taking care of the parking process.
Vehicles in such a system do not require ambient space for opening the doors, and can thus be parked much closer.
Moving a vehicle to a parking spot or a port is the key function for such systems.  
One of the solutions is to use robotic valets to move vehicles. Such systems are already commercially available, such as HKSTP \cite{HKSTP} in Hong Kong.
In such systems \cite{nayak2013robotic}, vehicles are parked such that they may block each other, requiring multiple rearrangements to retrieve a specific vehicle. 
Unfortunately, little information can be found on how well these systems function, e.g., their parking/retrieval efficiency. 
\begin{figure}[h]
    \centering
    \includegraphics[width=\linewidth]{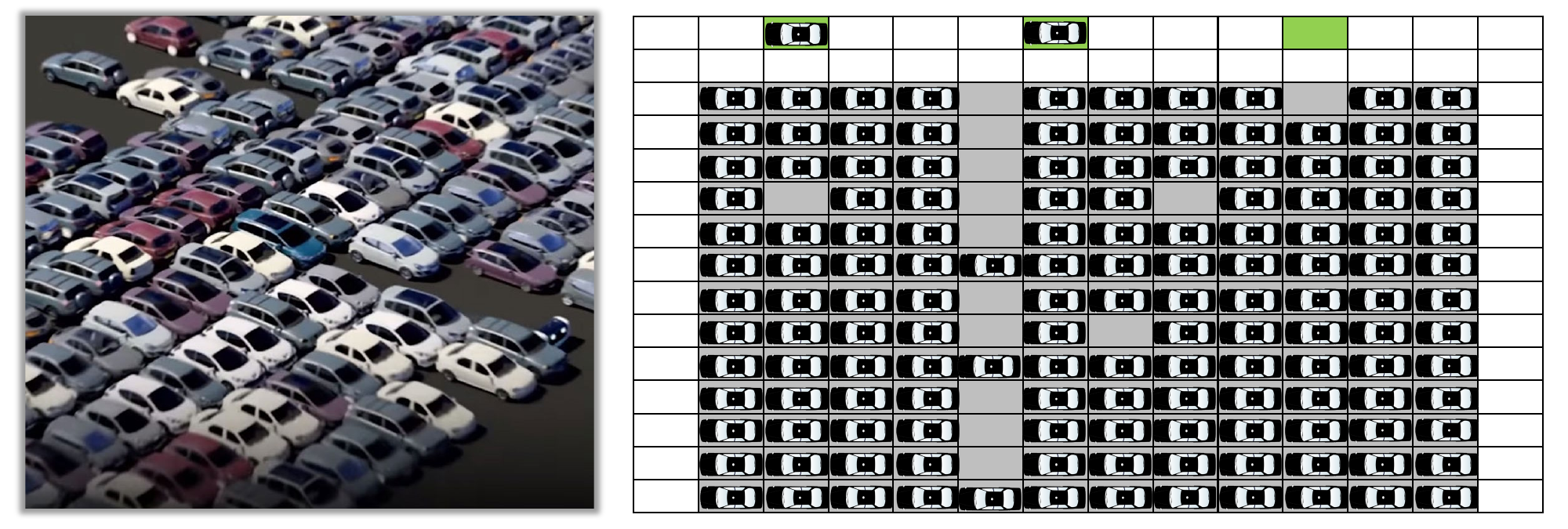}
    \caption{Left: an illustration of the density level of the envisioned automated garage system. Right: The grid-based abstraction with three I/O ports for vehicle dropoff and retrieval.}
    \label{fig:av_system}
    \vspace{-2mm}
\end{figure}%
Other solutions focus on parking for self-driving vehicles. 
In such an automated garage, vehicles are able to drive themselves  to parking slots and ports.  
This makes the system more flexible but provides limited space-saving advantages, besides requiring autonomy from the vehicles.

Recently, many efficient multi-robot path planning algorithms have been proposed, making it possible for lowering parking and retrieval cost using multiple robotic valets.
In this chapter, we study multi-robot based parking and retrieval problem, proposing a complete automated garage design, supporting near $100\%$ parking density, and developing associated algorithms for efficiently operating the garage. 

\textbf{Results and contributions.}
The main results and contributions of our work are as follows. In designing the automated garage, we introduce \emph{batched vehicle parking and retrieval} (\gls{bcpr}) and \emph{continuous vehicle parking and retrieval} (\gls{ccpr}) problems modeling the key operations required by such a garage, which facilitate future theoretical and algorithmic studies of automated garage systems.

On the algorithmic side, we study a system that can support parking density as high as $(m_1-2)(m_2-2)/m_1m_2$ on a $m_1\times m_2$ grid map and allow multi-vehicle parking and retrieving, which approaches $100\%$ parking density for large garages. Leveraging the regularity of the system, which is grid-like, we propose an optimal \gls{ilp}-based method and a fast suboptimal algorithm based on sequential planning.
Our suboptimal algorithm is highly scalable while maintaining a good level of solution quality, making it suitable for large-scale applications. 
We further introduce a shuffling mechanism to rearrange vehicles during off-peak hours for fast vehicle retrieval during rush hours, if the retrieval order can be anticipated. Our rearrangement algorithm performs such shuffles with total time cost of $O(m_1m_2)$ at near full garage density.

%

\textbf{Related work.} 
Researchers have proposed diverse approaches toward efficient high-density parking solutions.
Many systems for self-driving vehicles have been studied \cite{ferreira2014self,timpner2015k,nourinejad2018designing},
%
%
where vehicles are parked using a central controller and may be stacked in several rows and can block each other.
These designs increase parking capacity by up to $50\%$. However, the retrieval becomes highly complex due to blockages and is heavily affected by the maneuverability of self-driving vehicles.

With most vehicles being incapable of self-driving, robotic valet based high-density parking systems could be a more appropriate choice. 
The Puzzle Based Storage (PBS) system or  grid-based shuttle system, proposed originally by \cite{Gue2007PuzzlebasedSS}, is one of the most promising high-density storage systems.
In such a system, storage units, which can be  AGVs or shuttles, are movable in four cardinal directions. There must be at least one empty cell (escort).  To retrieve a vehicle, one must utilize the escorts to move the desired vehicles to an I/O port. This is similar to the 15-puzzle, which is known to be NP-hard to optimally solve \cite{ratner1986finding}.
Optimal algorithms for retrieving one vehicle with a single escort and multiple escorts have been proposed in \cite{Gue2007PuzzlebasedSS,optimalMultiEscort}.
However, these methods only consider retrieving one single vehicle at a time.
Besides, the average retrieval time can be much longer than conventional aisle-based solutions.
To achieve a trade-off between capacity demands and retrieval efficiency, we suggest using more escorts and I/O ports that allow retrieving and parking multiple vehicles simultaneously by utilizing recent advanced Multi-Robot Path Planning (\gls{mpp}) algorithms \cite{okoso2022high}.

\gls{mpp} has been widely studied. In the static or one-shot setting \cite{stern2019multi},  given a graph environment and a number of robots with each robot having a unique start position and a goal position,  the task is to find collision-free paths for all the robots from start to goal.
It has been proven that solving one-shot \gls{mpp} optimally in terms of minimizing either makespan or sum of costs is NP-hard \cite{surynek2009novel,yu2013structure}.
Solvers for \gls{mpp} can be categorized into \emph{optimal} and \emph{suboptimal}.  
Optimal  solvers either reduce \gls{mpp} to other well-studied problems, such as ILP\cite{yu2016optimal}, SAT\cite{surynek2010optimization} and ASP\cite{erdem2013general} or use search algorithms to search the joint space to find the optimal solution \cite{sharon2015conflict,sharon2013increasing}.
Due to the NP-hardness, optimal solvers are not suitable for solving large problems.
Bounded suboptimal solvers \cite{barer2014suboptimal} achieve better scalability while still having a strong optimality guarantee. 
However, they still scale poorly, especially in high-density environments.
There are polynomial time algorithms for solving large-scale \gls{mpp} \cite{luna2011push}, which are at the cost of solution quality. Other $O(1)$ time-optimal polynomial time algorithms \cite{Guo2022PolynomialTN,GuoFenYu22IROS,han2018sear,yu2018constant} are mainly focusing on minimizing the makespan, which is not very suitable for continuous settings.

\textbf{Organization.}
The rest of the chapter is organized as follows. In~\ref{sec4:problem}, we cover the preliminaries including garage design. 
In~\ref{sec4:algo-bcpr}-~\ref{sec4:algo-crp}, we provide the algorithms for operating the automated garage. We perform thorough evaluations and discussions of the garage system in~\ref{sec4:evaluation} and conclude with~\ref{sec4:conclusion}.

\section{Problem Definition}\label{sec4:problem}
\subsection{Garage Design Specification}

In this study, the automated grid-based garage is a four-connected $m_1\times m_2$ grid $\mathcal{G}(\mathcal{V},\mathcal{E})$ (see \ref{fig:av_system}). 
There are $n_o$ I/O ports (referred simply as \emph{ports} here on) distributed on the top border of the grid for dropping off vehicles for parking or for retrieving a specific parked vehicle.
A port can only be used for either retrieving or parking at a given time.
Vehicles must be parked at a \emph{parking spot}, a cell of the lower center $(m_1 -2)\times(m_2-2)$ subgrid. 
Once a vacant spot is parked, it becomes a movable obstacle.
%
%
$\mathcal{O}=\{o_1,...,o_{n_o}\}$ is the set of ports and $\mathcal{P}=\{p_1,...,p_{|\mathcal{P}|}\}$ is the set of parking spots.  
%

\subsection{Batched Vehicle Parking and Retrieval (\gls{bcpr})}
\emph{Batched vehicle parking and retrieval}, or \gls{bcpr}, seeks to optimize parking and retrieval in a \emph{batch} mode. 
In a single batch, there are $n_p$ vehicles to park, and $n_r$ vehicles to retrieve, $n_l$ parked vehicles to remain.
Denote $\mathcal{C}=\mathcal{C}_p\cup\mathcal{C}_r\cup\mathcal{C}_l$ as the set of all vehicles.
At any time, the maximum capacity cannot be exceeded, i.e.,  $|\mathcal{C}|<|\mathcal{P}|$.
Time is discretized into timesteps and multiple vehicles carried by AGVs/shuttles can move simultaneously.
In each timestep, each vehicle can move left, right, up, down or wait at the current position.
Collisions among the vehicles should be avoided:
\begin{enumerate}[leftmargin=4.5mm]
    \item Meet collision. Two vehicles cannot be at the same grid point at any timestep: $\forall i,j\in\mathcal{C}, v_i(t)\neq v_j(t)$;
    \item Head-on collision. Two vehicles cannot swap locations by traversing the same edge in the opposite direction: $\forall i,j\in\mathcal{C}, (v_i(t)= v_j(t+1) \wedge v_i(t+1)=v_j(t))=\text{false}$;
    \item Perpendicular following collisions (see \ref{fig:perp}). One vehicle cannot follow another when their moving directions are perpendicular. Denote $\hat{e}_i(t)=v_i(t+1)-v_i(t)$ as the moving direction vector of vehicle $i$ at timestep $t$, then $\forall i,j \in \mathcal{C}, (v_i(t+1)=v_j(t) \wedge \hat{e}_i(t)\perp \hat{e}_j(t))=\text{false}$.
\end{enumerate}
\begin{figure}[!htbp]
    \centering
    \includegraphics[width=.80\linewidth]{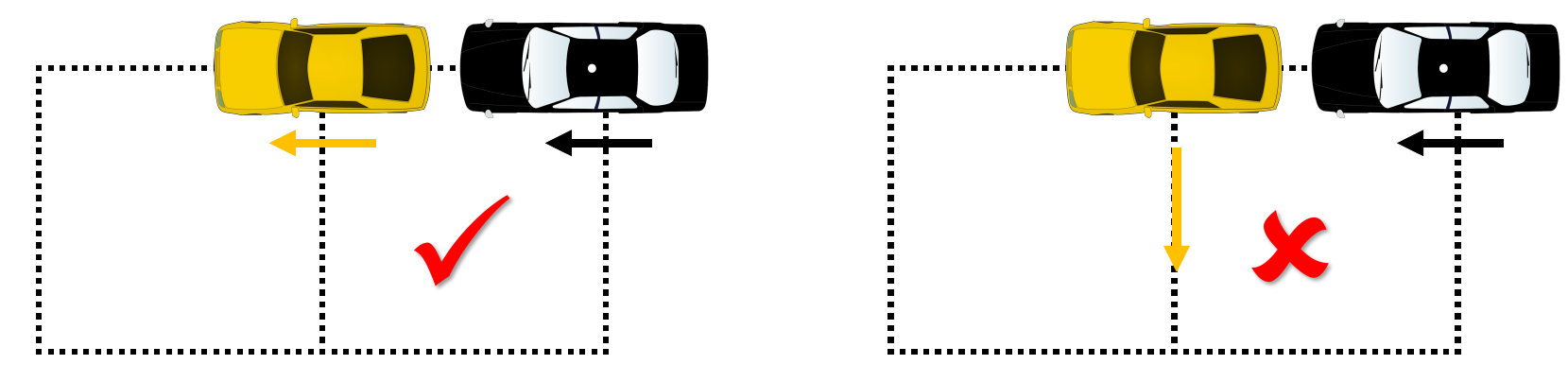}
    \caption{While parallel movement of vehicles in the same direction (left) is allowed, \emph{perpendicular following} of vehicles (right) is forbidden.}
    \label{fig:perp}
\end{figure}%

Unlike certain \gls{mpp} formulations \cite{stern2019multi}, we need to consider perpendicular following collisions, which makes the problem harder to solve.
The task is to find a collision-free path for each vehicle, moving it from its current position to a desired position.
Specifically, for a vehicle to be retrieved, its goal is a specified port. 
For other vehicles, the desired position is any one of the parking spots.
The following criteria are used to evaluate the solution quality:
\begin{enumerate}[leftmargin=4.5mm]
    \item \gls{aprt}: the average time required to retrieve or park a vehicle.
    \item \gls{makespan}: the time required to move all vehicles to their desired positions.
    \item \gls{anm}: the sum of the distance of all vehicles divided by the total number of vehicles in $\mathcal{C}_p\cup \mathcal{C}_r$.
\end{enumerate}
In general, these objectives create  a  Pareto  front \cite{yu2013structure} and it  is  not always possible to simultaneously optimize any two of these objectives.

\subsection{Continuous Vehicle Parking and Retrieval (\gls{ccpr})}\label{sec4:algo-crp}
\gls{ccpr} is the continuous version of the vehicle parking and retrieval problem. It inherits most of \gls{bcpr}'s structure, but with a few key differences.
In this formulation, we make the following assumptions. 
When a vehicle $i\in \mathcal{C}_r$ arrives at its desired port, it would be removed from the environment.
There will be new vehicles appearing at the ports that need to be parked (within capacity) and there would be new requests for retrieving vehicles.
Besides, when a port is being used for retrieving a vehicle, other users cannot park vehicles at the port until the retrieval task is finished.
Except for \gls{makespan} when the time horizon is infinite or fixed, the three criteria can still be used for evaluating the solution quality. 

\subsection{\gls{crp}}
In real-world garages, there are often off-peak periods (e.g., after the morning rush hours) where the system reaches its capacity and there are few requests for retrieval.
If the retrieval order of the vehicles at a later time (e.g., afternoon rush hours) is known, then we can utilize the information to reshuffle the vehicles to facilitate the retrieval later. 
This problem is formulated as a one-shot \gls{mpp}. 
Given the start configuration $X_I$ and a goal configuration $X_G$, we need to find collision-free paths to achieve the reconfiguration. The goal configuration is determined according to the retrieval time order of the vehicles; vehicles expected to be retrieved earlier should be parked closer to the ports so that they are not blocked by other vehicles that will leave later. 

\section{Solving BCPR}\label{sec4:algo-bcpr}
\subsection{Integer Linear Programming (\gls{ilp})}
Building on network flow based ideas from 
\cite{yu2016optimal,Ma2016OptimalTA}, we reduce \gls{bcpr} to a multi-commodity max-flow problem and use integer programming to solve it.
Vehicles in $\mathcal{C}_r$ have a specific goal location (port) and must be treated as different commodities.
On the other hand, since the vehicles in $\mathcal{C}_p\cup\mathcal{C}_l$ can be parked at any one of the parking slot, they can be seen as one single commodity.
Assuming an instance can be solved in $T$ steps, we  construct a $T$-step  time-extended  network as shown in \ref{fig:maxflow}. 

A $T$-step  time-extended network is  a  directed  network $\mathcal{N}_T=(\mathcal{V}_T,\mathcal{E}_T)$  with directed, unit-capacity edges.
The  network $\mathcal{N}_T$ contains $T+1$ copies of the original graph's vertices $\mathcal{V}$.
The copy of vertex $u\in \mathcal{V}$ at timestep $t$ is denoted as $u_t$.
At timestep $t$, an edge $(u_t,v_{t+1})$ is added to $\mathcal{E}_T$ if $(u,v)\in \mathcal{E}$ or $v=u$.
For $i\in \mathcal{C}_r$, we can give a supply of one unit of commodity type $i$ at the vertex $s_{i0}$ where $s_i$ is the start vertex of $i$.
To ensure that vehicle $i$ arrives at its  port, we add a feedback edge connecting its goal vertex $g_i$ at $T$ to its source node $s_{i0}$.
For the vehicles that need to be parked, we create an auxiliary source node $\alpha$ and an auxiliary sink node $\beta$.
For each $i\in \mathcal{C}_p \cup \mathcal{C}_l$, we add an edge of unit capacity connecting node $\alpha$ to its starting node $s_{i0}$.
As the vehicles can be parked at any one of the parking slots, for any vertex $u\in \mathcal{P}$ we add an edge of unity capacity connecting $u_T$ and $\beta$.
A supply of $n_p+n_l$ unit of commodity of the type for the vehicles in $\mathcal{C}_p\cup\mathcal{C}_l$ can be given at the node $\alpha$.

To solve the multi-commodity max-flow using \gls{ilp}, we create a set of binary variables $X=\{x_{iuvt}\}, i=0,...,n_r, (u,v)\in \mathcal{E}$ or $u=v$, $0\leq t\leq T$; a variable set to true means that the corresponding edge is used in the final solution.
The \gls{ilp} formulation is given as follows.
%
  \vspace{1mm}
\begin{equation}
      \text{Minimize \quad} \sum_{i,t,u\neq v}x_{iuvt} \label{eqn:eq_obj}
\end{equation}
\begin{equation}
    \text{subject to \quad}\forall t,v,i\quad \sum_{u}x_{iuv(t-1)}=\sum_{w}x_{ivwt}\label{eqn:flow_constraint}
\end{equation}
\begin{equation}
     \forall t,i,v\quad \sum_{v}x_{iuvt}\leq 1\label{eqn:vertex_constraint}
\end{equation}
  \begin{equation}
       \forall t,i,(u,v)\in\mathcal{E}\quad  \sum_{i}(x_{iuvt}+x_{ivut})\leq 1\label{eqn:edge_constraint}
  \end{equation}
\begin{equation}
     \forall t,i,(u,v)\perp(v,w) \sum_{i}(x_{iuvt}+x_{ivwt})\leq 1 \label{eqn:following_constraint}
\end{equation}
\begin{equation}
            \sum_{i=0}^{n_r-1} x_{ig_is_iT}+\sum_{u\in \mathcal{P}}x_{n_r u\beta T}=n_p+n_r+n_l \label{eqn:goal_constraints}
\end{equation}

\begin{equation}
          x_{iuvt}=\begin{cases}
         0&\text{if $i$ does not traverse edge $(u,v)$ at $t$ }\\
         1 &\text{if $i$  traverses edge $(u,v)$ at $t$ }\\
         \end{cases}\label{eqn:eq_binaries}
\end{equation}
\vspace{1mm}
%

In Eq. \eqref{eqn:eq_obj}, we minimize the total number of moves of all vehicles within the time horizon $T$.
Eq.~\eqref{eqn:flow_constraint} specifies the flow conservation constraints at each grid point.
Eq.~\eqref{eqn:vertex_constraint} specifies the vertex constraints to avoid meet-collisions.
In Eq.~\eqref{eqn:edge_constraint}, the vehicles are not allowed to traverse the same edge in opposite directions.
Eq.~\eqref{eqn:following_constraint} specifies the constraints that forbid perpendicular following conflicts.
If the programming for $T$-step time-expanded network is feasible, then the solution is found.
Otherwise, we increase $T$ step by step until there is a feasible solution.
The smallest $T$ for which the integer programming has a solution is the minimum makespan.
As a result, the \gls{ilp} finds a makespan-optimal solution minimizing the total number of moves as  a secondary objective.
\begin{figure}[!htbp]
\vspace{2mm}
    \centering
    \includegraphics[width=.90\linewidth]{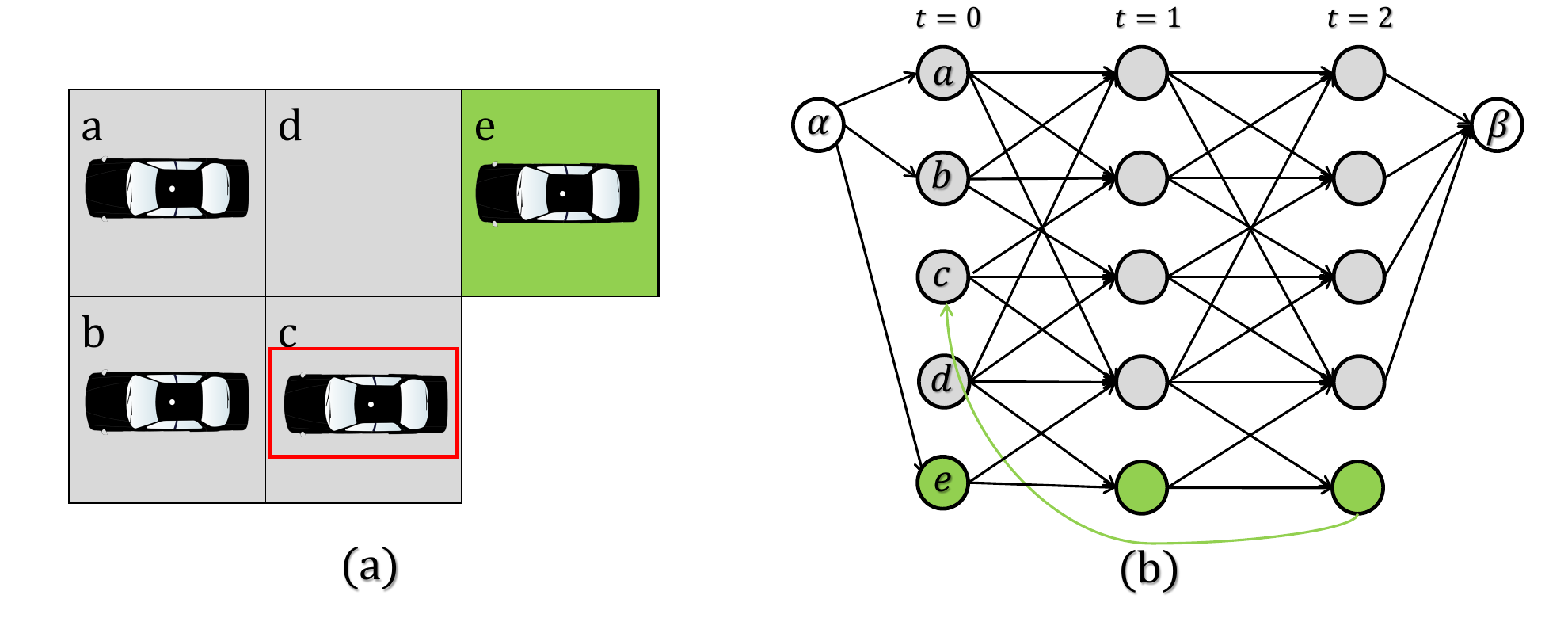}
    \vspace{-2mm}
    \caption[An example of the maxflow algorithm.]{(a) A illustrative \gls{bcpr} instance with 4 parking spots and one port. The vehicles within the red rectangle need to be retrieved while other vehicles should be parked in the gray areas. (b) The 2-step flow network reduced from the \gls{bcpr} instance.}
    \label{fig:maxflow}
    \vspace{-2mm}
\end{figure}%
\subsection{Efficient Heuristics for High-Density Planning}
\gls{ilp} can find makespan-optimal solutions but it scales poorly. 
We seek a fast algorithm that is able to quickly solve dense \gls{bcpr} instances at a small cost of solution quality.
The algorithm we propose is built on a single-vehicle motion primitive of retrieving and parking. 
\subsubsection{Motion Primitive for Single-Vehicle Parking/Retrieving}
Regardless of the solution quality, \gls{bcpr} can be solved by sequentially planning for each vehicle in $\mathcal{C}_r\cup\mathcal{C}_p$.
Specifically, for each round we only consider completing one single task for a given vehicle $i\in\mathcal{C}_r\cup\mathcal{C}_p$, which is either moving $i$ to its port or one of the parking spots.
%
%
After vehicle $i$ arrives at its destination, the next task is solved.
%
%
The examples in \ref{fig:single_retrieve} and \ref{fig:single_parking}, where the maximum capacity is reached, illustrate the method's operations.

For the retrieval scenario in \ref{fig:single_retrieve}, the vehicle marked by the red rectangle is to be retrieved.
For realizing the intuitive path indicated by the dashed lines on the left, vehicles blocking the path should be cleared out of the way, which can be easily achieved by moving those blocking vehicles one step to the left or to the right, utilizing the two empty columns. 
%
%
Such a motion primitive can always successfully retrieve a vehicle without deadlocks.

For the parking scenario in \ref{fig:single_parking}, we need to park the vehicle in the green port.
We do so by first searching for an empty spot (escort) greedily. 
Using the mechanism of parallel moving of vehicles, the escort can first be moved to the column of the parking vehicle in one timestep and then moved to the position right below the vehicle in one timestep.
After that, the vehicle can move directly to the escort, which is its destination.
\begin{figure}[!htbp]
\vspace{1mm}
    \centering
    \includegraphics[width=\linewidth]{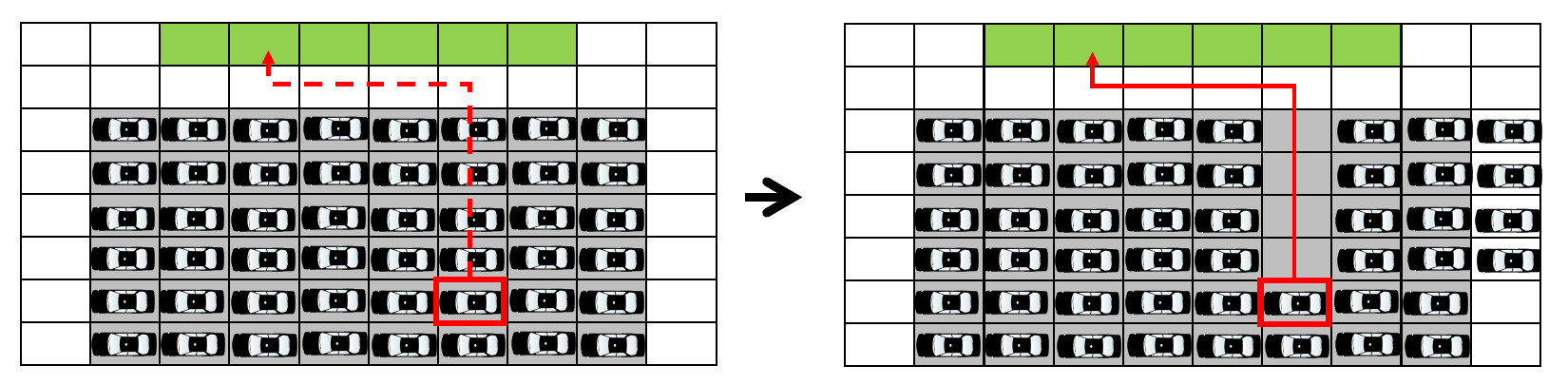}
    \caption{The motion primitive for retrieving a vehicle.}
    \label{fig:single_retrieve}
\end{figure}%

\begin{figure}[!htbp]
    \centering
    \includegraphics[width=\linewidth]{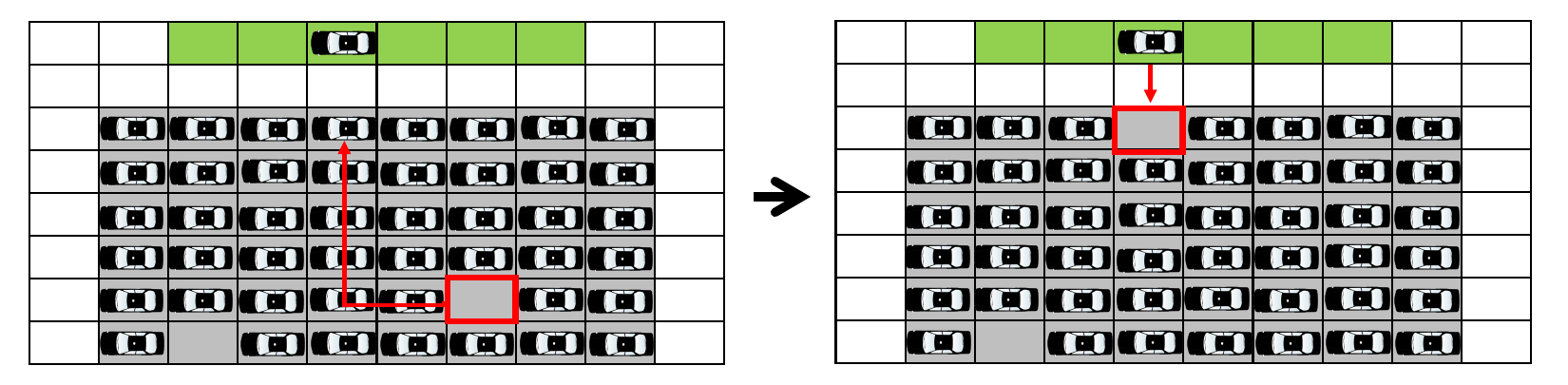}
    \caption{The motion primitive for parking a vehicle.}
    \label{fig:single_parking}
\end{figure}%

Sequential planning always returns a solution if there is one; however, when multiple parking and retrieval requests are to be executed, sequential solutions result in poor performance as measured by MKPN and ARPT.
\subsubsection{Coupling Single Motion Primitives by \gls{mcp} (\gls{csmp})} 
\gls{mcp} \cite{ma2017multi} is a robust multi-robot execution policy to handle unexpected delays without stopping unaffected robots.
During  execution,  \gls{mcp}  preserves  the  order  by which robots visit each vertex as in the original plan. 
When a robot $i$ is about to perform a move action and enter a vertex $v$, \gls{mcp} checks whether robot $i$ is the next to enter that vertex by the original plan.  
If a different robot, $j$, is planned to enter $v$ next, then $i$ waits in its current vertex until $j$ leaves $v$.

We use \gls{mcp} to introduce concurrency to plans found through sequential planning.
The algorithm is described in \ref{alg:csmp} and \ref{alg:MCP}.
\ref{alg:csmp} describes the framework of \gls{csmp}.
First, we find initial plans for all the vehicles in $\mathcal{C}_r\cup\mathcal{C}_p$ one by one (Line 4-6).
After obtaining the paths, we remove all the waiting states and record the order of vehicle visits for each vertex in a list of queues (Line 7).
Then we enter a loop executing the plans using \gls{mcp} until all vehicles have finished the tasks and reach their destination (Line 8-15).
In \ref{alg:MCP}, if  $i$ is the next vehicle that enters vertex $v_i$ according to the original order, we check if there is a vehicle currently at $v_i$.
If there is not, we let $i$ enter $v_i$.
If another vehicle $j$ is currently occupying $v_j$, we examine if $j$ is moving to its next vertex $v_j$ in the next step by recursively calling the function 
$\texttt{MCPMove}$. 
If $j$ is moving to $v_j$ in the next step and the moving directions of $i,j$ are not perpendicular, we let vehicle $i$ enter vertex $v_i$.
Otherwise, $i$ should  wait at $u_i$. 
The algorithm is deadlock-free by construction; we omit the relatively straightforward proof due to the page limit.
\begin{proposition}
\gls{csmp} is dead-lock free and always finds a feasible solution in finite time  if there is one.
\end{proposition}
%
%
\begin{algorithm}[!htbp]
\DontPrintSemicolon
\SetKwProg{Fn}{Function}{:}{}
\SetKwFunction{Fcsmp}{\gls{csmp}}
\SetKwFunction{FMCP}{MCPMove}
\SetKw{Continue}{continue}

 \caption{CSMP \label{alg:csmp}}
 
\Fn{\Fcsmp()}{
    \textbf{foreach} $v\in \mathcal{V}$, $VOrder[v]\leftarrow Queue()$\;
    $InitialPlans\leftarrow \{\}$\;
    \For{$i\in \mathcal{C}_p\cup\mathcal{C}_r$}{
    \texttt{SingleMP($i$,$InitialPlans$)}\;
    }
    \texttt{Preprocess($InitialPlans,VOrder$)}\;
    \While{True}{
    \For{$i\in \mathcal{C}$}{
        $mcpMoved\leftarrow Dict()$\;
        \FMCP($i$)\;
    }
    \If{\texttt{AllReachedGoal()}=true}{
    break\;
    }
    }
}
\end{algorithm}

\begin{algorithm}[!htbp]
\DontPrintSemicolon
\SetKwProg{Fn}{Function}{:}{}
\SetKwFunction{Fcsmp}{\gls{csmp}}
\SetKwFunction{FMCP}{MCPMove}
\SetKw{Continue}{continue}

 \caption{MCPMove \label{alg:MCP}}
 
\Fn{\FMCP($i$)}{
\If{$i$ in $mcpMoved$}{
\Return $mcpMoved[i]$\;
}
$u_i\leftarrow$ current position of $i$\;
$v_i\leftarrow$ next position of $i$\;
\If{$i=VOrder[v_i].front()$}{
    $j\leftarrow$ the vehicle currently at $v_i$\;
    \If{$j$=None or (\FMCP($j$)=true and $(u_i,v_i)\not \perp (u_j,v_j)$)}{
        move $i$ to $v_i$\;
        $VOrder[v_i].popfront()$\;
        $mcpMoved[i]$=true\;
        \Return true\;
    }
}
let $i$ wait at $u_i$\;
$mcpMoved[i]$=false\;
\Return false\;
}
\end{algorithm}
\subsubsection{Prioritization}
Sequential planning can always find a solution regardless of the planning order.
However, priorities will affect the solution quality. 
Instead of planning by a random priority order (Alg. \ref{alg:csmp} Line 4-6), when possible, we can first plan for parking since single-vehicle parking only takes two steps.
After all the vehicles in $\mathcal{C}_p$ have been parked, we apply \texttt{SingleMP} to retrieve vehicles.
Among vehicles in $\mathcal{C}_r$, we first apply \texttt{SingleMP} for those vehicles that are closer to their port so that they can reach their targets earlier and will not block the vehicles at the lower row.
\subsection{Complexity Analysis}
In this section, we analyze the time complexity and solution makespan upper bound of the \gls{csmp}.
In \texttt{SingleMP}, in order to park/retrieve one vehicle, we assume $n_{b}$ vehicles  may cause blockages and need to be moved out of the way. Clearly $n_b<n$ where $n=n_p+n_r+n_l$.
Therefore, the complexity of computing the paths using \texttt{SingleMP} for all vehicles in $\mathcal{C}_r\cup\mathcal{C}_p$ is bounded by $(n_p+n_r)n$.
The path length of each single-vehicle path computed by \texttt{SingleMP} is no more than $m_1+m_2$.
The makespan of the paths obtained by concatenating all the single-vehicle paths is bounded by $n_r(m_1+m_2)+2n_p$.
This means that \gls{mcp} will take no more than $n_r(m_1+m_2)+2n_p$ iterations.
Therefore, the makespan of the solution is upper bounded by $n_r(m_1+m_2)+2n_p$.
In each loop of \gls{mcp}, we essentially run DFS on a graph that has $n=n_p+n_r+n_l$ nodes and traverse all the nodes, for which the time complexity is $O(n)$,
Therefore the time complexity of \gls{csmp}  is $O(n(n_rm_1+n_rm_2+2n_p))$.
In summary, the time complexity of \gls{csmp} is bounded by $O(n(n_rm_1+n_rm_2+2n_p))$, while the makespan is upper bounded by $n_r(m_1+m_2)+2n_p$.

\subsection{Extending \gls{csmp} to \gls{ccpr}}\label{sec4:algo-ccpr}

CSMP can be readily adapted to solve \gls{ccpr}.
%
Similar to the \gls{bcpr} version, we call \texttt{MCPMove} for each vehicle at each timestep.
When a new request comes at some timestep, we compute the paths for the associated vehicles using the \texttt{SingleMP} and update the information of vertex visit order. 
That is, when we apply \texttt{SingleMP} on a vehicle $i \in \mathcal{C}_p\cup\mathcal{C}_r$, if vehicle $j$ will visit vertex $u$ at timestep $t'$, then we push $i$ to the queue $VOrder[u]$, where the queue is always sorted by the entering time of $u$.
In this way, \gls{mcp} will execute the plans while maintaining the visiting order. The previously planned vehicles will not be affected by the new requests and keep executing their original plan.
The main drawback of this method is that it usually has worse solution quality than replanning since the visiting order is fixed.

%
\gls{crp} is essentially solving a static/one-shot \gls{mpp}.
On an $m_1\times m_2$ grid, it can be solved by applying the Rubik Table algorithm \cite{szegedy2020rearrangement}, using no more than $2(m_2-2)$ column shuffles and $(m_1-2)$ row shuffles. 
As an example shown in \ref{fig:hw_heuristic}, we may use two nearby columns to shuffle the vehicles in a given column fairly efficiently, requiring only $O(m_1)$ steps \cite{guo2022sub}.
%
%
%
%
The same applies to row shuffles.
Depending on the number of parked vehicles, one or more multiple row/column shuffles may be carried out simultaneously. We have (straightforward proofs are omitted due to limited space)
\begin{proposition}
\label{p:arxiv_prop}
\gls{crp} may be solved using $O(m_1m_2)$ makespan at full garage capacity and $O(m_1 + m_2)$ makespan when the garage has $\Theta(m_1m_2)$ empty spots and $\Omega(m_1m_2)$ parked vehicles. In contrast, with $\Omega(m_1m_2)$ parked vehicles, the required makespan for solving \gls{crp} is $\Omega(m_1 + m_2)$.
\end{proposition}

\begin{figure}[!htbp]
    \centering
    \includegraphics[width=\linewidth]{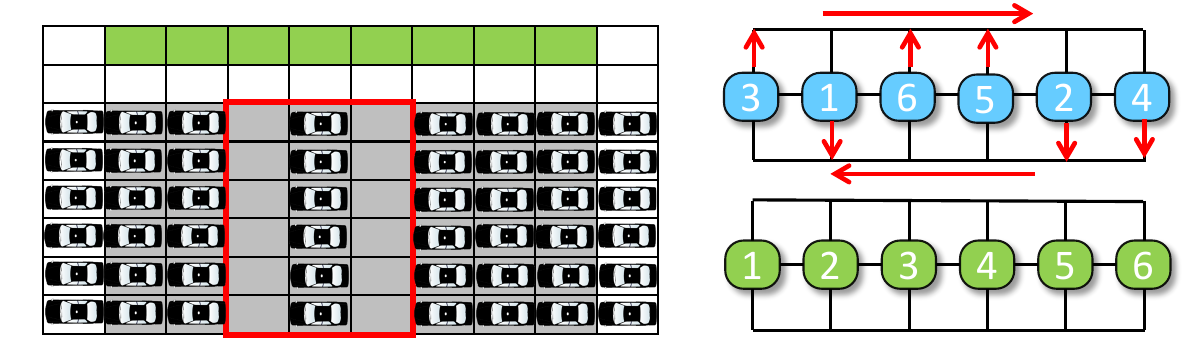}
    \caption{Illustration of the mechanism for ``shuffling'' a single row/column.}
    \label{fig:hw_heuristic}
\end{figure}%



\section{Evaluation}\label{sec4:evaluation}
In  this  section,  we  evaluate the  proposed algorithms.
All experiments are performed on an Intel\textsuperscript{\textregistered} Core\textsuperscript{TM} i7-9700 CPU at 3.0GHz. Each data point is an average over 20 runs on randomly generated instances unless otherwise stated.
\gls{ilp} is implemented in C++ and other algorithms are implemented in CPython. 
A video of the simulation can be found at \url{https://youtu.be/CZTaxnAS7TU}.

\subsection{Algorithmic Performance on \gls{bcpr}}
\textbf{Varying grid sizes.}
In the first experiment, we evaluate the proposed algorithms on $m\times m$ grids with varying grid side length, under the densest scenarios: there are $(m-2)^2$ vehicles in the system and all the ports are used for either parking or retrieving ($n_p+n_r=n_o$). 
The result can be found in \ref{fig:size_exp}.
CONCAT is the method that simply concatenates the sing-vehicle paths. 
In rCSMP, we apply \texttt{SingleMP} on vehicles with random priority order, while in p\gls{csmp}, we apply the prioritization strategy.
\begin{figure}[!htbp]
\vspace{2mm}
    \centering
    \includegraphics[width=1.0\linewidth]{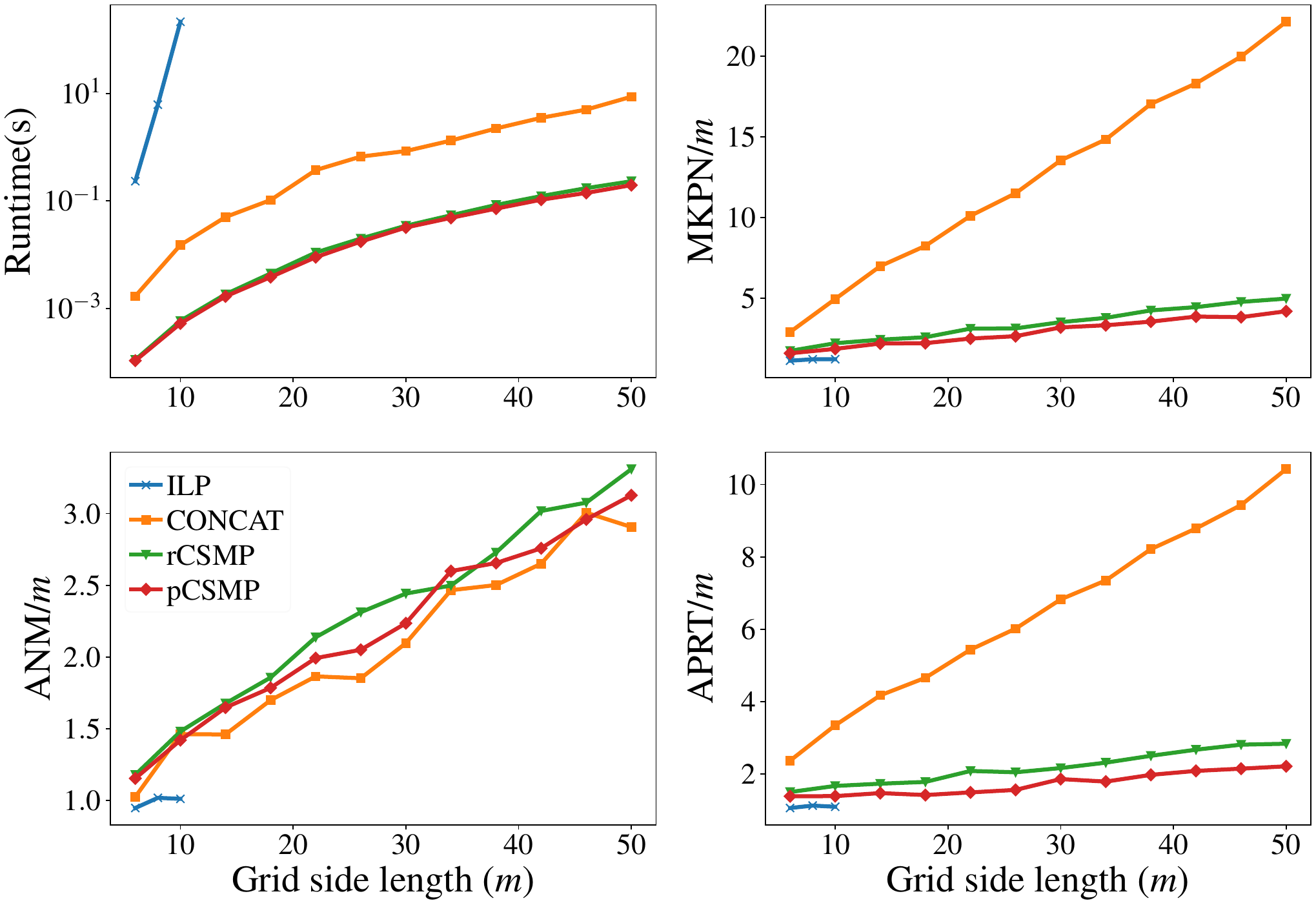}
    \caption{Runtime, MKPN, APRT, AVN data of the proposed methods on $m\times m$ grids under the densest scenarios.}
    \label{fig:size_exp}
\end{figure}%

Among the methods, \gls{ilp} has the best solution quality in terms of MKPN, ANM, and APRT, which is expected since its optimality is guaranteed. 
However, \gls{ilp} has the poorest scalability, hitting a limit with $m\leq 10$ and $n\leq64$.
CONCAT, r\gls{csmp}, and p\gls{csmp} are much more scalable, capable of solving instances on $50\times50$ grids with 2304 vehicles in a few seconds.
%
%
Since CONCAT just concatenates sing-vehicle paths, this results in very long paths compared to r\gls{csmp} and p\gls{csmp}; we observe that the \gls{mcp} procedure greatly improves the concurrency, leading to much better solution quality.
MKPN and APRT of paths obtained by CONCAT can be $10m$-$20m$ while the MKPN and APRT of the paths obtained by r\gls{csmp} and p\gls{csmp} are $2m$-$4m$.
MKPN and APRT of p\gls{csmp} with prioritization strategy is about $20\%$ lower than these of r\gls{csmp}. 

\textbf{Impact of vehicle density.}
In the second experiment, we examine the behavior of algorithms as vehicle density changes, fixing grid size at $20\times 20$. We still let $n_p+n_r=n_o$.
The result is shown in \ref{fig:density_exp}.
As in the previous case, \gls{ilp} can only solve instances with a density below $20\%$ in a reasonable time, while the other three algorithms can all tackle the densest scenarios.
For all algorithms, vehicle density in $\mathcal{C}_l$ has limited impact on MKPN and APRT, where r\gls{csmp} and p\gls{csmp} have much better quality than CONCAT. 
In low-density scenarios, fewer vehicles need to move, which may cause blockages for retrieving/parking a vehicle.
As a result, ANM increases as vehicle density increases.

\begin{figure}[!htbp]
    \centering
    \includegraphics[width=1.0\linewidth]{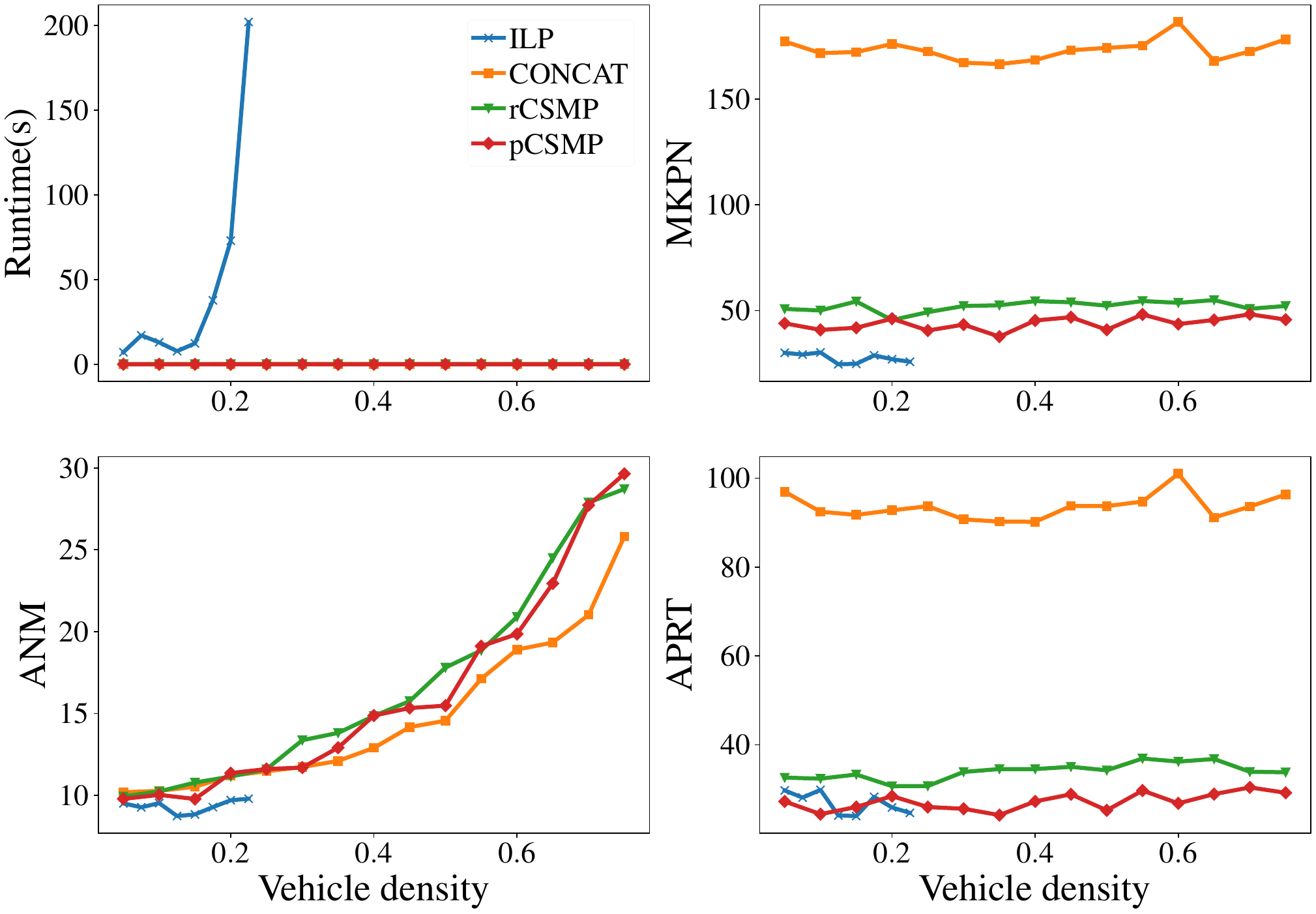}
    \vspace{-4mm}
    \caption{Runtime, \gls{makespan}, \gls{aprt}, \gls{anm} data of the proposed methods on $20\times 20$ grids with varying vehicle density.}
    \label{fig:density_exp}
\end{figure}%
    \vspace{-1.5mm}
\subsection{Algorithmic Performance on \gls{ccpr}}
    \vspace{-1mm}
\textbf{Random retrieving and parking.}
We test the continuous \gls{csmp} on a $12\times 12$ grid with $10$ ports. 
In each time step, if a port is available,  there would be a new vehicle that need to be parked appearing at this port with probability $p_p$ if it does not exceed the capacity. And with probability $p_r$ this port will be used to retrieve a random parked vehicle if there is one.
We simulate the following three scenarios:

(i). Morning rush hours. Initially, no vehicles are parked. There are many more requests for parking than retrieving: $p_p=0.6,p_r=0.01$.

(ii). Workday hours. Initially, the garage is full. Request for parking and retrieval are equal: $p_p=p_r=0.05$.

(iii). Evening rush hours. Initially, the garage is full. Retrieval requests dominate parking: $p_p=0.01, p_r=0.6$.

The maximum number of timesteps is set to 500. We evaluate the average retrieval time, average parking time, and total number of moves under these scenarios. The result is shown in \ref{fig:online_scenarios}.
Online CSMP achieves the best performance in the morning due to fewer retrievals.
On the other hand, the average retrieval time in all three scenarios is less than $2m$ and parking time is less than $m$. This shows that the algorithm is able to plan paths with good solution quality even in the densest scenarios and rush hours.

\begin{figure}[!htbp]

    \centering
    \includegraphics[width=1.0\linewidth]{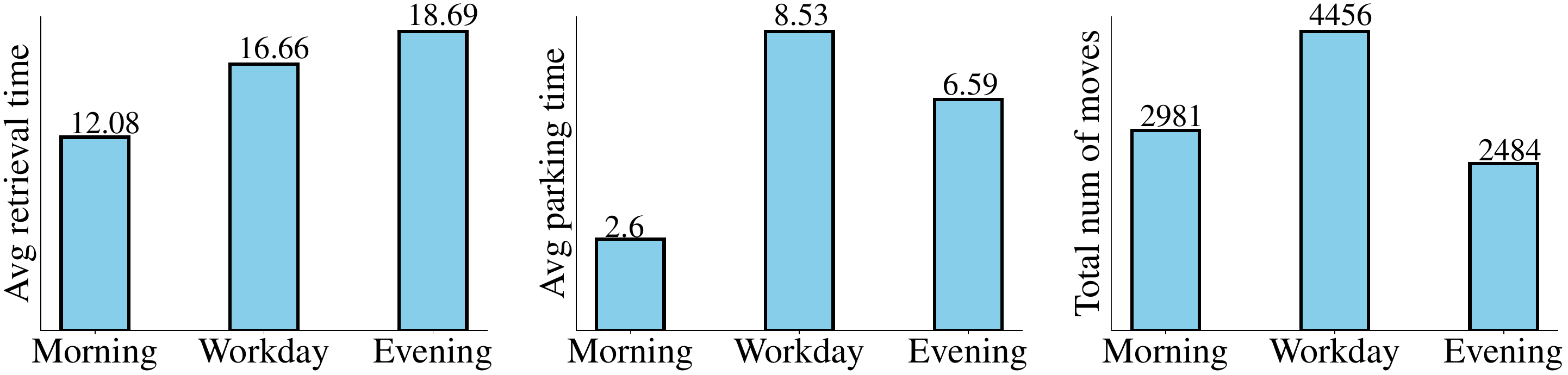}
    \caption{\gls{ccpr} performance statistics under three garage traffic patterns.}
    \label{fig:online_scenarios}
    \vspace{-2mm}
\end{figure}%

\textbf{Benefits of shuffling.}
In this experiment, we examine the effect of the shuffling (for solving \gls{crp}).
We assume that each vehicle is assigned a retrieval priority order, as to be expected in the evening rush hours when some people go home earlier than others. 
We perform the \emph{column} shuffle operations on the vehicles to facilitate the retrieval.
The makespan, average number of moves, and computation time of the column shuffle operations on $m\times m$ grids with different grid sizes under the densest settings are shown in \ref{fig:shuffles}(a)-(c).
While the paths of shuffling can be computed in less than 1 second, the makespan of completing the shuffles scales linearly with respect to $m^2$ and the average number of moves scales linearly with respect to $m$. 

\begin{figure}[!htbp]
\vspace{-1mm}
    \centering
    \includegraphics[width=1.0\linewidth]{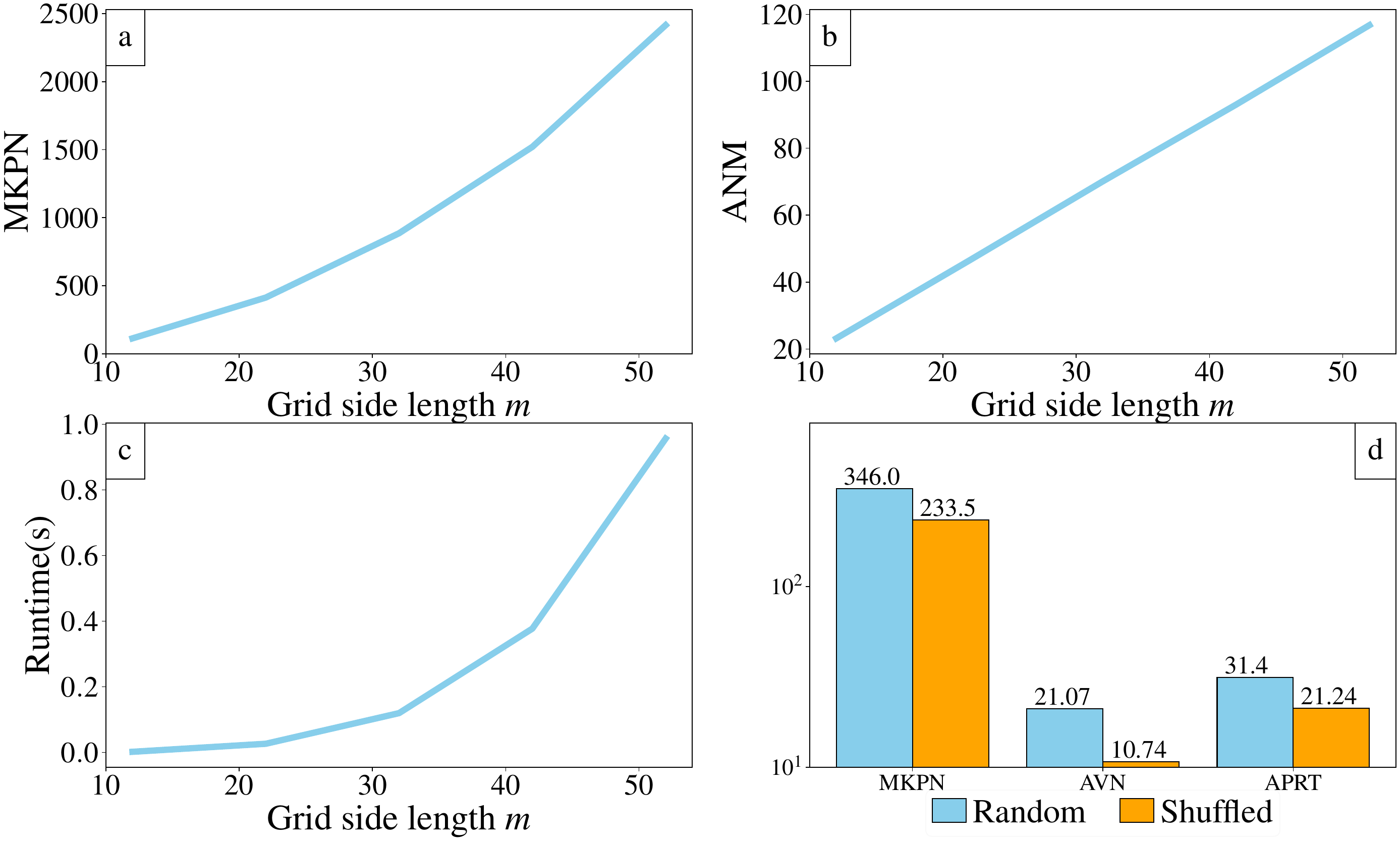}
    \caption{Statistics of performing the column shuffle operations on $m\times m$ grids with varying grid size.}
    \label{fig:shuffles}
\end{figure}%
After shuffling, continuous \gls{csmp} with $p_r=1,p_p=0$ is applied to retrieve all vehicles and compared to the case where no shuffling is performed. 
The outcome makespan, average number of moves per vehicle, and average retrieval time per vehicle are shown in \ref{fig:shuffles}(d).
Compared to unshuffled configuration, \gls{csmp} is able to retrieve all the vehicles with $30\%$-$50\%$ less number of moves and  retrieval time 
(note that logarithmic scale is used to fit all data), showing that rearranging vehicles in anticipation of rush hour retrieval provides significant benefits.  

\section{Conclusion and Discussions}\label{sec4:conclusion}
\vspace{-1mm}
In this chapter, we present the complete physical and algorithmic design of an automated garage system, aiming at allowing the dense parking of vehicles in metropolitan areas at high  speeds/efficiency. 
We model the retrieving and parking problem as a multi-robot path planning problem, allowing our system to support nearly $100\%$ vehicle density.
The proposed \gls{ilp} algorithm can provide makespan-optimal solutions, while \gls{csmp} algorithms are highly scalable with good solution quality. Also clearly shown is that it can be quite beneficial to perform vehicle rearrangement during non-rush hours for later ordered retrieval operations, which is a unique high-utility feature of our automated garage design. 
%

For future work, we intend to further improve \gls{csmp}'s scalability and flexibility, possibly leveraging the latest advances in  \gls{mpp}/MAPF research. 
%
%
We also plan to extend the automated garage design from 2D to 3D, supporting multiple levels of parking. Finally, we would like to build a small-scale test
beds realizing the physical and algorithmic designs.

%% file: chapters/chapter6.tex
\chapter{Decentralized Lifelong Path Planning for Multiple Ackermann Car-Like Robots}\label{chap:cars}

\section{Introduction}\label{sec5:intro}
%
%
%
%

The rapid development of robotics technology in recent years has made possible many revolutionary applications. One such area is multi-robot systems, where many large-scale systems have been successfully deployed, including, e.g., in warehouse automation for general order fulfillment \cite{wurman2008coordinating}, grocery order fulfillment \cite{mason2019developing}, and parcel sorting \cite{wan2018lifelong}. However, upon a closer look at such systems, we readily observe that the robots in such systems largely live on some discretized grid structure. In other words, while we can effectively solve multi-robot coordination problems in grid-like settings, we do not yet see applications where many non-holonomic robots traverse smoothly in continuous domains due to a lack of good computational solutions. Whereas many factors contribute to this (e.g., state estimation), a major roadblock is the lack of efficient computational solutions tackling the lifelong motion planning for non-holonomic robots in continuous domains. 

Toward clearing the above-mentioned roadblock, in this chapter, we proposed two algorithms specially designed to solve static/one-shot and lifelong path/motion planning tasks for Ackermann car-like robots.
The basic idea behind our methods is straightforward: to enable the adaptation of discrete search strategies for car-like robots, we use a small set of fixed but representative motion primitives to transition the robots' states. These fixed motion primitives, when properly put together, yield near-optimal trajectories that ``almost'' connect the starts and goals for robots. Some local adjustments are then used to complete the full trajectory. 
The first algorithm and our main contribution in this research, \emph{Priority-inherited Backtracking for Car-like Robots} (\gls{clpibt})  adapts a decentralized strategy that leverages search-prioritization strategies from Priority Inheritance and Backtracking (PIBT) \cite{Okumura2019PriorityIW}. 
The second algorithm, \emph{Enhanced Conflict-based search for Car-like Robots} (\gls{clecbs}), is a centralized method building on the principles of Enhanced Conflict-Based Search (ECBS) \cite{barer2014suboptimal} and CL-CBS \cite{wen2022cl}, the car-like robot extension of (basic) conflict-based search \cite{sharon2015conflict}. 
We further boost algorithms' success rates by introducing carefully designed, effective heuristics which also reduce the occurrence of deadlocks.
Thorough simulation-based evaluations confirm that our methods deliver scalable SOTA performances on many key practical metrics. 
While our centralized methods tend to find shorter trajectories due to their access to global information, our decentralized method produces better scalability, yielding a higher success rate.

\begin{figure}[t]
    \centering
    \vspace{2.5mm}
  \begin{overpic}               
        [width=1\linewidth]{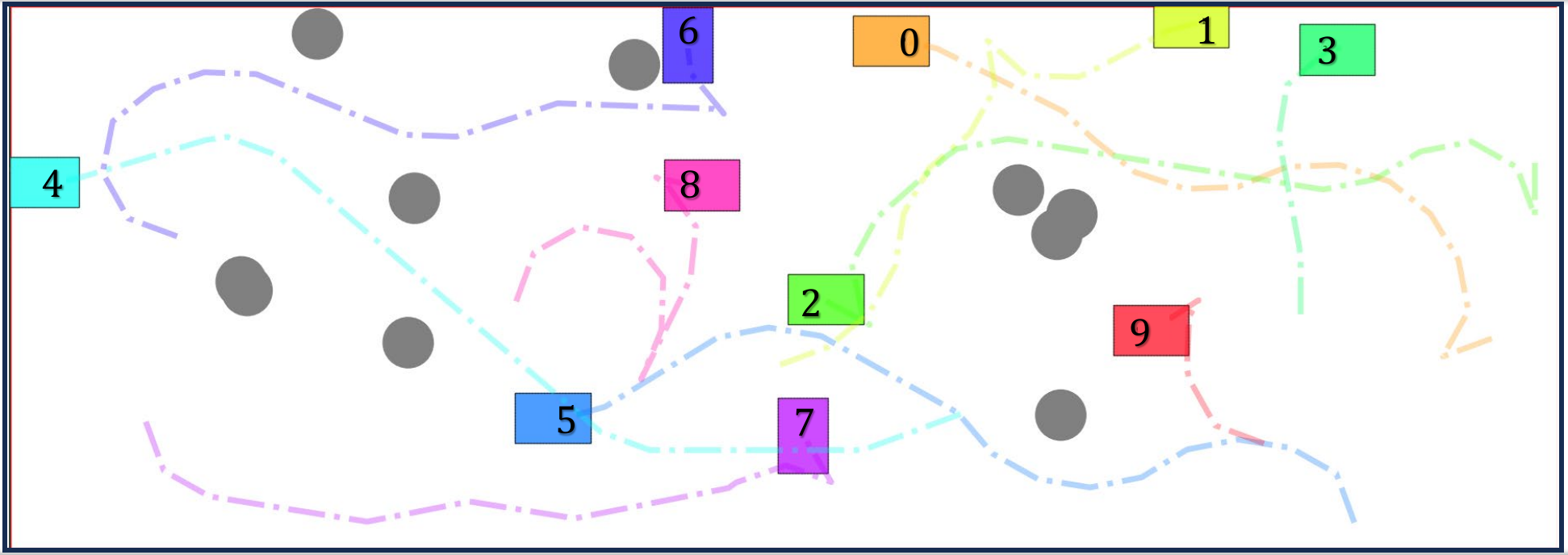}
             \small
        \end{overpic}
\vspace{-3mm}
    \caption{A simulated example on a $60\times 30$ map with 10 obstacles and robots. The collision-free trajectories for the car-like robots are shown. 
    }
    \label{fig:introduction}
\end{figure}%

\textbf{Related Work}. Multi-Robot Path Planning (\gls{mpp}) has garnered extensive research interests in robotics and artificial intelligence in general. The graph-based \gls{mpp} variant, also known as Multi-Agent Path Finding (MAPF) \cite{stern2019multi}, is to find collision-free paths for a set of robots within a given graph environment. Each robot possesses a distinct starting point, and a goal position, and the challenge lies in determining paths that ensure their traversal from start to goal without collisions.
It has been proven many times over that optimally solving the graph-based problem to minimize objectives like makespan or cumulative costs is NP-hard; see, e.g.,  \cite{surynek2010optimization,yu2013structure}. 

Computational methods for \gls{mpp} can be broadly classified into two categories: \emph{centralized} and \emph{decentralized}. Centralized solvers operate under the assumption that robot paths can be computed centrally and subsequently executed with minimal coordination errors. On the other hand, decentralized solvers capitalize on the autonomy of individual robots, enabling them to calculate paths independently while requiring coordination for effective decision-making.
Centralized solutions often involve reducing \gls{mpp} to well-established problems \cite{yu2016optimal,surynek2010optimization,erdem2013general}, employing search algorithms to explore the joint solution space \cite{sharon2015conflict,barer2014suboptimal,sharon2013increasing,silver2005cooperative,okumura2023lacam} or apply human-designed rules for coordination and collision-avoidance \cite{luna2011push, guo2022sub}. 
Decentralized approaches \cite{han2020ddm,Okumura2019PriorityIW} leverage efficient heuristics to address conflicts locally, bolstering their success rates. Machine learning and reinforcement learning methods for \gls{mpp} have also started to emerge, with researchers putting forward data-driven strategies to directly learn decentralized policies for \gls{mpp}  \cite{sartoretti2019primal,li2020graph}.
Nevertheless, we note that the paths generated by graph-based \gls{mpp} algorithms cannot be directly applied to physical robots due to their disregard for robot kinematics. 

As autonomous driving gains momentum, interest in path planning for multiple car-like vehicles is also steadily rising \cite{Guo2023TowardEP,okoso2022high}. This has led to the development of new methods, including adaptations of CBS to car-like robots (CL-CBS) \cite{wen2022cl}, prioritized trajectory optimization \cite{li2020efficient}, sampling-based techniques \cite{le2019multi}, and optimal control strategies \cite{ouyang2022fast}. Decentralized approaches have been introduced as well, such as B-ORCA \cite{alonso2012reciprocal} and $\epsilon$CCA \cite{alonso2018cooperative}, both stemming from adaptations of ORCA \cite{berg2008a} designed for car-like robots. These methods offer faster computation compared to centralized methods. However, their reliance on local information means they cannot provide a guarantee that robots will successfully reach their destinations. Moreover, they are prone to deadlock issues and tend to achieve much lower success rates, particularly in densely populated environments.

\textbf{Organization.}
The rest of the paper is organized as follows. 
Sec.~\ref{sec5:problem} covers the preliminaries, including the problem formulation.
In Sec.~\ref{sec5:pibt}-\ref{sec5:ecbs}, we demonstrate our algorithms in detail.
In Sec.~\ref{sec5:evaluation}, we conduct evaluations on the proposed methods in static settings and lifelong settings and discuss their implications.
We conclude in Sec.~\ref{sec5:conclusion}.

\section{Problem Formulation}\label{sec5:problem}
\subsection{Multi-Robot Path Planning for Car-Like Robots}
 A static instance of this problem is specified as $(\mathcal{W},\mathcal{S},\mathcal{G})$, where $\mathcal{W}$ constitutes a map with dimensions $W\times H$, housing a set of obstacles $\mathcal{O}=\{o_1,...,o_{n_o}\}$. $\mathcal{S}=\{s_1,...s_n\}$ defines the initial configurations of $n$ robots within the workspace, and $\mathcal{G}=\{g_1,...g_n\}$ represents the corresponding goal configurations.

Each robot's $SE(2)$ configuration, denoted as $v_i=(x_i,y_i,\theta_i)$, comprises a 3-tuple of position coordinates and the yaw orientation $\theta_i$. The car-like robot is modeled as a rectangular shape with length $\ell$ and width $w$. 
The subset of $\mathcal{W}$ occupied by a robot's body at a given state $v$ is denoted by $\Gamma(v)$. 
The motion of the robot adheres to the Ackermann-steering kinematics (see \ref{fig:car_model}(a)):
\begin{equation}\label{eq:ackerman}
\dot{x}=u\cos\theta, \quad \dot{y}=u\sin\theta, \quad \dot{\theta}=\frac{u}{\ell_{b}}\tan\phi,
\end{equation}
in which $u\in[-u_{m},u_{m}]$ is the linear velocity of the robot, $\phi\in [-\phi_{m},\phi_{m}]$ is the steering angle. These are the control inputs. $\ell_b$ is the wheelbase length, i.e., the distance between the front and back wheels.

To make planning more tractable, we discretize into intervals $\Delta t$.
The goal is to ascertain a  feasible path for each robot $i$, expressed as a state sequence $P_i=\{p_i(0),...p_i(t),...p_i(T)\}$ that adheres to the following constraints:
(i) $p_i(0)=s_i$ and $p_i(T)=g_i$;
(ii) $\forall t\in [0,T], \forall i\neq j, \Gamma(p_i(t))\bigcap \Gamma(p_j(t))=\emptyset$;
(iii) The path follows the Ackermann-steering kinematic model.
The time interval is kept small during the discretization process. This choice ensures that the distance moved within each step remains smaller than the size of the robot. Consequently, the need to account for swap conflicts \cite{stern2019multi} is eliminated.

Trajectory quality is evaluated using the following criteria:
\begin{itemize}
    \item Makespan: $\max_{1\leq i\leq n}len(P_i)/u_{\max}$;
    \item Flowtime: $\sum_{1\leq i\leq n}  len(P_i)/u_{\max}$
\end{itemize}
where $len(P_i)$ denotes the length of trajectory $P_i$.

In a lifelong setting, we assume an infinite stream of tasks for each robot. Upon completing the current task, a robot immediately receives a new target.
In this scenario, evaluating the system's efficiency often relies on throughput—the number of tasks accomplished (or goal states arrived) within a specified number of timesteps.

\section{Priority-Inherited Backtracking for Car-Like Robots}\label{sec5:pibt}
In this section, we propose a \emph{decentralized} method called \gls{clpibt} for car-like robots, adapting the prioritization mechanism of the decentralized \gls{mpp} algorithm \gls{pibt}~\cite{Okumura2019PriorityIW}. 
In conjunction with the introduction of \gls{clpibt}, we also describe the general methodology we adopt to plan continuous trajectories for non-holonomic robots that is generally applicable, provided that the optimal trajectories connecting the robot can be compactly represented using a few motion primitives.

To render planning for the continuous system feasible, we constrain the robot's possible state transitions within the discretized time interval $\Delta t$, forming a set of \emph{motion primitives}.
Our motion primitives $\mathcal{M}$ are defined similarly as done in \cite{wen2022cl, li2020efficient}, which contain a total of seven actions: forward max-left (FL), forward straight (FS), forward
max-right (FR), backward max-left (BL), backward straight (BS), backward max-right (BR), and wait, as shown in~\ref{fig:car_model} (b).
For the first six motion primitives, we assume that the robot maintains a constant velocity of $u_m$ throughout a single time interval $\Delta t$. When the robot turns, we assume it is pivoting using the maximum steering angle $\phi_m$ and then tracing an arc with a turning radius $r_m$. This arc has a $u_{m}\Delta t$ length.
To ensure a robot can reach its goal state for \gls{clpibt}, one additional \emph{greedy motion primitive} (GM) is added to $\mathcal{M}$.
This greedy motion primitive is derived by truncating the first segment of length $u_{m}\Delta t$ from the shortest path connecting the current state to the goal state for a single robot. In the absence of obstacles on the map, this shortest path corresponds to the (optimal) Reeds-Shepp \cite{reeds1979optimal} path while if obstacles are present, the path can be determined using the vanilla (non-spatio-temporal) hybrid A* algorithm.
It's important to recognize that this greedy motion primitive could be identical to other motion primitives, resulting in the same subsequent state as those alternatives. In such cases, we opt to retain solely the greedy one.

\begin{figure}[h]
    \centering
  \begin{overpic}               
        [width=1\linewidth]{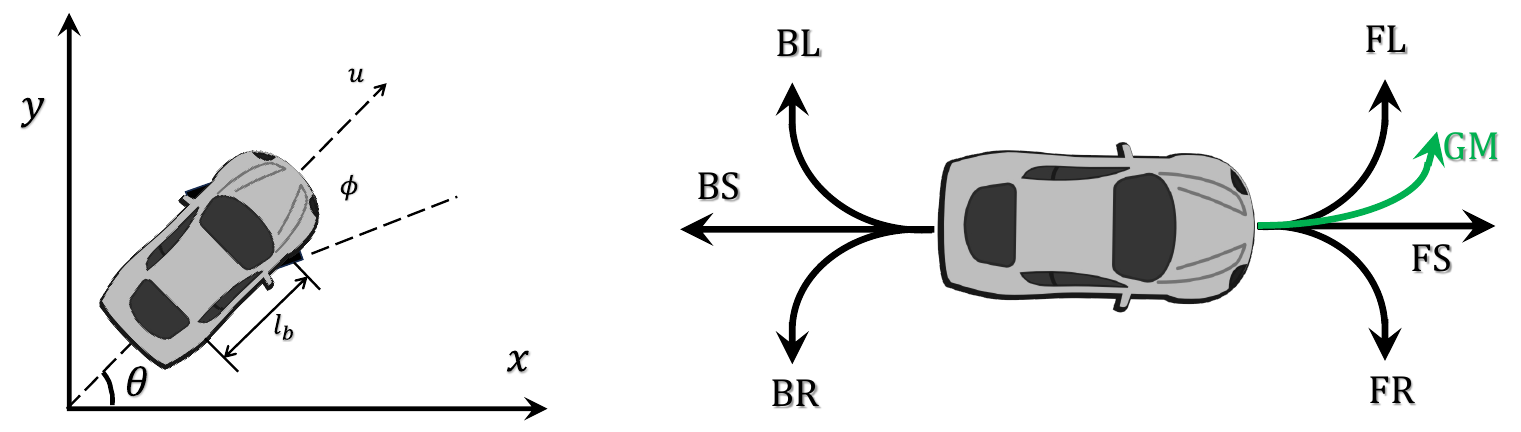}
             \small
             \put(18.5, -3) {(a)}
             \put(70.5, -3) {(b)}
        \end{overpic}
\vspace{-3mm}
    \caption[The Ackermann steering kinematic model]{(a) Ackermann steering kinematic model. (b)  The predefined motion primitives. We have at most 8 motion primitives, which are
forward max-left (FL), forward straight (FS), forward max-
right (FR), backward max-left (BL), backward straight (BS),
backward max-right (BR),  wait, and greedy motion primitive (GM)}
    \label{fig:car_model}
\end{figure}%

\subsection{Algorithm Skeleton}
\gls{clpibt} is outlined in \ref{alg:pibt_main}-\ref{alg:pibt_func}, at the core of which lies the \texttt{PIBTLoop}, which orchestrates the decision-making process for the robots at each time step \( t \). 
The primary objective of \texttt{PIBTLoop} is to determine each robot's next motion and succeeding state without encountering collisions.
\texttt{PIBTLoop} initializes two essential sets: \(\text{UNDECIDED}\) and \(\text{OCCUPIED}\).
The $\text{UNDECIDED}$ set keeps track of robots that have yet to finalize their next actions, while the $\text{OCCUPIED}$ set stores the information about occupied states, preventing multiple robots from attempting to occupy the same space simultaneously.
At the outset of each iteration, the algorithm updates the priorities of all robots based on relevant factors. 
The priorities are crucial in determining which robot takes precedence in decision-making. 
This step ensures that the higher-priority robots have their actions determined earlier, facilitating swift and efficient navigation.
\gls{clpibt}'s decision-making process revolves around an iterative approach. The algorithm enters a loop that continues until all robots in the $\text{UNDECIDED}$ set have their actions determined.
The robot with the highest priority within the loop is selected from the $\text{UNDECIDED}$ set. This robot is designated as robot $i$ for the current iteration.
\begin{algorithm}
\DontPrintSemicolon
\SetKwProg{Fn}{Function}{:}{}
\SetKw{Continue}{continue}
  \Fn{\textsc{PIBTLoop}()}{
 \caption{Priority selection at time step $t$ \label{alg:pibt_main}}
    $\text{UNDECIDED}\leftarrow \mathcal{R}(t)$\;
    $\text{OCCUPIED}\leftarrow \emptyset$\;
    update  all priorities\;
    \While{$\text{UNDECIDED}!=\emptyset$}{
        $i\leftarrow$ the robot with the highest priority in UNDECIDED\;
        $\texttt{PIBT}(i,\text{NONE},\text{NONE})$\;
    }
}
\end{algorithm}

The \texttt{PIBT} function is invoked for the selected robot $i$, along with placeholders for two parameters: the parent robot $j$ which robot $i$ is inherited from, and the potential succeeding state $v_j$ of the parent robot. This function serves as the primary interface for the robot's decision-making. It evaluates potential actions considering the robot's current state, goals, obstacles, and the presence of other robots. Based on these evaluations, the robot's next action is determined to ensure collision-free and goal-oriented navigation.

\begin{algorithm}
\DontPrintSemicolon
\SetKwProg{Fn}{Function}{:}{}
\SetKw{Continue}{continue}
  \Fn{\textsc{PIBT}($i,j,v_j$)}{
 \caption{PIBT function  \label{alg:pibt_func}}
    $\text{UNDECIDED}\leftarrow \mathcal{R}(t)-i$\;
        $\text{NbrUndecided}\leftarrow\text{ the neighboring robots of } i\text{ that remain undecided}$\;
        $\text{NbrOcc}\leftarrow \text{the occupied states of the neighboring decided robots}$\;
    $C\leftarrow\{\texttt{ValidSuccState}(p_i(t), m_i)|m_i\in\mathcal{M}_i\}$\;

    \While{$C\neq \emptyset$}{
        $v_i\leftarrow \text{argmax}_{v\in C}Q_i(v)$ and remove $v_i$ from $C$\;
    
        \If{$\texttt{FindCollision}(v_i,\text{NbrOcc}\bigcup\{p_j(t),v_j\})$}{
            continue\;
        }
     $\text{OCCUPIED}.add(v_i)$\;
        \If{$\exists k\in \text{NbrUndecided}$ such that $\texttt{FindCollision}(p_k(t),v_i)\text{= true}$}{
        \If{$\texttt{PIBT}(k,i,v_i)$ is valid}{
            $p_i(t+1)\leftarrow v_i$\;
   
            \Return valid\;
        }
        \Else{    
            $\text{OCCUPIED}.remove(v_i)$\;
            }
        }
    }
   $\text{UNDECIDED}.add(i)$\;    
    \Return invalid\;
}
\end{algorithm}

The \texttt{PIBT} function plays a pivotal role within \gls{clpibt}, encapsulating the decision-making process for an individual robot. 
This function is called iteratively for each robot to compute its next action, considering its current state, goals, and interactions with neighboring robots.
The function identifies the set $\text{UNDECIDED}$, which comprises all robots yet to finalize their actions, and marks robot $i$ as decided temporarily. Additionally, $\text{NbrUndecided}$ contains neighboring robots of $i$ that remain undecided, and $\text{NbrOcc}$ holds the occupied states of the neighboring robots that have already determined their actions. These sets provide crucial contextual information to facilitate informed decision-making.
The function computes a set $C$ of potential future states $v_i$ for the current robot $i$. These states are generated based on all feasible actions $m_i$ available to robot $i$ at its current state $p_i(t)$. The goal is to explore various potential actions to lead to a successful next state for $r_i$.
The function selects the potential state $v_i$ within a loop that maximizes the robot's utility function $Q_i(v)$ from the set $C$. This action selection aims to optimize the robot's decision by choosing the most promising action according to the utility criterion.
However, before finalizing the action, the function checks for collisions between the selected $v_i$ and the occupied states of neighboring robots, as well as the state of the parent robot $j$ (if it is not NONE) and its potential succeeding state $v_j$. If a collision is detected, the function continues to the next iteration of the loop, considering alternative potential actions.

If no collisions are found for the chosen $v_i$, $v_i$ is added to the $\text{OCCUPIED}$ set, indicating the robot intends to occupy this state. The function then checks if an undecided neighboring robot $k$ will collide with $v_i$.
If such a robot $k$ is found, the function recursively calls \texttt{PIBT} for robot $k$ with the parent robot $i$ and the tentative state $v_i$ as the parameters. If the resulting action is valid, robot $i$ updates its next state $p_i(t+1)$ to $v_i$, and the function returns ``valid."
If \texttt{PIBT} for robot $k$ fails, the state $v_i$ is removed from $\text{OCCUPIED}$.
After evaluating all potential actions and considering interactions with neighboring robots, robot $i$ remains in the $\text{UNDECIDED}$ set. The function returns ``invalid" if a valid action cannot be determined. Otherwise, it returns ``valid" along with the updated next state $p_i(t+1)$ for robot $i$.

We note that \gls{clpibt}'s one-step planning and execution nature makes it applicable to static and lifelong scenarios.

\subsection{Performance-Boosting Heuristics}
Multiple heuristics are introduced to boost the performance of \gls{clpibt}, while some are basic, others are involved. 

\textbf{Distance heuristic.} We use the max of the holonomic cost with obstacles, the length of the shortest Reeds-Shepp path between two states, and 2D Euclidean distance as our distance heuristic function (\texttt{DistH}) similar to \cite{dolgov2008practical,wen2022cl}.
This distance heuristic is admissible. 

\textbf{Priority heuristic.}
\gls{clpibt} constitutes a single-step, priority-driven planning approach that necessitates the ongoing adjustment of each robot's priority at each time step.
Initially, the priority assigned to each robot is determined based on the number of time steps that have transpired since its preceding task was updated.
The robot with more time since its previous goal update receives a higher priority ranking.
In static scenarios, the elapsed time is reset to zero whenever a robot reaches its designated goal state.
This reset mechanism prevents robots that have achieved their goal state from obstructing the progress of other robots.
When multiple robots share the same elapsed time, prioritization is resolved by favoring those robots with a larger heuristic distance value to their respective goal states.
This heuristic principle finds widespread application in the realm of prioritized planning \cite{van2005prioritized} to improve the success rate.

\textbf{The Q-function.}
The $Q$-function is a critical tool for assessing the optimal choice of a motion primitive in guiding the robot from its current state toward its goal state.
\gls{pibt} algorithm employs the single-agent shortest path length from the current vertex to its goal vertex, neglecting inter-agent collisions, as the basis for its evaluation function. This approach is feasible for discrete 2D graphs since these shortest path lengths can be pre-calculated and stored in a table. Nevertheless, this approach is impractical for the state space we're addressing with car-like robots. Firstly, the state space is infinite. Secondly, determining a single-robot shortest trajectory from a state to a goal state using hybrid A* for evaluating all the motion primitives is both time-intensive and incomplete.
For these reasons, we consider employing the distance heuristic function discussed earlier as well as the action cost, denoted as 
\vspace{1mm}
\begin{equation}\label{eq: Q_funcv0}
    Q_i'(v)=-\texttt{DistH}(v,g_i)-\lambda \texttt{Cost}(p_i(t),v)
\vspace{1.5mm}
\end{equation}
where $\lambda$ is a weight and $\texttt{Cost}(p_i(t),v)$ returns the action cost from state $p_i(t)$ to $v$.
We use the action cost similar to \cite{wen2022cl} where backward motions, turning, and changes of directions will receive additional penalty cost.
However, it's important to note that the distance heuristic is not a ``reachable" heuristic in the context of \gls{clpibt}. A heuristic is labeled as ``reachable" for $\gls{pibt}$-like algorithms if, while ignoring inter-robot collisions, a solitary robot is guaranteed reach its goal state by consistently choosing the ``best" action with the maximum $Q$ value at each timestep. In discrete \gls{mpp}, the single-agent shortest path length qualifies as a ``reachable" heuristic.
We mention that while Manhattan distance is a reliable ``reachable" heuristic when the map is obstacle-free, this is not guaranteed when obstacles are present. In scenarios with obstacles, a robot directed by the Manhattan distance heuristic might become entrapped in local minima, obstructing its path to the goal state.
Our case showed \texttt{DistH} is not a ``reachable" heuristic either. Despite being admissible, it lacks the required precision. Notably, the assertion that a greedy motion primitive should yield the highest $Q$ value among the possibilities isn't consistently upheld when using $Q_i(v)=-\texttt{DistH}(v,g_i)$.
This discrepancy easily leads the robot into a state of stagnation. 
Furthermore,  we found that this greedy heuristic could also lead to deadlocks.
An example can be seen in \ref{fig:dead_lock}.
\begin{figure}[h]
    \centering
  \begin{overpic}               
        [width=1\linewidth]{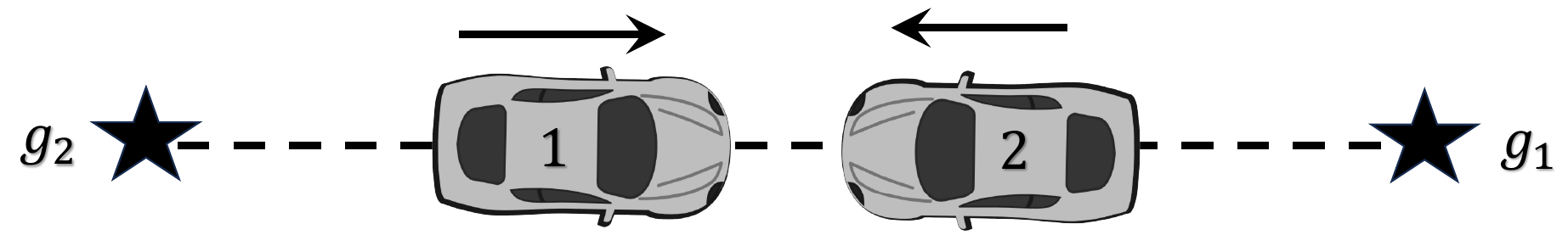}
        \end{overpic}
\vspace{-3mm}
    \caption[An illustrative case highlighting the potential occurrence of deadlocks due to the exclusion of the count-based heuristic is as follows]{
An illustrative case highlighting the potential occurrence of deadlocks due to the exclusion of the count-based heuristic is as follows:
In this scenario, two robots, labeled as robot 1 and robot 2, travel in opposite directions along a linear path. When robot 2 yields to robot 1 in \gls{clpibt}, a predicament arises. In this situation, the motion primitives FS, FR, and FL lead to collisions, compelling robot 2 to consistently opt for the motion primitive BS. This choice is driven by BL and BR, involving additional turning penalties.
A dynamic shift occurs in priorities upon robot 1's successful arrival at its designated goal state. Robot 1 is assigned a lower priority than robot 2 and consequently yields to the latter. Similarly, robot 1 consistently selects BS using the greedy strategy. Consequently, an unending cycle emerges, entangling robot 1 and robot 2 in an indefinite sequence of movements }
    \label{fig:dead_lock}
\end{figure}%

We present a novel count-based exploration heuristic to address the challenge, inspired by \cite{tang2017exploration}. This heuristic is specifically designed to surmount the issue of local minima arising from inaccuracies in distance heuristics and deadlocks and to foster exploration of uncharted territories.
To achieve this, we incorporate supplementary penalties for states that have been visited, aiming to encourage robots to break free from local minima and venture into unexplored regions. Each state $v=(x,y,\theta)$ is mapped to its nearest discrete state $\widetilde{v}=(\lfloor\frac{x}{\delta x}\rfloor,\lfloor\frac{y}{\delta y}\rfloor,\lfloor\frac{\theta}{\delta \theta}\rfloor)$. This mapping facilitates the tallying of occurrences using a hash table $\mathcal{H}_i$ for each robot, where $\delta x$, $\delta y$, and $\delta \theta$ delineate the state-space resolution.

When a robot $i$ traverses state $v$ at a given timestep $t$, the count $\mathcal{H}_i(\widetilde{v})$ is incremented by one. In addition, we introduce bonus rewards to incentivize the selection of the greedy motion primitive.
The resultant $Q$-function is:
\vspace{1mm}
\begin{equation}\label{eq: q_funcv1}
\begin{aligned}
Q_i(v)=Q'_i(v)+\alpha Gr(v)-\beta\mathcal{H}_i(\widetilde{v})
\end{aligned}
\vspace{1mm}
\end{equation}

Here, $\alpha$ and $\beta$ represent positive weight parameters. If the state $v$ results from a greedy motion primitive, $Gr(v)=1$; otherwise, $Gr(v)=0$.
To accommodate the requirement that robots should remain at their goal states upon arrival, we refrain from applying penalties to states near the goal states. As a result, when $v=g_i$ (the goal state of robot $i$), we set $\beta$ to zero.
We mention that such a $Q$ function could be ``reachable" if we properly choose the weight parameters.
Moreover, by implementing the count-based heuristic, which compels robots to explore distinct motion patterns for collision avoidance, the deadlock issues can be nicely achieved.
%
%
\section{Centralized ECCR}\label{sec5:ecbs}
We enhance the \gls{clcbs} algorithm by substituting the conventional low-level spatiotemporal hybrid state A* planner with focal hybrid state A* search \cite{barer2014suboptimal}. This refined approach is named \gls{clecbs}. In comparison to \gls{clcbs}, the \gls{clecbs} method guides the low-level planner to discover trajectories within a bounded suboptimality ratio while encountering fewer potential conflicts with other robots. 
This reduction in conflicts subsequently results in a significantly lower number of high-level expansions.

For the purpose of adapting \gls{clecbs} to lifelong scenarios, characterized by the necessity for frequent online replanning, we incorporate the windowed version of focal search in the low-level planning process \cite{li2021lifelong, han2022optimizing}. Specifically, during path planning, we only address conflicts that arise within a window of $\omega$ steps. This adjustment effectively decreases the runtime of the \gls{clecbs} algorithm during the planning phase. Replanning occurs at intervals of every $\omega$ steps and also when robots reach their current goals and receive new tasks.
\section{Evaluation}\label{sec5:evaluation}
In this section,  we evaluate the proposed algorithms in static scenarios and lifelong scenarios.
All methods are implemented in C++. 
The source code can be found in \url{https://github.com/GreatenAnoymous/CarLikePlanning}.
%
All experiments are performed on an Intel\textsuperscript{\textregistered} Core\textsuperscript{TM} i7-6900K CPU at 3.0GHz in Ubuntu 18.04LTS.
In the simulation, the following parameters are used:
$w=2$, $l=3$, $l_b=2$, $\delta x=\delta y= 2$, $\delta\theta=40.1^{\circ}$, $u_{m}=2$, $\phi_{m}=40.1^{\circ}$, $r_m=3$, $\Delta t=r_m\delta \theta/u_{m}$.

\subsection{Static Scenarios}
In this section, we assess the performance of the algorithms on a grid of dimensions $100\times 100$, both with and without obstacles. For the scenarios involving obstacles, we create 50 circular obstacles with a radius of 1 unit length and place them randomly on the map.
For each value of $n$, we generate 50 distinct instances, ensuring that the states of the robot do not overlap with each other and with the obstacles in start and goal configurations. A time limit of 60 seconds is imposed on each instance.
The success rate is determined by tallying the instances each algorithm successfully solves within the time limit. Additionally, we evaluate the average runtime (in seconds), makespan, and flowtime across the solved instances.
Our evaluation includes a comparison with two reference algorithms: firstly, the prior centralized algorithm known as CL-CBS, and secondly, the decentralized SHA* algorithm as described in \cite{wen2022cl}. We emphasize that all algorithms use identical predefined motion primitives, except that GM is exclusively used for \gls{clpibt}.

A suboptimality ratio of 1.5 is selected for \gls{clecbs}.
Three variants of \gls{clpibt} are evaluated. The first, \gls{clpibt}(v0), employs the $Q$-function described in \eqref{eq: q_funcv1} without incorporating the count-based exploration heuristic.
Both \gls{clpibt}(v1) and \gls{clpibt}(v2) leverage the count-based exploration heuristic to resolve local minima and deadlocks.
In the case of \gls{clpibt}(v1), we opt to clear the hash table $\mathcal{H}_i$ and reset the associated counts to zero whenever robot $i$ reaches its designated goal state. Conversely, for \gls{clpibt}(v2), the visiting history is not cleared.
For all variants of \gls{pibt}, the maximum time step is constrained to 500.
And we set $\lambda=0.3$, $\alpha=\beta=\lambda \texttt{Cost}(p_i(t),v)$.

Detailed evaluation results are presented in~\ref{fig:empty_oneshot} through~\ref{fig:arival_rate}.
Compared to CL-CBS, \gls{clecbs} showcases significantly enhanced scalability and success rates due to its capability to expand a notably smaller number of high-level nodes. Notably, the solution quality of \gls{clecbs} closely approaches that of CL-CBS.
On the flip side, the \gls{clpibt} variants excel in terms of runtime efficiency, resulting from their decentralized nature and one-step planning strategy. They can resolve instances involving 60 robots in as little as 4 seconds. Among these variants, \gls{clpibt} featuring the count-based exploration heuristic achieves an elevated success rate. Conversely, without the count-based exploration heuristic, \gls{clpibt} (v0) falters in instances with obstacles, revealing its limitations. However, \gls{clpibt}(v2) successfully tackles $93\%$ of cases involving 60 robots within 4 seconds, excelling in challenging scenarios—a substantial improvement over other variants and the previous decentralized SHA* approach.
\begin{figure}[h]
    \centering
    \includegraphics[width=1\linewidth]{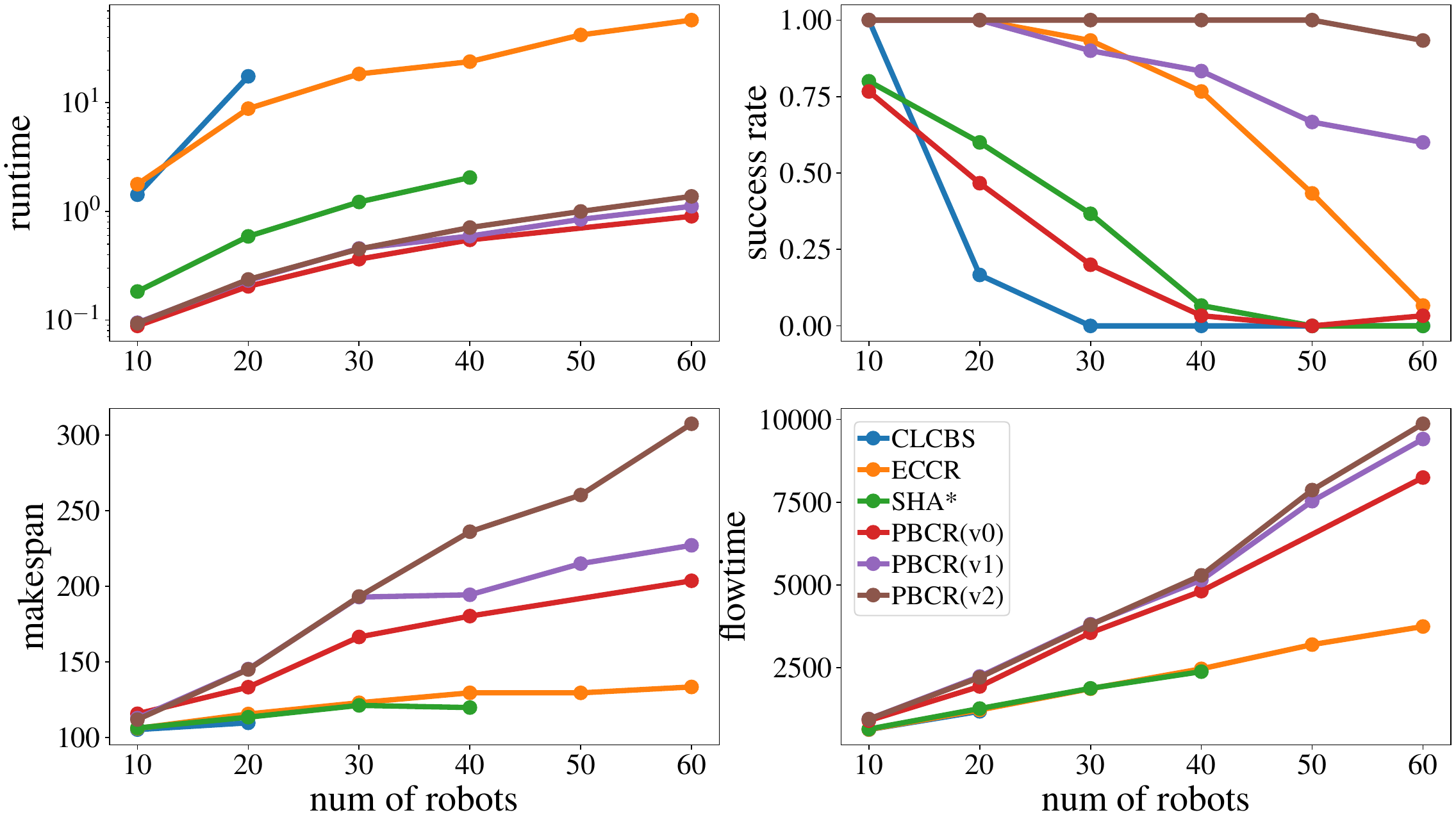}
\vspace{-5mm}
    \caption{Evaluation results on a $100\times 100$ empty map for a varying number of robots.}
    \label{fig:empty_oneshot}
\end{figure}%

\begin{figure}[h]
\vspace{2mm}
    \centering
   \includegraphics[width=1\linewidth]{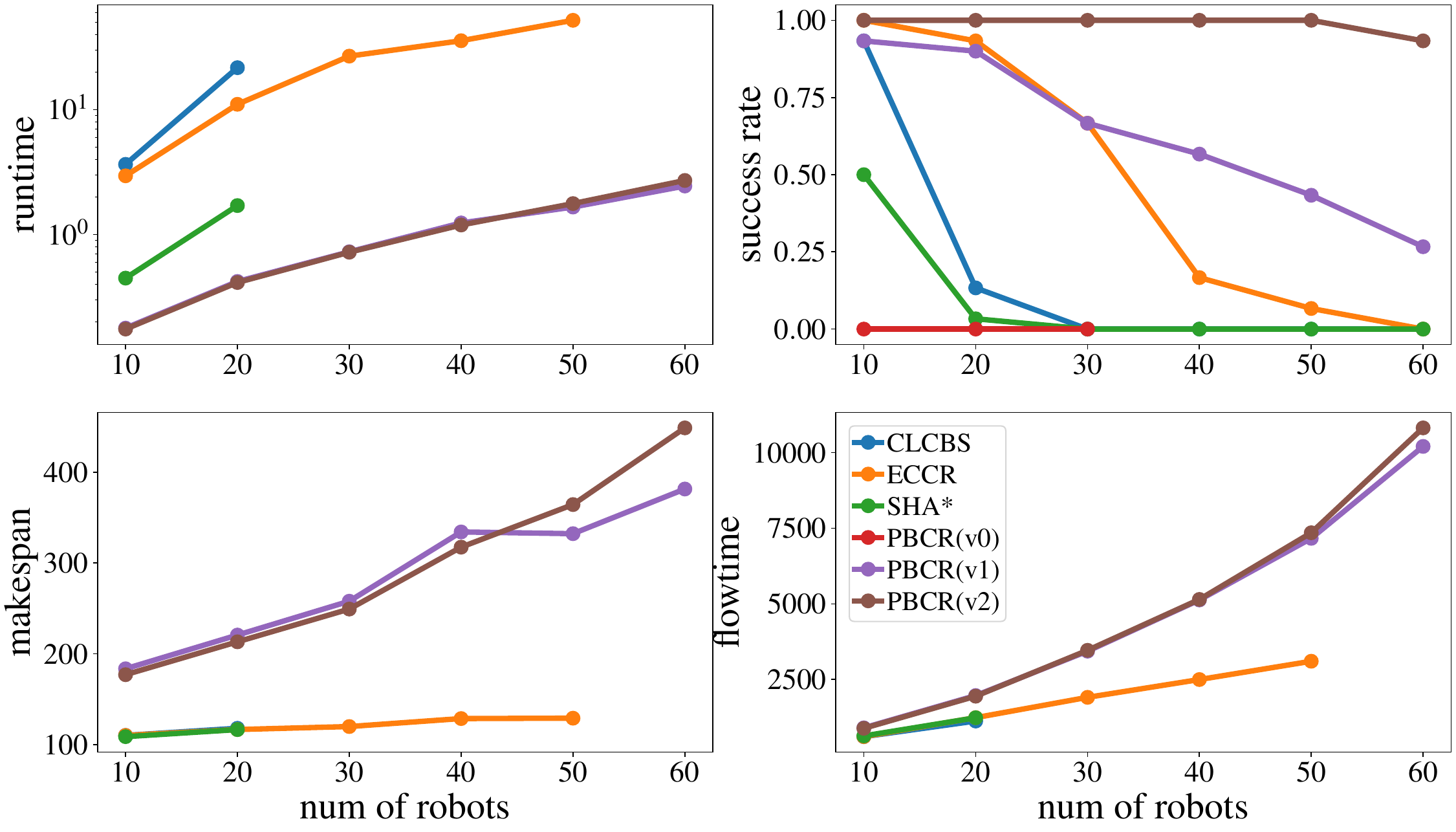}

\vspace{-3mm}
    \caption{Evaluation data on a $100\times 100$ map with $50$ randomly placed obstacles for a varying number of robots.}
    \label{fig:obstacle_oneshot}
\end{figure}%

\begin{figure}[h]
    \centering
  \begin{overpic}               
        [width=1\linewidth]{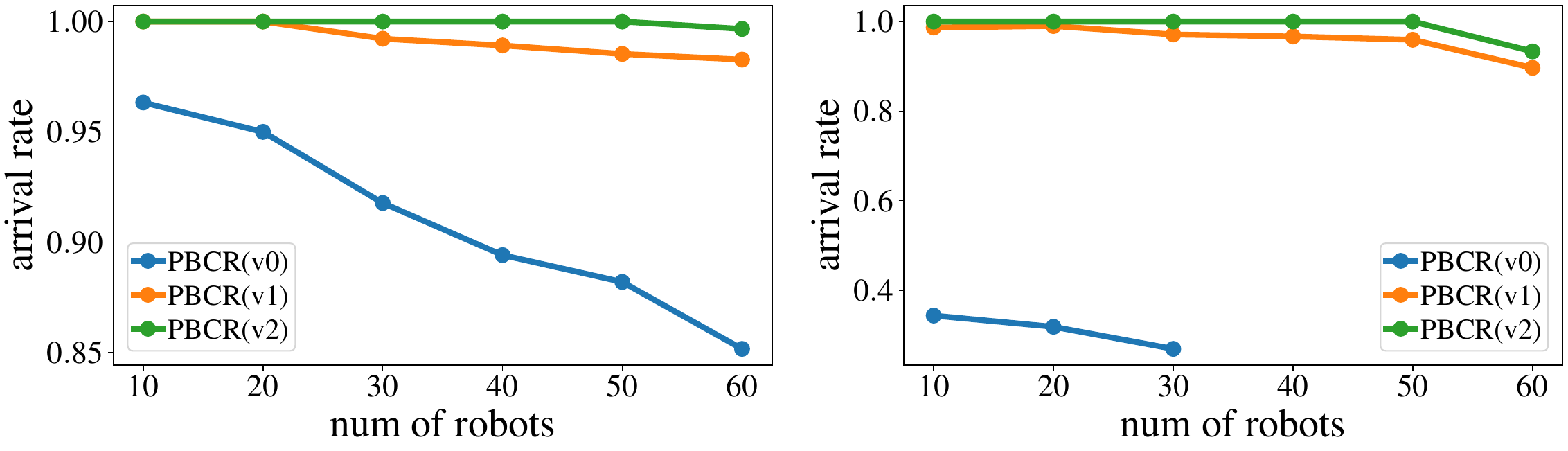}
             \put(25.5, -3) {(a)}
             \put(75.5, -3) {(b)}
        \end{overpic}
\vspace{-3mm}
    \caption{The average percentage of robots arrived at their goal state for each \gls{clpibt} variant on the $100\times 100$ maps for varying numbers of robots. (a) without obstacles (b) with obstacles.}
    \label{fig:arival_rate}
\end{figure}%
It's crucial to emphasize that all \gls{clpibt} variants do not fail due to time constraints but rather due to surpassing the maximum timestep.
For those failed instances, \gls{clpibt} with count-based heuristic still can guide more than $90\%$ of the robots to their goal states as shown in \ref{fig:arival_rate}.
The contrast in success rates between \gls{clpibt}(v2) and \gls{clpibt}(v1) implies that retaining the visiting history until completion, instead of resetting it upon each robot's arrival at its goal state, is more effective.
This arises from the fact that in the context of \gls{clpibt}, when a robot reaches its designated goal, it often has to temporarily vacate its goal position to accommodate other robots that must traverse it. This frequent shifting in and out of the goal state can potentially lead to repetitive cycles and hinder progress.
Maintaining the visiting history mitigates this issue by preventing the robot from becoming trapped in an endless loop of entering and leaving the goal, consequently boosting the overall success rate.
One drawback of \gls{clpibt} lies in its trajectory quality. This stems from the absence of global information and the inherent nature of one-step planning, leading to longer planned trajectories compared to centralized algorithms such as \gls{clecbs} and CL-CBS. This effect is particularly pronounced when dealing with a large number of robots.

\subsection{Lifelong Scenarios}
In this section, we subject both \gls{clecbs} and \gls{clpibt} to testing within lifelong settings.
We adopt a $50 \times 50$ map configuration, both with and without obstacles. For scenarios involving obstacles, a set of 10 obstacles is distributed randomly.
For each value of $n$, we randomly generate 20 unique instances. In each instance, both the initial configurations and 4000 goal states, randomly generated, are allocated to each robot.
Throughout the simulation, we assume a lack of a priori knowledge on the part of the robots regarding their subsequent tasks.
In each instance, we define a maximum of 5000 simulated steps and set a runtime limit of 600 seconds. Each robot has many tasks that cannot be feasibly completed within the designated maximum steps.
Regarding \gls{clecbs}, a suboptimality ratio of 1.5 is established, while a window size $\omega$ of 5 is employed.
Given that new goals are allocated to robots upon completing current tasks, the visiting history is consistently cleared. This decision is driven by the fact that the visiting experience only relates to the robot's preceding task.
\gls{clpibt}(v0) is excluded from consideration owing to its poor performance in addressing issues related to deadlocks and stagnation.
The outcome of our evaluations is depicted in \ref{fig:lifelong}.
\gls{clpibt} achieves significantly lower execution times than \gls{clecbs}, rendering it suitable for handling large-scale scenarios and real-time planning. Nonetheless, it accomplishes a smaller number of tasks.

\begin{figure}[h]
    \centering
  \begin{overpic}               
        [width=1\linewidth]{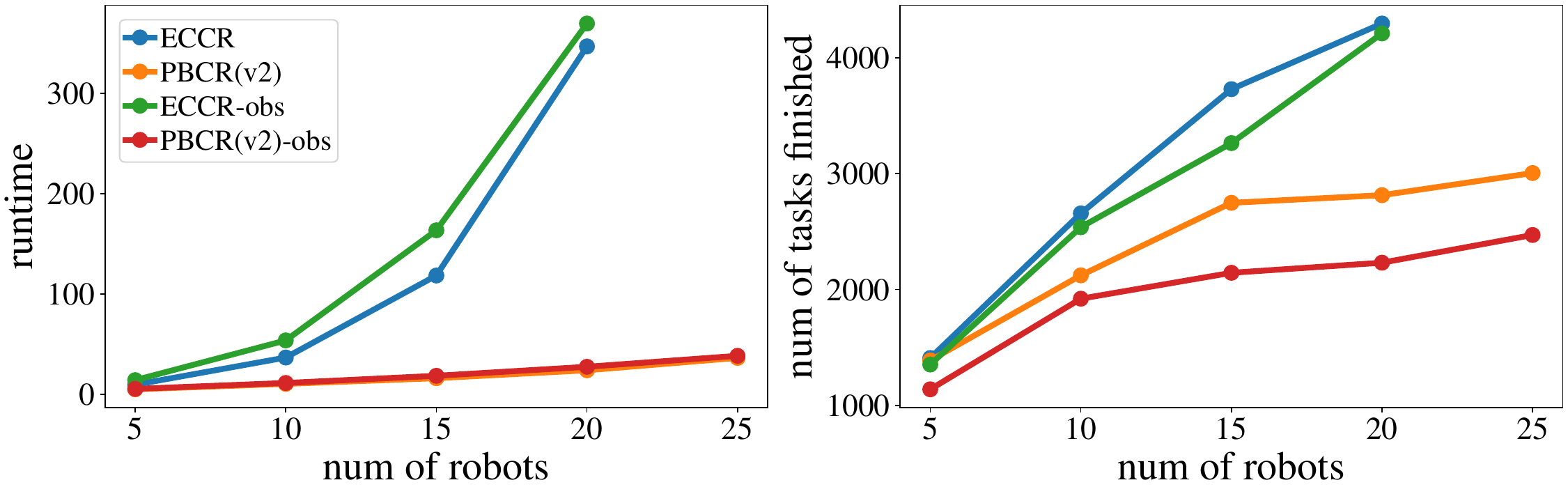}
        \end{overpic}
\vspace{-5mm}
    \caption{Runtime and the average number of tasks finished within given timesteps in lifelong scenarios for varying number of robots. The term ``obs'' is an abbreviation for scenarios with obstacles.}
    \label{fig:lifelong}
\end{figure}%

\subsection{Real-Robot Experiments}
We employed the portable multi-robot platform, microMVP \cite{yu2017portable}, for conducting real-world experiments involving car-like robots to validate the execution of our algorithmic solutions on real hardware. Whereas the microMVP robots are differential drive robots, a software layer can be imposed to simulate car-like robots, which is what we did \footnote{While achieving the same goal, the simulation makes the experiment more challenging than running directly on car-like robots.}. 
\ref{fig:micromvp} gives a visual representation of the setup. As evident from the attached video, paths generated by our planners can be successfully executed on the robots' controllers.
\begin{figure}[h]
\vspace{2mm}
    \centering
  \begin{overpic}               
        [width=\linewidth]{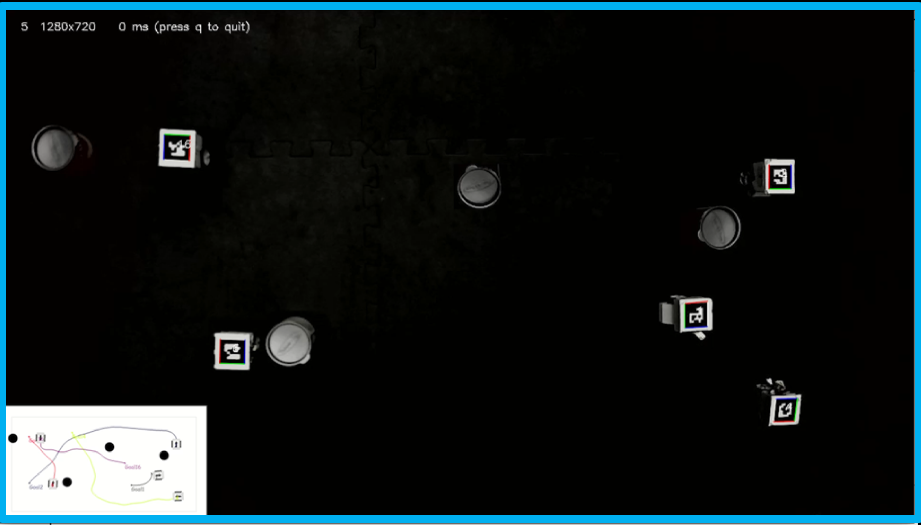}
        \end{overpic}
    \caption{Snapshot of a real robot experiment with 4 obstacles and 5 robots.}
    \label{fig:micromvp}
    \vspace{-2mm}
\end{figure}%

\section{Conclusion and Discussions}\label{sec5:conclusion}
In this study, we investigate path planning for multiple car-like robots. We present two distinct algorithms tailored to the specific demands of car-like robot navigation scenarios.
The first and our main contribution, \gls{clpibt}, employs an effective count-based exploration heuristic. Focusing on decentralized decision-making, \gls{clpibt} demonstrates a notable advancement over previous decentralized approaches. The improvement is evident through significantly heightened success rates with some manageable optimality trade-offs of yielding longer trajectories. 
Opportunities for improving trajectory quality within \gls{clpibt} remain; in particular, our approach employs a manually designed $Q$-function for action selection, which may be replaced with data-driven approaches for reaching optimal performance.

%% file: chapters/chapter7.tex
\chapter{CONCLUSION AND FUTURE WORK}~\label{chap:conclusion}
\gls{mpp} is a fundamental research topic in robotics and artificial intelligence. It has been actively studied for a few decades due to its hardness and importance in applications, and will be continuously studied in the near future.

\subsection{Contributions and Influences}

In this dissertation, we have focused on multiple \gls{mpp} variations and proposed theoretically strong algorithms, and high performance heuristics to push the current \gls{mpp} algorithm design to new levels from a few directions. Furthermore, we investigated different potential applications of \gls{mpp} on well-formed infrastructures, high-density autonomous parking and retrieval systems as well as path planning for multiple car-like robots.

In~\ref{chap:rubik}, we present a polynomial-time algorithm for \gls{mpp} on grids, achieving improved constant-factor optimality guarantees over existing methods. By integrating the Grid Rearrangement Algorithm with our parallel odd-even sorting approach, we attain a makespan of $4(m_1 + 2m_2)$ for fully occupied grids, significantly outperforming prior work. For grids with robot densities of $1/2$ and $1/3$, our highway and line-merge motion primitives yield an asymptotic makespan optimality ratio of $1 + \frac{m_2}{m_1 + m_2}$ under uniform random start and goal distributions. This demonstrates that limiting robot density enables substantial performance gains. We also extend the 2D grid algorithms to $k$-dimensional grids and derive corresponding theoretical guarantees and optimality bounds. Additionally, we introduce a path refinement strategy using \gls{mcp}, further enhancing solution quality. 
The work is pulished in~\cite{guo2022rubik, Guo2022PolynomialTN, GuoYu2024JAIR}.

In~\ref{chap:well-connected}, we systematically study well-formed infrastructures by formulating and analyzing the maximal and largest well-connected vertex set problem. Such layouts are prevalent in real-world applications, including autonomous warehouses and conventional parking lots. Maximizing the well-connected vertex set size is crucial for optimizing space utilization. We first prove the intractability of finding the maximum well-connected vertex set by reducing the problem to 3SAT. To address this challenge, we propose efficient algorithms, including a greedy approach for computing maximal well-connected sets and a search-based algorithm for identifying the largest set. These methods facilitate more effective infrastructure design for applications such as warehouse automation. Furthermore, we highlight the relationship between well-connected sets and well-formed \gls{mpp}. This work is published in~\cite{Guo2024well}.
 
In~\ref{chap:garage},  we study the problem of automated vehicle parking and retrieval using multiple robotic valets, proposing a complete automated garage design that supports near $100\%$ parking density and develops efficient algorithms for its operation.
We introduce the \gls{bcpr} and \gls{ccpr} problems, which model key operations in high-density automated garages. These formulations establish a foundation for future theoretical and algorithmic research in this domain.
We design a system that supports a parking density as high as $(m_1-2)(m_2-2)/m_1m_2$ on an $m_1 \times m_2$ grid, enabling multi-vehicle parking and retrieval. This density approaches $100\%$ for large-scale garages.
Leveraging the structured, grid-like nature of the parking environment, we propose an optimal \gls{ilp}-based method and a highly scalable, fast suboptimal algorithm based on sequential planning. The suboptimal method maintains strong solution quality while being suitable for large-scale applications.
We introduce a shuffling mechanism to rearrange vehicles during off-peak hours, improving retrieval efficiency during peak demand. Our algorithm executes these rearrangements in $O(m_1m_2)$ time at near-full garage density, significantly reducing retrieval delays. to address the growing need for space efficiency in urban environments, we introduced a puzzle-based automated parking and retrieval system capable of operating at near $100\%$ robot density, along with specialized \gls{mpp}-based algorithms that ensure efficient vehicle storage and retrieval. This work is published in
~\cite{guo2023toward}. 

Finally, in~\ref{chap:cars}, we extended our research to continuous domains with non-trivial robot kinematics, specifically focusing on Ackerman-steering car-like robots. This work addresses the challenges of motion planning for non-holonomic robots, specifically Ackerman car-like robots, in continuous domains. The key contributions are as follows:
We propose two novel algorithms designed to solve both static and lifelong path planning tasks for car-like robots. The first algorithm, \gls{clpibt}, adapts a decentralized strategy leveraging \gls{pibt}. The second algorithm, \gls{clecbs}, extends \gls{cbs} to car-like robots, offering a centralized solution.
Our methods incorporate effective heuristics to enhance performance, reduce deadlocks, and improve success rates. Through thorough simulation-based evaluations, we demonstrate that our decentralized approach provides higher scalability and better success rates, while the centralized approach generates shorter trajectories by leveraging global information.
These algorithms represent a significant step toward practical, scalable solutions for multi-robot systems in dynamic, real-world environments, especially in applications like warehouse automation and collaborative exploration.. This work is published in~\cite{Guo2024de}.

\subsection{Limitations and Future Works}

First, from a theoretical and algorithmic perspective on grid-based \gls{mpp}, our \gls{rta}-based algorithms face limitations, particularly in grids with randomly distributed obstacles. For such challenging scenarios, further development of novel heuristics and algorithms is required to solve them in polynomial time while maintaining optimal makespan guarantees.

Regarding the autonomous parking and retrieval problem, we plan to implement more advanced simulations that account for vehicle shapes and the movements of robots carrying the vehicles. Additionally, our current methods do not consider obstacles within the environment, and we aim to extend our solutions to scenarios involving obstacles in the parking and retrieval system.

In our study of well-connected sets and well-formed infrastructures, we assume that the environment is represented as a graph, with each robot occupying a vertex. However, this assumption is not always valid, particularly in parking lot scenarios. Therefore, extending our algorithms for finding \gls{lwcs} to account for robot shapes is one of our future research directions.

Lastly, while we have introduced \gls{clpibt}, which outperforms previous decentralized algorithms in terms of success rate, our experiments indicate that there is still room for improvement in solution quality. Currently, we use a rule-based heuristic to reduce the likelihood of deadlock, but an adaptive, deep-learning-based method is a promising direction for future research.